\documentclass[twoside]{article}

\usepackage[accepted]{aistats2022}
\usepackage[round]{natbib}

\PassOptionsToPackage{numbers}{natbib}
 \usepackage{amsmath}
 \usepackage{amsthm}
\usepackage{amssymb}
\usepackage{graphicx}
\usepackage{xspace}
\usepackage{multirow}

\usepackage[compatible]{algpseudocode}
\usepackage[linesnumbered,ruled,vlined]{algorithm2e}

\usepackage[most]{tcolorbox}
\usepackage[compatible]{algpseudocode}
\usepackage[linesnumbered,ruled,vlined]{algorithm2e}
\usepackage{enumitem}
\newtheorem{theorem}{Theorem}[section]
\newtheorem{lemma}{Lemma}[section]
\newtheorem{corollary}{Corollary}[section]
\newtheorem{claim}{Claim}[section]

\newtheorem{assumption}{Assumption}
\newtheorem{definition}{Definition}[section]

\usepackage{color, xcolor}
\definecolor{mydarkblue}{rgb}{0,0.08,0.45}
\usepackage[colorlinks=true,
    linkcolor=mydarkblue,
    citecolor=mydarkblue,
    filecolor=mydarkblue,
    urlcolor=mydarkblue,
    pdfview=FitH]{hyperref}

\SetCommentSty{mycommfont}
\SetKwInput{KwInput}{Input}
\SetKwInput{KwInputServer}{Server Input}
\SetKwInput{KwInputClient}{i-th User Input}
\SetKwInput{KwOutput}{Output}
\def\HiLiYellow{\leavevmode\rlap{\hbox to \hsize{\color{yellow!50}\leaders\hrule height .8\baselineskip depth .5ex\hfill}}}
\def\HiLiRed{\leavevmode\rlap{\hbox to \hsize{\color{red!40}\leaders\hrule height .8\baselineskip depth .5ex\hfill}}}

\newcommand{\scaffold}{\texttt{SCAFFOLD}\xspace}
\newcommand{\fedavg}{\texttt{FedAvg}\xspace}
\newcommand{\fedsgd}{\texttt{FedSGD}\xspace}
\newcommand{\dpfedsgd}{\texttt{DP-FedSGD}\xspace}
\newcommand{\dpfedavg}{\texttt{DP-FedAvg}\xspace}
\newcommand{\fedprox}{\texttt{FedProx}\xspace}
\newcommand{\dpscaf}{\texttt{DP-SCAFFOLD}\xspace}

\def\DpScaffold{\texttt{DP-SCAFFOLD}}
\def\Scaffold{\texttt{SCAFFOLD}}

\usepackage{todonotes}
\usepackage{thm-restate}

\def\r{R}

\usepackage{cleveref} 

\crefname{Th}{Theorem}{Theorems}

\crefname{lemma}{Lemma}{Lemmas}
\crefname{fact}{Fact}{Facts}
\crefname{theorem}{Theorem}{Theorems}
\crefname{corollary}{Corollary}{Corollaries}
\crefname{Prop}{Proposition}{Propositions}

\crefname{claim}{Claim}{Claims}
\crefname{example}{Example}{Examples}
\crefname{problem}{Problem}{Problems}
\crefname{definition}{Definition}{Definitions}
\crefname{assumption}{Assumption}{Assumptions}
\crefname{subsection}{Subsection}{Subsections}
\crefname{section}{Section}{Sections}
\crefname{algorithm}{Algorithm}{Algorithms}
\crefname{algocf}{alg.}{algs.}
\Crefname{algocf}{Algorithm}{Algorithms}
\crefname{proposition}{Proposition}{Propositions}
\crefname{exemple}{Exemple}{Examples}
\crefname{remark}{Remark}{Remarks}

\def\sigmaoracle{\varsigma}

\usepackage{booktabs}
\usepackage{xr}

\begin{document}

%

%

\twocolumn[

\aistatstitle{Differentially Private Federated Learning on Heterogeneous Data}

\aistatsauthor{ Maxence Noble \And Aurélien Bellet \And  Aymeric Dieuleveut }

\aistatsaddress{  Centre de Math\'{e}matiques Appliqu\'{e}es \\ Ecole
  Polytechnique, France\\ Institut Polytechnique de Paris \And
Univ. Lille, Inria, CNRS,\\Centrale Lille,\\ UMR 9189 - CRIStAL,\\F-59000 Lille, France  \And  Centre de Math\'{e}matiques Appliqu\'{e}es \\ Ecole
  Polytechnique, France\\ Institut Polytechnique de Paris } ]

\begin{abstract}
Federated Learning (FL) is a paradigm for large-scale distributed learning which faces two key challenges:
(i) training efficiently from highly heterogeneous user data, and (ii) protecting the privacy of participating users. In this work, we propose a novel FL approach (DP-SCAFFOLD) to tackle these two challenges together by incorporating Differential Privacy (DP) constraints into the popular SCAFFOLD algorithm. We focus on the challenging setting where users communicate with a ``honest-but-curious'' server without any trusted intermediary, which requires to ensure privacy not only towards a third party observing the final model but also towards the server itself. Using advanced results from DP theory and optimization, we establish the convergence of our algorithm for convex and non-convex objectives.
Our paper clearly highlights the trade-off between utility and privacy and demonstrates the superiority of DP-SCAFFOLD over the state-of-the-art algorithm DP-FedAvg when the number of local updates and the level of heterogeneity grows. Our numerical results confirm our analysis and show that DP-SCAFFOLD provides significant gains in practice.
\end{abstract}

\section{INTRODUCTION}

Federated Learning (FL) enables a set of users with local datasets to collaboratively train a machine learning model without centralizing data \citep{fl_problems}.
Compared to machine learning in the cloud, the promise of FL is to avoid the costs of moving data and to mitigate privacy concerns. Yet, this promise can only be fulfilled if two key challenges are addressed. First, FL algorithms must be able to \emph{efficiently deal with the high heterogeneity of data across users}, which stems from the fact that each local dataset reflects the usage and production patterns specific to a given user. Heterogeneous data may prevent FL algorithms from converging unless they use a large number of communication rounds between the users and the server, which is often considered as a bottleneck in FL \citep{fl_theory_update,fl_scaffold}.
Second, when training data contains sensitive or confidential information, FL algorithms must \emph{provide rigorous privacy guarantees} to ensure that the server (or a third party) cannot accurately reconstruct this information from model updates shared by users \citep{inverting}.
The widely recognized way to quantify such guarantees is Differential Privacy (DP) \citep{dp_foundations}.

Since the seminal \fedavg algorithm proposed by \citet{communication_efficient_fl}, a lot of effort has gone into addressing these two challenges \emph{separately}. FL algorithms like \scaffold \citep{fl_scaffold} and \fedprox \citep{fl_fedprox} can better deal with heterogeneous data, while versions of \fedavg with Differential Privacy (DP) guarantees have been proposed based on the addition of random noise to the model updates \citep{dp_fedavg_user_level,dp_fed_sgd_user_level,triastcyn2019federated}.
Yet, we are not aware of any approach designed to tackle data heterogeneity while ensuring differential privacy, or of any work studying the associated trade-offs. This appears to be a challenging problem: on the one hand, data heterogeneity can hurt the privacy-utility trade-off of DP-FL algorithms (by requiring more communication rounds and thus more noise). On the other hand, it is not clear how to extend existing heterogeneous FL algorithms to satisfy DP and what the resulting privacy-utility trade-off would be in theory and in practice.

Our work precisely aims to tackle the issue of data heterogeneity in the context of FL under DP constraints. We aim to protect the privacy of any user's data against a honest-but-curious server observing all user updates, and against a third party observing only the final model.
We present \dpscaf, a novel differential private FL algorithm for training a global model from heterogeneous data based on \scaffold \citep{fl_scaffold} augmented with the addition of noise in the local model updates. Our convergence analysis leverages a particular initialization of the algorithm, and controls a different set of quantities than in the original proof. 

Relying on recent tools for tightly keeping track of the privacy loss of the subsampled Gaussian mechanism \citep{renyi_dp_sampled_gm_explicit} under Rényi Differential Privacy (RDP) \citep{renyi_dp}, we formally characterize the privacy-utility trade-off of \dpfedavg, considered as the state-of-the-art DP-FL algorithm \citep{dp_fed_sgd_user_level}, and \dpscaf in convex and non-convex regimes. Our results show the superiority of \dpscaf over \dpfedavg when the number of local updates is large and/or the level of heterogeneity is high.
Finally, we provide experiments on simulated and real-world data which confirm our theoretical findings and show that the gains achieved by \dpscaf are significant in practice.

The rest of the paper is organized as follows. Section~\ref{sec:related} reviews some background and related work on FL, data heterogeneity and privacy. Section~\ref{section_setting} describes the problem setting and introduces \dpscaf. In Section~\ref{section_analysis}, we provide theoretical guarantees on both privacy and utility for \dpscaf and \dpfedavg. Finally, Section~\ref{section_experiments} presents the results of our experiments and we conclude with some perspectives for future work in Section~\ref{section_perspectives}. 

\section{RELATED WORK}
\label{sec:related}

\textbf{Federated learning \& heterogeneity.}
The baseline FL algorithm \texttt{FedAvg} \citep{communication_efficient_fl} is known to suffer from instability and convergence issues in heterogeneous settings, related to device variability or non-identically distributed data \citep{fl_theory_update}. In the last case, these issues stem from a \emph{user-drift} in the local updates, which occurs even if all users are available or full-batch gradients are used \citep{fl_scaffold}. Several FL algorithms have been proposed to better tackle heterogeneity. \texttt{FedProx} \citep{fl_fedprox} features a proximal term in the objective function of local updates. However, it is often numerically outperformed by \texttt{SCAFFOLD} \citep{fl_scaffold}, which relies on variance reduction through control variates. In a nutshell, the update direction of the global model at the server ($c$) and the update direction of each user $i$'s local model ($c_i$) are estimated and combined in local Stochastic Gradient Descent (SGD) steps ($c - c_i$) to correct the user-drift (see Section \ref{algo_description} for more details).

\texttt{MIME}~\citep{karimireddy2020mime} also focuses on client heterogeneity and improves on \texttt{SCAFFOLD} by using the \textit{stochastic gradient} evaluated on the global model 
as the local variate ${c}_{i}$ and the synchronized \textit{full-batch gradient}
as the global control variate $c$. However, computing full-batch gradients is very costly in practice. Similarly, incorporating DP noise into \texttt{FedDyn}~\citep{acar2020federated}, which is based on the exact minimization of a proxy function, is not straightforward. On the other hand, the adaptation of \Scaffold~to \DpScaffold~is more natural as control variates only depend on stochastic gradients and thus do not degrade the privacy level throughout the iterations (see details in Section~\ref{subsec:privacy}).

\emph{Extension to other optimization schemes:}
While \texttt{Fed-Opt}~\citep{reddi2020adaptive} generalizes \texttt{FedAvg} by using different optimization methods locally (e.g., \texttt{Adam}~\citep{kingma2014adam}, \texttt{AdaGrad}~\citep{duchi2011adaptive}, etc., instead of vanilla local SGD steps) or a different aggregation on the central server, these methods may also suffer from user-drift. Their main objective is to improve the convergence rate~\citep{wang2021field} without focusing on heterogeneity. We thus choose to focus on the simplest algorithm to highlight the impact of DP and heterogeneity.

\textbf{Federated learning \& differential privacy.} 
Even if datasets remain decentralized in FL, the privacy of users may still be compromised by the fact that the server (which may be ``honest-but-curious'') or a third party has access to model parameters that are exchanged during or after training \citep{fredrikson2015model,shokri2017membership,inverting}. Differential Privacy (DP) \citep{dp_foundations} provides a robust mathematical way to  quantify the information that an algorithm $A$ leaks about its input data. DP relies on a notion of \textit{neighboring datasets}, which in the context of FL may refer to pairs of datasets differing by one user (\textit{user-level} DP) or by one data point of one user (\textit{record-level} DP).

\begin{definition}[Differential Privacy, \citealp{dp_foundations}]\label{definition_dp} Let $\epsilon, \delta>0$. A randomized algorithm $A : \mathcal{X}^n \rightarrow \mathcal{Y}$ is $(\epsilon, \delta)$-DP if for all pairs of neighboring datasets $D,D'$ and every subset $S \subset \mathcal{Y}$, we have:
\begin{align*}
    \mathbb{P}[A(D) \in S] \leq e^{\epsilon}\mathbb{P}[A(D') \in S] + \delta.
\end{align*}
\end{definition}
The privacy level is controlled by the parameters $\epsilon$ and $\delta$ (the lower, the more private).
A standard building block to design DP algorithms is the Gaussian mechanism \citep{dp_foundations}, which adds Gaussian noise to the output of a non-private computation. The variance of the noise is calibrated to the sensitivity of the computation, i.e., the worst-case change (measured in $\ell_2$ norm) in its output on two neighboring datasets. The design of private ML algorithms heavily relies on the Gaussian mechanism to randomize intermediate data-dependent computations (e.g. gradients). The privacy guarantees of the overall procedure are then obtained via \emph{composition} \citep{boosting_dp,composition_dp}. Recent theoretical tools like \textit{Rényi Differential Privacy} \citep{renyi_dp} (see Appendix \ref{appendix:privacy}) allow to obtain tighter privacy bounds for the Gaussian mechanism under composition and data subsampling \citep{renyi_dp_sampled_gm_explicit}.



In the context of FL, the output of an algorithm $A$ in the sense of Definition~\ref{definition_dp} contains all information observed by the party we aim to protect against. Some work considered a trusted server and thus only protect against a third party who observes the final model. In this setting, \cite{dp_fedavg_user_level} introduced \dpfedavg and \dpfedsgd (i.e., \dpfedavg with a single local update), which was also proposed independently by \cite{dp_fed_sgd_user_level}. These algorithms extend \fedavg and \fedsgd by having the server add Gaussian noise to the aggregated user updates. \cite{triastcyn2019federated} used a relaxation of DP known as Bayesian DP to provide sharper privacy loss bounds. However, these papers do not discuss the theoretical trade-off between utility and privacy.
Some recent work by \cite{dp_federated_averaging_nbafl} has formally examined this trade-off for \dpfedsgd, providing a utility guarantee for strongly convex loss functions. However, they do not consider multiple local updates.

Some papers also considered the setting with a ``honest-but-curious'' server, where users must randomize their updates locally before sharing them. 
This corresponds to a stronger version of DP, referred to as \textit{Local Differential Privacy} (LDP) \citep{ldp_baseline, ldp_mechanisms, ldp_mechanism_tradeoff}. \dpfedavg and \dpfedsgd can be easily adapted to this setting by pushing the Gaussian noise addition to the users, which induces a cost in utility. \cite{ldp_mechanisms} consider \dpfedsgd in this setting but do not provide any utility analysis. \cite{fl_shuffle} provide utility and compression guarantees for variants of \dpfedsgd in an intermediate model where a trusted \textit{shuffler} between the server and the users randomly permutes the user contributions, which is known to amplify privacy \citep{dp_privacy_blanket_shuffle_model, dp_via_shuffling,dp_shuffled_model, dp_amplification_by_shuffling}. However, both of these studies do not consider multiple local updates, which is key to reduce the number of communication rounds. \cite{dp_secure_federated_averaging} consider the server as ``honest-but-curious'' but does not ensure end-to-end privacy to the users.
Finally, \cite{dp_personalized_fl} present a personalized DP-FL approach as a way to tackle data heterogeneity, but it is limited to linear models.

\textbf{Summary.}
To the best of our knowledge, there exists no FL approach designed to tackle \textit{data heterogeneity under DP constraints}, or any study of existing DP-FL algorithms capturing the impact of data heterogeneity on the privacy-utility trade-off.



\section{DP-SCAFFOLD}
\label{section_setting}

In this section, we first describe the framework that we consider for FL and DP, before giving a detailed description of \texttt{DP-SCAFFOLD}.
A table summarizing all notations is provided in Appendix~\ref{app:complementary}.

\subsection{Federated Learning Framework} 
We consider a setting with a central server and $M$ users. Each user $i \in [M]$, holds a private local dataset $D_i=\{d^i_1,...,d^i_\r\} \subset \mathcal{X}^\r$, composed of $\r$ observations living in a space $\mathcal{X}$. We denote by $D:=D_1 \sqcup ... \sqcup D_M$ the disjoint union of all user datasets.  
Each dataset $D_i$ is supposed to be independently sampled from distinct distributions. The objective is to solve the following empirical risk minimization problem over parameter $x$:
\begin{align*}
   \textstyle \min_{x \in \mathbb{R}^d} F(x):=\frac{1}{M}\sum_{i=1}^M F_i(x),
\end{align*} 
where $F_i(x):=\frac{1}{\r} \sum_{j=1}^{\r} f_i(x,d^{i}_j)$ is the empirical risk on user $i$, and for all $x\in \mathbb{R}^d$ and $d\in \mathcal X$, $f_i(x,d)$ is the loss of the model $x$ on observation $d$.
We  denote by $\nabla f_i(x,d^{i}_j) $ the gradient of the loss $f_i$ computed on a sample $d^i_j \in D_i$, and by extension, for any $S_i\subset D_i$,  $\nabla f_i(x,S_i) := \frac{1}{|S_i|}\sum_{j\in S_i}\nabla f_i(x,d^{i}_j)  $ is the  averaged mini-batch gradient. We note that our results can easily be adapted to optimize any weighted average of the loss functions and to imbalanced local datasets.

\begin{restatable}[t]{algorithm}{HBK}
\DontPrintSemicolon
\caption{DP-SCAFFOLD$(T,K,l,s,\sigma_g,\mathcal{C})$ \label{algo_dp_scaffold}}
  \KwInputServer{initial $x^0$, initial $c^0$}
  \KwInputClient{initial $c_i^0$}
  \KwOutput{$x^T$}
  \For{$t=1,...,T$}
    {
        \HiLiRed\textbf{User subsampling by the server:}\\~~Sample $C^t \subset [M]$ of size $\lfloor lM\rfloor$ \\
        \textbf{Server sends} $(x^{t-1},c^{t-1})$ to users $i \in C^t$\\
        \For{user $i \in C^t$}
        {
            Initialize model: $y_i^0 \leftarrow x^{t-1}$\\
            \For{$k=1,...,K$}
            {   
                \HiLiRed\textbf{Data subsampling by user $i$: }Sample $S_i^k \subset D_i$ of size $\lfloor sR\rfloor$\\
                \For{sample $j \in S_i^k$}
                {
                    \textbf{Compute gradient:} $g_{ij} \leftarrow \nabla f_i(y_i^{k-1},d^i_{j}) $\\
                    \HiLiYellow \textbf{Clip gradient:} $\tilde{g}_{ij} \leftarrow g_{ij} / \max \big(1, ||g_{ij}||_2/\mathcal{C}\big)$
                    }
                \HiLiYellow\textbf{Add DP noise to local gradients:}\\  ~~$\tilde{H}_i^k \leftarrow \frac{1}{s \r}\sum_{j \in S_i^k} \tilde{g}_{ij} + \frac{2\mathcal{C}}{s\r}\mathcal{N} (0,\sigma^2_g)$ \label{noised_SGD_step}\\
                ~~$y_i^{k} \leftarrow y_i^{k-1} - \eta_l(\tilde{H}_i^k -c_i^{t-1} + c^{t-1})$}
            $\tilde{c}_i^t \leftarrow c_i^{t-1} - c^{t-1} + \frac{1}{K \eta_l}( x^{t-1} - y_i^K)$\\
            $(\Delta y_i^t, \Delta c_i^t) \leftarrow (y_i^K - x^{t-1}, \tilde{c}_i^t -c_i^{t-1})$\\
            \textbf{User $i$ sends to server} $(\Delta y_i^t, \Delta c_i^t)$\\
            $c_i^{t} \leftarrow \tilde{c}_i^t$
            }
        \textbf{Server aggregates}: \\
        $(\Delta x^{t}, \Delta c^{t}) \leftarrow \frac{1}{lM}\sum_{i \in C^t}(\Delta y_i^t, \Delta c_i^t)$\label{aggregation_step0}\\
        $x^{t} \leftarrow x^{t-1} + \eta_g\Delta x^t$, $c^{t} \leftarrow c^{t-1} + l\Delta c^t$ \label{aggregation_step}
    }
\end{restatable}

\subsection{Privacy Model}
\label{dp_framework}

We aim at controlling the information leakage from individual datasets $D_i$ in the updates shared by the users. 
For simplicity, our analysis focuses on \textit{record-level} DP with respect to (w.r.t) the joint dataset $D$. We thus consider the following notion of neighborhood:
$D,D'\in\mathcal{X}^{M\r}$  are \textit{neighboring datasets} (denoted $||D-D'|| \leq 1$) if they differ by at most one record, that is if there exists at most one $i \in [M]$ such that $D_i$ and $D_i'$ differ by one record.
We want to ensure privacy (or quantify privacy level) (i) towards a third party observing the final model and (ii) towards an honest-but-curious server. 
Our DP budget is set in advance and denoted by $(\epsilon,\delta)$, 
and corresponds to the desired level of privacy towards a third party observing the final model (or any model during the training process).\footnote{The use of composition for analyzing the privacy guarantee for the final model implies that the same guarantee holds even if every intermediate \emph{global} model is observed.}
We will also report the corresponding (weaker) DP guarantees towards the server.

\subsection{Description of \texttt{DP-SCAFFOLD}} 
\label{algo_description}

We now explain how our algorithm \dpscaf is constructed.
\dpscaf proceeds similarly as standard FL algorithms like \fedavg: all users perform a number of local updates $K$, before communicating with the central server. We denote $T$ the  number of communication rounds. As \scaffold, \dpscaf relies on the use of control variates that are updated throughout the iterations of the algorithm: (i) on the server side ($c$, downloaded by the users) and (ii) on the user side ($\{c_i\}_{i\in [M]}$, uploaded to the server).

At any round $t\in [T]$, a subset $C^t$ of users with cardinality $\lfloor lM\rfloor$ is uniformly selected by the server, where $l$ is the user sampling ratio. 
Each user  $i \in C^t$ downloads the global model  $x^{t-1}$ held by the central server and performs $K$ local updates on their local copy $y_i$ of the model (with step-size $\eta_l\geq 0$), starting from $y_i^0=x^{t-1}$. 

At iteration $k\in [K]$, user $i\in C^t$ samples an independent $\lfloor s\r\rfloor$-mini-batch of data $S_i^k \subset D_i$, where $s$ is the data sampling ratio. Given a clipping parameter $\mathcal{C}>0$, for all $j\in S_i^k$, the gradient $\nabla f_i(y_i^{k-1},d^i_{j})$ is computed and clipped at threshold $\mathcal{C}$ \citep{dp_deep_l}, giving $\tilde{g}_{ij}$.  
The resulting average stochastic gradient $H_i^k(y_i^{k-1})$ is made private w.r.t. $D_i$ using Gaussian noise calibrated to the $\ell_2$-sensitivity $S=2\mathcal{C}/s\r$ and to the scale $\sigma_g$ (a parameter which will depend on the privacy budget),
giving $\tilde{H}_i^k(y_i^{k-1})$ such that 
\begin{equation*}
    \tilde{H}_i^k(y_i^{k-1}):= H_i^k(y_i^{k-1}) + S \mathcal N (0,\sigma^2_g).
\end{equation*}
Finally, we update the model $y_i^{k-1}$ (omitting index $t$):
\begin{align}\label{eq:update_y_i}
    y_i^k \leftarrow y_i^{k-1} - \eta_l \big(\underbrace{\textcolor{orange!80!black}{ \tilde{H}_i^k(y_i^{k-1}) }}_{\text{``noisy'' gradient}} + \underbrace{\textcolor{green!80!black}{c^{t-1} - c_i^{t-1}}}_{\text{drift correction}}\big),
\end{align}
using the control variates which are updated at the end of each inner loop:
\begin{align*}
    c_i^t & \leftarrow c_i^{t-1} - c^{t-1} + \frac{1}{K \eta_l}( x^{t-1} - y_i^{K}) = \frac{1}{K}\sum_{k=1}^K \tilde{H}_i^k(y_i^{k-1}).
\end{align*}
After $K$ local iterations, each user communicates  $(y_i^K -x^{t-1})$ and $(c_i^t -c_i^{t-1})$ to the central server, and updates the global model with step-size $\eta_g$, as
described in Step~\ref{aggregation_step} of Alg.~\ref{algo_dp_scaffold}.

From the privacy point of view, the updates $(\Delta y_i, \Delta c_i)$ that are transmitted to the server are private w.r.t. $D$ (proved in Section~\ref{subsec:privacy}), thus making private $(x,c)$ w.r.t $D$ by postprocessing. 

The complete pseudo-code is given in Algorithm~\ref{algo_dp_scaffold}. Subsampling steps, which amplify privacy \citep{dp_subsampling}, are highlighted in red, and steps specifically related to DP are highlighted in yellow.
Setting $\sigma_g=0$ and $\mathcal{C}=\infty$ recovers the classical \scaffold algorithm, and removing control variates (i.e., setting $c_i^t$ to 0 for all $t\in [T], i\in [M]$) recovers \dpfedavg, which we describe in Appendix~\ref{app:complementary} (Algorithm~\ref{algo_dp_fedavg}) for completeness.



\textbf{Intuition for control variates.}
In \texttt{SCAFFOLD}, the local control variate $c_i$ converges to the local gradient $\nabla f_i(x^*)$ at the optimal, while $c$ approximates $\frac{1}{M}\sum_{i=1}^M c_i$ \citep[Appendix E]{fl_scaffold}. Therefore, adding $(c-c_i)$ in the update balances the local stochastic gradient and limits user-drift. 


\textbf{Warm-start version of \texttt{DP-SCAFFOLD}.} 
We adapt the \textit{warm-start strategy} from \citet[Appendix E]{fl_scaffold} to accommodate DP constraints, leading to \texttt{DP-SCAFFOLD-warm}. The first few rounds of communication are saved to set\footnote{This happens with high probability: typically, after $4/l$ where $l=o(1)$, all users have been selected at least once with probability $1 -e^{-4} \approx 0.98$.}
the initial values of the control variates to $c^0= \frac{1}{M} \sum_{i=1}^M c_i^0$, with $c_i^0 = \frac{1}{K}\sum_{k=1}^K \tilde{H}_i^k(x^0)$ (perturbed by DP-noise), without updating the global model.
Note that as we leverage user sampling in the privacy analysis, the server cannot communicate with all users at a single round and  the users have to be \textit{randomly} picked  to ensure privacy. 
We prove the convergence of \texttt{DP-SCAFFOLD-warm} in Section~\ref{subsec:utility} (assuming that every user
participated to the warm-start phase). Our experiments in Section \ref{sec:experiments} are conducted with this version of \texttt{DP-SCAFFOLD}.

\textbf{User-level privacy.} Our framework can easily be adapted to \textit{user-level} privacy, by setting $S=2\mathcal{C}/s$.

\section{THEORETICAL ANALYSIS}
\label{section_analysis}
We first provide the analysis of the privacy level in Section~\ref{subsec:privacy}, then analyze utility in Section~\ref{subsec:utility}.

\subsection{Privacy} \label{subsec:privacy}
We first establish that the setting of our algorithms \texttt{DP-SCAFFOLD} and \texttt{DP-FedAvg} enables a fair comparison in terms of privacy.

\begin{claim} \label{lemma_privacy} For a given noise scale $\sigma_g >0$, $x^t$ has the same level of privacy at any round $t \in [T]$ in \texttt{DP-SCAFFOLD(-warm)} and \texttt{DP-FedAvg} after the server aggregation.
\end{claim}
This claim can be proved by induction, see Appendix~\ref{appendix:privacy}.
Consequently, the analysis of privacy is similar for \texttt{DP-FedAvg} or \texttt{DP-SCAFFOLD}. Theorem~\ref{theorem_privacy} gives the order of magnitude of $\sigma_g$ (same for \texttt{DP-FedAvg} and \texttt{DP-SCAFFOLD})  to ensure DP towards the server or any third party. Similar to previous work \citep[see e.g.,][]{fl_shuffle}, the results presented below consider the following regime, as it allows to obtain simple closed forms for the privacy guarantees in Theorem~\ref{theorem_privacy}.

\begin{assumption}\label{ass:dplebvels}
We consider a noise level $\sigma_g$, a privacy budget $\epsilon>0$ and a data-subsampling ratio $s$ s.t.:
(i) $s=o(1)$, (ii) $\epsilon<1$ and (iii) $\sigma_g=\Omega(s\sqrt{K/\log(2Tl/\delta)})$ (high privacy regime).
\end{assumption}

Note that our analysis does not require Assumption~\ref{ass:dplebvels}, but the resulting expressions and the dependency on the key parameters are then difficult to interpret. This assumption is actually not used in our experiments, where we compute the privacy loss numerically using the complete formulas from our proof.

\begin{theorem}[Privacy guarantee]\label{theorem_privacy} Let $\epsilon,\delta >0$. Under Assumption~\ref{ass:dplebvels},  set $\sigma_g=\Omega\big(s\sqrt{lTK\log(2Tl/\delta)\log(2/\delta)}/\epsilon\sqrt{M}\big)$. Then, for \texttt{DP-SCAFFOLD(-warm)} and \texttt{DP-FedAvg}, $x^T$ is
(1) $\big(\mathcal{O}(\epsilon),\delta\big)$-DP  towards a third party, (2) $\big(\mathcal{O}(\epsilon_s),\delta_s\big)$-DP towards the server, where $\epsilon_s=\epsilon\sqrt{{M}/{l}}$ and $\delta_s=\frac{\delta}{2}(\frac{1}{l} + 1)$.
\end{theorem}

\textbf{Sketch of proof.} We here summarize the main steps of the proof. Let $\sigma_g$ be a given DP noise level. Our proof stands for the privacy analysis over a query function of sensitivity 1 (since calibration is made with constant $S$ in Section~\ref{dp_framework}). We denote GM($\sigma_g$) the corresponding Gaussian mechanism. We first provide the result for any third party.

We combine the following steps:
\begin{itemize}[topsep=0pt,itemsep=1pt,leftmargin=*,noitemsep,wide]
    \item \emph{Data-subsampling with Rényi DP}. Let $t\in[T]$ be an arbitrary round. We first estimate an upper DP bound $\epsilon_a$ (w.r.t. $D$) of the privacy loss after the \textit{aggregation} by the server of $lM$ individual contributions (Step~\ref{aggregation_step} in Alg.~\ref{algo_dp_scaffold}). Those are \textit{private} w.r.t. to the corresponding local datasets, say $(\alpha, \epsilon_i)$-RDP w.r.t. $D_i$ where $i \in C^t$ stands for the $i$-th user, each one being the result of the composition of $K$ adaptative $s$-subsampled GM($\sigma_g$). For any $\alpha >1$, 
    we know that GM($\sigma_g$) is $(\alpha, \alpha/2\sigma_g^2)$-RDP \citep{renyi_dp}. \cite{renyi_dp_sampled_gm_explicit} proves that the $s$-subsampled GM($\sigma_g$) is $(\alpha, \mathcal{O}(s^2\alpha/\sigma_g^2))$-RDP under Assumption~\ref{ass:dplebvels}-(i). By the RDP composition rule over the $K$ local iterations, we have $\epsilon_i(\alpha) \leq \mathcal{O}(Ks^2\alpha/\sigma_g^2)$. Therefore, the aggregation over all users considered in $C^t$ is private w.r.t. $D$ with a corresponding Gaussian noise of variance $S^2\sigma_a^2$ where $\sigma_a^2=\frac{1}{lM}\frac{\sigma_g^2}{K s^2}$ (mean of independent Gaussian noises). Yet, making the whole aggregation private w.r.t. $D$ only requires a $l_2$ calibration equal to $S'=S/lM$ (by triangle inequality) which means we can quantify the gain of privacy as $(\alpha, \mathcal{O}(Ks^2\alpha/lM\sigma_g^2))$-RDP. After converting this result into a DP bound \citep{renyi_dp}, we get that for any $\delta'>0$, the whole mechanism is $(\epsilon_a(\alpha, \delta'), \delta')$-DP where $\epsilon_a(\alpha, \delta')=\mathcal{O}\big(\frac{Ks^2\alpha}{lM\sigma_g^2} + \frac{\log(1/\delta')}{\alpha-1}\big)$. 
    \item \emph{User-subsampling with DP.}  In order to get explicit bounds (that may not be optimal), we then use classical DP tools to estimate an upper DP bound $\epsilon_T$ after $T$ rounds.
    By combining amplification by subsampling results \citep{dp_subsampling} over users and strong composition \citep{composition_dp} (with Assumption~\ref{ass:dplebvels}-(ii)) over communication rounds, we finally get that, for any $\delta''>0$, $x^T$ is $(\epsilon_T(\alpha, \delta',\delta''), T l \delta' + \delta'')$-DP where $\epsilon_T(\alpha, \delta',\delta'')=\mathcal{O}(l\epsilon_a(\alpha, \delta')\sqrt{T\log(1/\delta'')})$.
\item \emph{Fixing parameters.} Considering our final privacy budget $\delta$ for any third party, we fix $\delta':=\delta/2Tl$ and $\delta'':=\delta/2$. 
Following the method of the \textit{Moments Accountant} \citep{dp_deep_l}, we minimize the bound on $\epsilon_T$ w.r.t. $\alpha >1$, which gives that $\epsilon_T=\mathcal{O}(\tilde{\epsilon})$ where 
\begin{align*}
    \tilde{\epsilon}=l\sqrt{T\log(2/\delta)}\bigg(\frac{s\sqrt{K\log(2Tl/\delta)}}{\sigma_g\sqrt{lM}}+\frac{Ks^2}{lM\sigma_g^2}\bigg).
\end{align*}
\end{itemize}
Finally, under Assumption~\ref{ass:dplebvels}-(iii), the second term is bounded by the first one. We then invert the formula of this upper bound of $\tilde{\epsilon}$ to express $\sigma_g$ as a function of a given privacy budget $\epsilon$, proving the first statement. 

To prove the second statement, we recall that the server has access to individual contributions before aggregation (which prevents a reduction by a factor $lM$ of the variance) and that it knows the selected users at each round, which cancels the user-sampling effect (factor $l$). We refer to Appendix~\ref{appendix:privacy} for the full proof as well as the non-asymptotic (tighter) formulas.

\textbf{Remarks on privacy accounting.} The RDP analysis conducted to handle \textit{data} subsampling allows to limit the impact of $K$ in the expression of $\sigma_g$. A standard analysis would require a noise level increased by an extra factor $\mathcal{O}(\sqrt{\log(TKls/\delta)})$. On the other hand, we tracked the privacy loss over the communication rounds using standard strong composition \citep{boosting_dp}, which gives a closed-form expression but is often sub-optimal in practice. In our experiments, we use RDP upper bounds to calibrate $\sigma_g$ more tightly.
We refer to Appendix~\ref{appendix:privacy} for more details.

\textbf{Extension to other frameworks.}
Instead of the Gaussian mechanism, other randomizers could be applied, possibly to the \textit{per-example} gradients. The privacy analysis would be then similar to ours as long as a tight RDP bound on the subsampling of this mechanism is provided \citep[see the work of][for more details]{renyi_dp_sampled_gm_explicit}. Otherwise, classic DP results for composition and subsampling must be used instead. Besides this, our analysis could be extended to the use of a shuffler between the users and the server to amplify privacy guarantees. For instance, one could use a recent RDP result for shuffled subsampled pure DP mechanisms \cite[Theorem 1]{girgis2021renyi}.


\subsection{Utility}\label{subsec:utility}

We denote by $\|\cdot\|$ the Euclidean $\ell_2$-norm. We assume that $F$ is bounded from below by $F^*= F(x^*)$, for an $x^* \in \mathbb{R}^d$. 
Furthermore, we make standard assumptions on the functions $(F_i)_{i \in [M]}$.
\begin{assumption}\label{ass:smooth_f_i}
For all $i \in [M]$, $F_i$ is differentiable and $\nu$-smooth (i.e., $\nabla F_i$ is $\nu$-Lipschitz). 
\end{assumption}
We also make the following assumption on the stochastic gradients and data sampling.
\begin{assumption}\label{ass:stoch_grad}
For any iteration $t\in [T], k\in [K]$,  
\begin{enumerate}[topsep=0pt,itemsep=1pt,leftmargin=*,noitemsep,wide]
    \item the stochastic gradient $\nabla f_i(y_i^{k-1},d^i_{j}) $ is conditionally unbiased, i.e., $\mathbb E_{d^i_{j}} [\nabla f_i(y_i^{k-1},d^i_{j}) |  y_i^{k-1}] = \nabla F_i(y_i^{k-1})$. 
    \item the stochastic gradient has bounded variance, i.e., for any $y\in \mathbb R^d$, $\mathbb E_{d^i_{j}} [\|\nabla f_i(y,d^i_{j}) - \nabla F_i(y) \|^2] \le \sigmaoracle^2$.
    \item there exists a clipping constant $\mathcal{C}$ independent of $i,j$ such that  $\| \nabla f_i(y_i^{k-1},d^i_{j}) \| \le \mathcal{C}$.
\end{enumerate}
\end{assumption}
The first condition is naturally satisfied when $d^i_j$ is uniformly sampled in $[\r]$. The second condition is classical in the literature, and can be relaxed to only assume that the noise is bounded at the optimal point~$x^*$~\citep{gower_sgd_2019}. Remark that consequently, the variance of a mini-batch of size $s\r$ uniformly sampled over $D_i$ is upper bounded by $\sigmaoracle^2/s\r$. Finally, the third point ensures that we can safely ignore the impact of gradient clipping.

 Lastly, to obtain a convergence guarantee for \texttt{DP-FedAvg} (but not for \texttt{DP-SCAFFOLD}), we use Assumption \ref{assumption_bgd}
on the data-heterogeneity, which bounds gradients $\nabla f_i$ towards $\nabla f$.

\begin{assumption}[Bounded Gradient dissimilarity]\label{assumption_bgd} There exist constants  $G\geq0$ and $B\geq1$ such that:
\begin{align*}
    \textstyle\forall x \in \mathbb{R}^d, \frac{1}{M}\sum_{i=1}^M||\nabla F_i(x)||^2 \leq G^2 + B^2 ||\nabla F(x)||^2.
\end{align*}
\end{assumption}
Quantifying the heterogeneity between users by controlling the difference between the local gradients and the global one is classical in federated optimization~\citep[e.g.][]{fl_problems}. 
  We can now state a utility result in the convex case, by considering $ \textcolor{purple!80!black}{\sigma_g^* :=s\sqrt{lTK\log(2Tl/\delta)\log(2/\delta)}/\epsilon\sqrt{M}}$ (order of magnitude of noise scale to approximately ensure end-to-end $(\epsilon, \delta)$-DP w.r.t. $D$ according to Theorem~\ref{theorem_privacy}). This result is extended to the strongly convex and nonconvex cases in Appendix~\ref{app:othertheoreticalresults}. 

\begin{theorem}[Utility result - convex case] \label{utility_result} 
Assume that for all $i\in [M]$, $F_i$ is convex. Let $x^0 \in \mathbb R ^d$ and  denote $D_0:=||x^0 -x^*||$. 
Under Assumptions~\ref{ass:smooth_f_i} and \ref{ass:stoch_grad}, we  consider the sequence of iterates $(x^t)_{t\geq0}$ of Algorithm \ref{algo_dp_scaffold} (\dpscaf) and Algorithm \ref{algo_dp_fedavg} (\dpfedavg), starting from $x^0$, and with DP noise $\sigma_g:=\sigma_g^*$.
Then there exist step-sizes $(\eta_g,\eta_l)$ and weights $(w_t)_{t\in [T]}$ such that the expected excess of loss $\mathbb{E}[F(\overline{x}^T)] - F^*$, where $\overline{x}^T=\sum_{t=1}^T w_t x^t$, is bounded by:
\begin{itemize}[leftmargin=*]
    \item  For \texttt{DP-FedAvg}, under  Assumption~\ref{assumption_bgd}: 
    \begin{align*}
        &\footnotesize\mathcal{O}\bigg(\displaystyle
        \underbrace{\textcolor{orange!80!black}{\frac{D_0\mathcal{C}\sqrt{d\log(Tl/\delta)\log(1/\delta)}}{\epsilon M\r}}}_{\textstyle \text{privacy bound}}+ \\ 
        &\footnotesize\underbrace{\textcolor{green!50!black}{\frac{\sigmaoracle D_0}{\sqrt{s\r lMKT}}+   \frac{B^2\nu D_0^2}{T} }+ \textcolor{blue}{ \frac{GD_0\sqrt{1-l}}{\sqrt{lMT}} + \frac{D_0^{4/3}\nu^{1/3}G^{2/3}}{T^{2/3}}} }_{\textstyle \text{optimization bound}} \bigg).
    \end{align*}
    \item For \texttt{DP-SCAFFOLD-warm}:
    \begin{align*}
        \footnotesize\mathcal{O}\bigg(\displaystyle \underbrace{\textcolor{orange!80!black}{\frac{D_0\mathcal{C}\sqrt{d\log(Tl/\delta)\log(1/\delta)}}{\epsilon M\r}}}_{\textstyle \text{privacy bound}} + \underbrace{\textcolor{green!50!black}{\frac{\sigmaoracle D_0}{\sqrt{s\r lMKT}} +\frac{\nu D_0^2}{l^{2/3}T}}}_{\textstyle \text{optimization bound}}\bigg).
    \end{align*}
\end{itemize}
\end{theorem}

The two bounds given in Theorem~\ref{utility_result} consist of three and two terms respectively: 
\begin{enumerate}[topsep=0pt,itemsep=1pt,leftmargin=*,noitemsep,wide]
    \item A classical  convergence rate resulting from (non-private) first order optimization, highlighted in \textcolor{green!50!black}{green}. The dominant part, as $T\to \infty$, is $\textcolor{green!50!black}{\frac{\sigmaoracle D_0}{\sqrt{s\r lMKT}}}$. This term is inversely proportional to the square root of the \textit{total number of iterations $TK$} times the \textit{average number of gradients computed per iteration $lM \times s \r$}, and increases proportionally to the stochastic gradients' standard deviation $\sigmaoracle$ and the initial distance to the optimal point~$D_0$. 
    \item An extra term,  in \textcolor{blue}{blue}, showing that heterogeneity hinders the convergence of \texttt{DP-FedAvg}, for which Assumption \ref{assumption_bgd} is required. Here, as $T\to \infty$, the dominant term in it is $\textcolor{blue}{\frac{GD_0\sqrt{1-l}}{\sqrt{lMT}}}$, except if the user sampling ratio $l=1$, then the dominating term becomes   $\textcolor{blue}{ \frac{D_0^{4/3}\nu^{1/3}G^{2/3}}{T^{2/3}}}$.  Both these terms do not decrease with the number of local iterations~$K$, and increase with heterogeneity constant~$G$.
    This extra term for \texttt{DP-FedAvg} highlights the \textbf{superiority of \texttt{DP-SCAFFOLD} over \texttt{DP-FedAvg} under data heterogeneity.}
    
    \item Lastly, an \textcolor{orange!80!black}{additional term} showing the impact of DP appears. This term is diverging with the number of iterations $T$, which results in the privacy-utility trade-off on~$T$. Moreover, this term decreases proportionally to the whole number of data records $M\r$. It also outlines the cost of DP since it sublinearly grows with the size of the model~$d$ and dramatically increases inversely to the DP budget~$\epsilon$.  
\end{enumerate}

\textbf{Take-away messages.} Our analysis highlights that: (i) \texttt{DP-SCAFFOLD} improves on \texttt{DP-FedAvg} in the presence of heterogeneity; and (ii) increasing the number of local updates $K$ is very profitable to \texttt{DP-SCAFFOLD}, as it improves the \textcolor{green!50!black}{dominating optimization bound}  without degrading the \textcolor{orange!80!black}{privacy bound}.
These aspects are numerically confirmed in Section \ref{section_experiments}.
    
 \textbf{Sketch of proof and originality.}
To establish Theorem \ref{utility_result}, we adapt the proof of Theorems V and VII in \citet{fl_scaffold}. However, we consider a weakened assumption on stochastic gradients due the addition of Gaussian noise in the local updates. Consequently,  in order to limit the impact of this additional noise, we  change the quantity (Lyapunov function) that is controlled during the proof: we combine the squared distance to the optimal point $\|x^t-x_*\|^2$ to a control of the lag at iteration $t$,  $\frac{1}{K M} \sum_{k=1}^{K} \sum_{i=1}^{M} \mathbb{E}\|\alpha_{i, k-1}^{t}-x^{t} \|^{2}$; where we ensure that our control variates $(c_{i}^{t})_{i\in [M]}$ at iteration $t$ correspond to noisy stochastic gradients measured at points $(\alpha_{i, k-1}^{t})_{i\in [M]}$, that is,  $c_{i}^{t}=\frac{1}{K} \sum_{k=1}^{K} \tilde{H}_{i}^k\left(\alpha_{i, k-1}^{t}\right)$.\footnote{In contrast, the proof in the convex case in \citet{fl_scaffold} controls $\frac{1}{M} \sum_{i=1}^{M} \mathbb{E}\|c_i^t - \nabla f_i(x^*) \|^{2}$.}
This proof is detailed in Appendix~\ref{app:othertheoreticalresults}.

Relying directly on the result in \citet{fl_scaffold} would require to devote a large fraction (e.g., half) of the privacy budget to the initialization phase to obtain a reasonable bound. Such a strategy did not perform well in experiments.

\textbf{On the warm-start strategy.}
To obtain the utility result, we have to ensure that initial users' controls $c_i^0$ are set as follows: $c_i^0=\frac{1}{K}\sum_{k=1}^K \tilde{H}_i^k(x^0)$ (notations of Alg. \ref{algo_dp_scaffold}). Our theoretical result thus only holds for the \texttt{DP-SCAFFOLD-warm} version. However, we observed in our experiments that \texttt{DP-SCAFFOLD} (which uses initial user control variates equal to $0$) led to the same results as \texttt{DP-SCAFFOLD-warm}.

\textbf{Extension to other local randomizers.}
Our utility analysis would easily extend to any unbiased  mechanism with explicit variance (see Appendix~\ref{app:othertheoreticalresults}).

\section{EXPERIMENTS}
\label{sec:experiments}
\label{section_experiments}

\textbf{Experimental setup.} 
In our experiments,\footnote{Code available on \href{https://github.com/maxencenoble/Differential-Privacy-for-Heterogeneous-Federated-Learning}{Github}.} we perform federated classification with two models: (i) logistic regression (\textit{LogReg}) for synthetic and real-world data, and (ii) a deep neural network with one hidden layer (\textit{DNN}) (see Appendix~\ref{app:setting_algo} for the precise architecture) on real-world data. We fix the global step-size $\eta_g=1$, local step-size $\eta_l=\eta_0/sK$ where $\eta_0$ is carefully tuned (see Appendix~\ref{app:setting_algo}), and use a $\ell_2$-regularization parameter set to $5.10^{-3}$. Regarding privacy, we fix $\delta=1/M\r$ in all experiments. Then, for each setting, once the parameters related to sampling and number of iterations are fixed, we calculate the corresponding privacy bound $\epsilon$ by using non-asymptotic upper bounds from RDP theory (see Section~\ref{subsec:privacy}). Details on the clipping heuristic are given in Appendix~\ref{app:setting_algo}.  We report average results over 3 random runs.

\textbf{Datasets.} 
For synthetic data, we follow the data generation setup of \cite{fl_fedprox}, which enables to control heterogeneity between users' local models and between users' data distributions, respectively with parameters $\alpha$ and $\beta$ (the higher, the more heterogeneous). Note that the setting $(\alpha, \beta)=(0,0)$ still creates heterogeneity and does not lead to i.i.d. data across users. Our data is generated from a logistic regression design with $10$ classes, with input dimension $d'=40$. We consider $M=100$ users, each holding $\r=5000$ samples. We compare three levels of heterogeneity: ($\alpha,\beta)\in\{(0,0),(1,1),(5,5)\}$. Details on data generation are given in Appendix~\ref{app:data_gen}.

We also conduct experiments on the EMNIST-`balanced' dataset \citep{emnist}, which consists of 47 balanced classes (letters and digits) containing all together $131,600$ samples. The dataset is divided into $M=40$ users, who each have $R=2500$ data records. Heterogeneity is controlled by parameter~$\gamma$. For $\gamma\%$ similar data, we allocate to each user $\gamma\%$ i.i.d.~data and the remaining $(100-\gamma)\%$ by sorting according to the label \citep{measuring_heterogeneity}, which corresponds to `FEMNIST' (\textit{Federated EMNIST}). For experiments involving DNN, we rather use the seminal MNIST dataset, which features $60,000$ samples labeled by one of the 10 balanced classes. All of the samples are allocated between $M=60$ users (thus $R=1000$).
For both datasets, we consider heterogeneity levels $\gamma\in \{0\%, 10\%, 100\%\}$.

We split each dataset in train/test sets with proportion $80\%/20\%$. Features are first standardized, then each data point is normalized to have unit L2 norm.


\textbf{Superiority of \texttt{DP-SCAFFOLD}.}
We first study the performance of different algorithms under varying levels of data heterogeneity and number of local updates. 
We set subsampling parameters to $l=0.2$ and $s=0.2$ for all of the datasets and fix the noise level $\sigma_g=60$ for synthetic data, and $\sigma_g=30$ for real-world data.
We compare 6 algorithms: \texttt{FedAvg}, \texttt{FedSGD} (\texttt{FedAvg} with $Ks=1$), \texttt{SCAFFOLD(-warm)}, with and without DP.
The results for LogReg (\textit{convex objective}) with $T=400$ are shown in Figure~\ref{logistic_heterogene_K50_K100} for synthetic data and Figure~\ref{femnist_and_mnist_heterogene_K50} (top row) for FEMNIST.
Figure~\ref{femnist_and_mnist_heterogene_K50} (bottom row) shows 
results for DNN (\textit{non-convex objective}) with $T=100$ on MNIST data.
We report in the figure caption the corresponding privacy bound $\epsilon$ for the last iterate with respect to a third party.



In both convex and non-convex settings,
\texttt{DP-SCAFFOLD} clearly outperforms \texttt{DP-FedAvg} and \texttt{DP-FedSGD} under data heterogeneity. The performance gap also increases with the number $K$ of local updates, see Figure~\ref{logistic_heterogene_K50_K100}.
These results confirm our theoretical results: they show that the control variates of \texttt{DP-SCAFFOLD} are robust to noise, and allow to overcome the limitations of \texttt{DP-FedAvg} under high heterogeneity and many local updates. 

\textbf{Trade-offs between parameters.} In \texttt{DP-SCAFFOLD}, a fixed guarantee $\epsilon$ can be achieved by different combinations of values for $K$, $T$  and $\sigma_g$, as shown in Theorem~\ref{utility_result}). We propose to empirically observe these trade-offs  on synthetic data under a high privacy regime $(\epsilon=3)$. The sampling parameters are fixed to $l=0.05, s=0.2$. Given $\sigma_g$ and $K$, we calculate the maximal value of $T$ such that the privacy bound is still maintained after $T$ communication rounds. Table~\ref{table_experiments_varying_k_sigma_gaussian_5_5} shows the test accuracy obtained after these iterations for a high heterogeneity setting $(\alpha, \beta)=(5,5)$.

Our results highlight the trade-off between $T$ and $K$ (which relates to
hardware and communication constraints in real deployments) to achieve some given performance. Indeed, if $K$ is too large, $T$ has to be chosen very low to ensure the desired privacy, leading to poor accuracy. For instance, with $K=40$, $T$ cannot exceed $90$, and the resulting accuracy thus barely reaches $22\%$, even with low private noise. On the other hand, if we set $K$ too low, \texttt{DP-SCAFFOLD} does not converge despite a high value of $T$, since it does not take advantage of the local updates. Moreover, we can observe another dimension of the trade-off involving $\sigma_g$. It seems that better performance can be achieved by setting $\sigma_g$ relatively low, although it implies to choose a smaller $T$. This trade-off is evidenced by the fact that the accuracy achieved in the first two rows ($\sigma_g=10$ and $\sigma_g=20$) is quite similar, showing that $\sigma_g$ and $T$ compensate each other.

\textbf{Other results.} Appendix~\ref{appendix_experiment} shows results with other metrics and heterogeneity levels, higher privacy regimes, and presents additional experiments on the effect of sampling parameters $l$ and $s$ (and the trade-off with $T$) on privacy and convergence.
\section{CONCLUSION}\label{section_perspectives}

Our paper introduced a novel FL algorithm, \texttt{DP-SCAFFOLD}, to tackle data heterogeneity under DP constraints, and showed that it improves over the baseline \texttt{DP-FedAvg} from both the theoretical and empirical point of view. In particular, our theoretical analysis highlights an interesting trade-off between the parameters of the problem, involving a term of heterogeneity in \texttt{DP-FedAvg} which does not appear in the rate of \texttt{DP-SCAFFOLD}. As future work, we aim at providing additional experiments with deep learning models and various sizes of local datasets across users, for more realistic use-cases. Besides, our paper opens other perspectives. \texttt{DP-SCAFFOLD} may be improved by incorporating other ML techniques such as momentum. On the experimental side, a larger number of samples and a more precise tuning of the trade-off between $T$, $K$ and subsampling parameters may dramatically improve the utility for real-world cases under a given privacy budget. From a theoretical perspective, investigating an adaptation of our approach to a personalized FL setting \citep{fallah2020personalized,sattler2020clustered,Marfoq2021a}, where formal privacy guarantees have seldom been  studied \citep[at the exception of][]{personalized_ml,dp_personalized_fl}, is a direction of interest.

\begin{table*}[h!]
\vspace{-1.5em}
\caption{Test Accuracy ($\%$) For \texttt{DP-SCAFFOLD} on Synthetic Data, With $\epsilon=3$, $l=0.05$, $s=0.2$, $(\alpha, \beta)=(5,5)$.}
\label{table_experiments_varying_k_sigma_gaussian_5_5}
\begin{center}
\vspace{-1em}
\resizebox{\linewidth}{!}{
\begin{tabular}{cclclclclcl}
\toprule
\textbf{$\sigma_g$} &\multicolumn{2}{c}{$K=1$} &\multicolumn{2}{c}{$K=5$}&\multicolumn{2}{c}{$K=10$} & \multicolumn{2}{c}{$K=20$} & \multicolumn{2}{c}{$K=40$} \\
\cmidrule(l){1-1}\cmidrule(l){2-3}\cmidrule(l){4-5}\cmidrule(l){6-7}\cmidrule(l){8-9}\cmidrule(l){10-11}
$10$ &  $27.41_{\pm 0.71}$ & $T=542$
     &$\mathbf{45.53}_{\pm 0.99}$ & $T=488$
     & $43.52_{\pm 1.52}$ & $T=428$
     & $42.51_{\pm 0.80}$ & $T=324$
     & $21.80_{\pm 3.28}$ & $T=72$\\
$20$ &  $27.34_{\pm 1.31}$ & $T=545$
     & $\mathbf{44.39}_{\pm 0.46}$ & $T=502$
     & $43.47_{\pm 1.74}$ & $T=451$
     & $42.33_{\pm 0.77}$ & $T=352$
     & $20.14_{\pm 2.67}$ & $T=83$\\
$40$ & $21.05_{\pm 2.27}$ & $T=546$
     & $34.50_{\pm 0.65}$ & $T=505$
     & $\mathbf{36.85}_{\pm 0.85}$ & $T=457$
     & $33.24_{\pm 0.41}$ & $T=360$
     & $14.85_{\pm 0.95}$ & $T=86$\\
$80$ & $17.61_{\pm 2.62}$ & $T=546$
     & $24.41_{\pm 0.81}$ & $T=506$
     & $\mathbf{27.33}_{\pm 0.37}$ & $T=458$
     & $19.42_{\pm 0.51}$ & $T=362$
     & $14.08_{\pm 0.14}$ & $T=87$\\
$160$ & $13.97_{\pm 1.70}$ & $T=546$
     & $15.99_{\pm 0.30}$ & $T=506$
     & $\mathbf{19.27}_{\pm 1.65}$ & $T=458$
     & $14.86_{\pm 0.75}$ & $T=362$
     & $14.17_{\pm 0.06}$ & $T=87$\\
\bottomrule
\end{tabular}
}
\end{center}
\end{table*}
 
\begin{figure*}[h!]
    \centering
    \includegraphics[width=\linewidth]{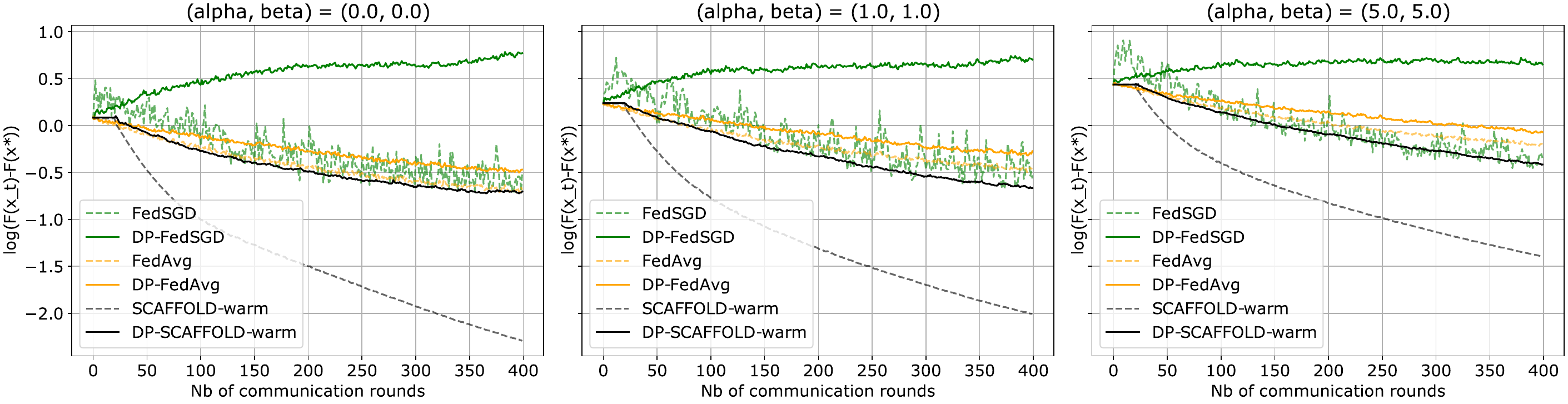} \\
\vspace{0.6em}
    \includegraphics[width=\linewidth]{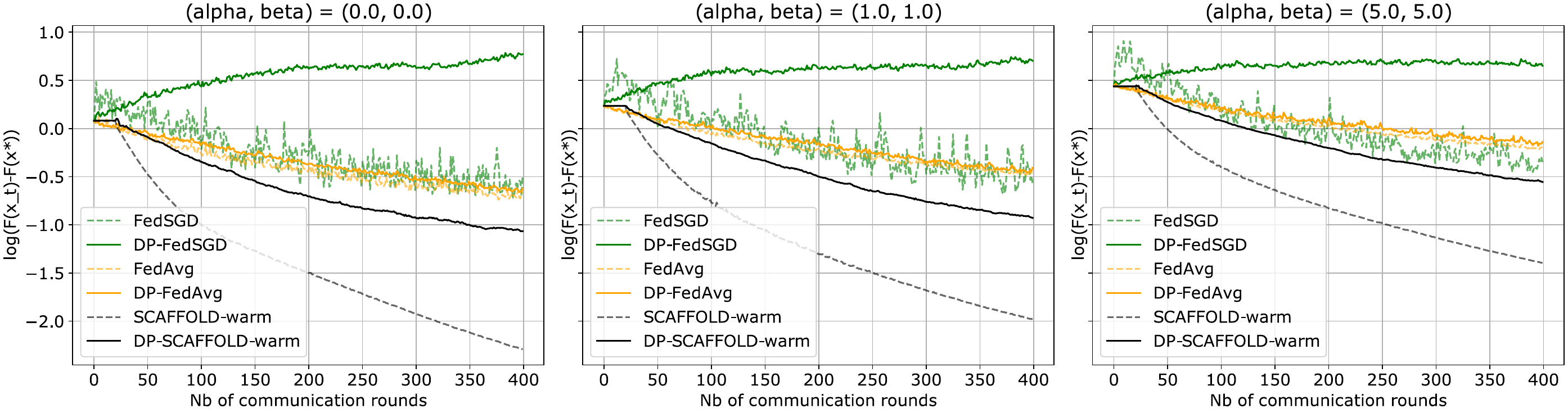}
\vspace{-2em}
    \caption{Train Loss On Synthetic Data ($\epsilon=13$). First Row: $K=50$; Second Row: $K=100$.}
    \label{logistic_heterogene_K50_K100}
\end{figure*}

\begin{figure*}[h!]
    \centering
    \includegraphics[width=\linewidth]{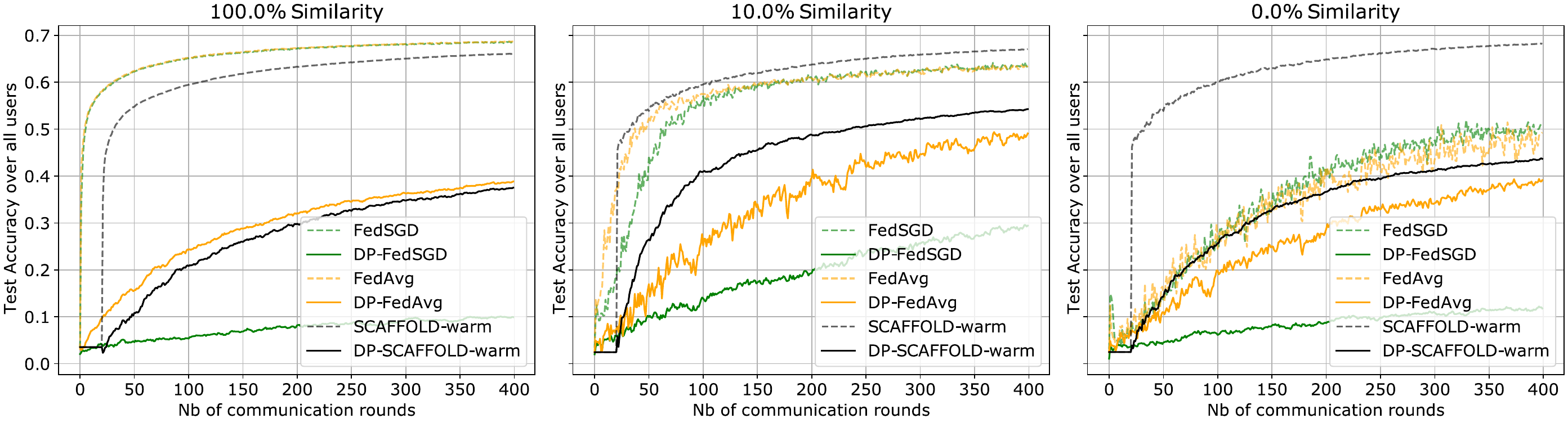}\\
\vspace{0.6em}
    \includegraphics[width=\linewidth, height=4.5cm]{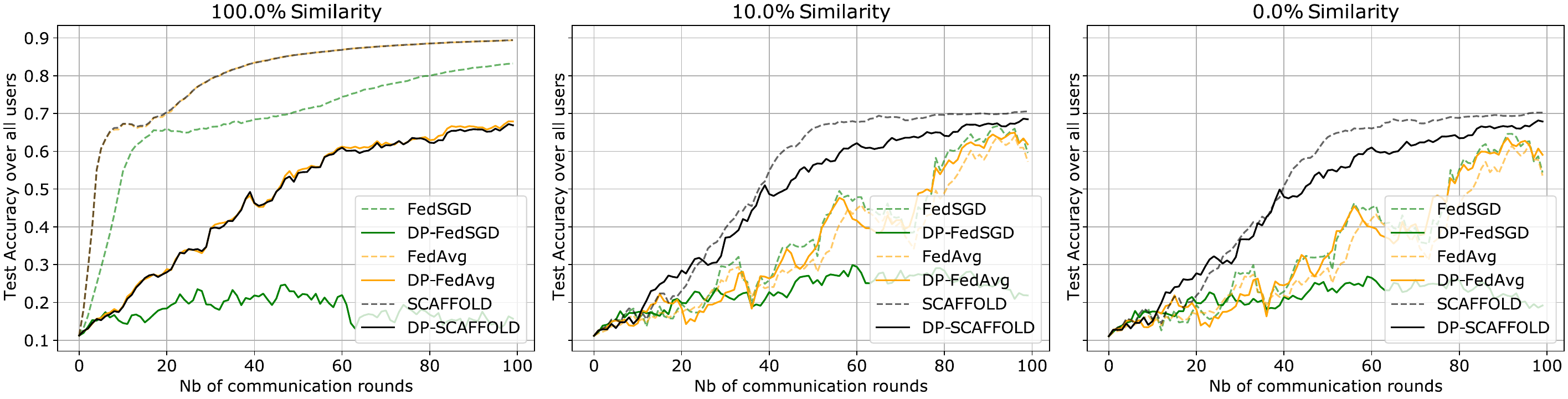}
\vspace{-2em}
    \caption{Test Accuracy With $K=50$. First Row: FEMNIST (LogReg), $\epsilon=11.4$; Second Row: MNIST (DNN), $\epsilon=7.2$.}
\vspace{-3em}
    \label{femnist_and_mnist_heterogene_K50}
\end{figure*}

\section*{Acknowledgments}
We thank Baptiste Goujaud and Constantin Philippenko for interesting discussions.  We thank anonymous reviewers for their constructive feedback.
The work of A. Dieuleveut is partially supported by ANR-19-CHIA-0002-01 /chaire SCAI, and Hi! Paris. 
The work of A. Bellet is supported by grants ANR-16-CE23-0016 (Project PAMELA) and ANR-20-CE23-0015 (Project PRIDE).

\bibliography{bibfile}

\begin{thebibliography}{47}
\providecommand{\natexlab}[1]{#1}
\providecommand{\url}[1]{\texttt{#1}}
\expandafter\ifx\csname urlstyle\endcsname\relax
  \providecommand{\doi}[1]{doi: #1}\else
  \providecommand{\doi}{doi: \begingroup \urlstyle{rm}\Url}\fi

\bibitem[Abadi et~al.(2016)Abadi, Chu, Goodfellow, McMahan, Mironov, Talwar,
  and Zhang]{dp_deep_l}
Martin Abadi, Andy Chu, Ian Goodfellow, H~Brendan McMahan, Ilya Mironov, Kunal
  Talwar, and Li~Zhang.
\newblock Deep learning with differential privacy.
\newblock In \emph{Proceedings of the 2016 ACM SIGSAC conference on computer
  and communications security}, pages 308--318, 2016.

\bibitem[Acar et~al.(2021)Acar, Zhao, Matas, Mattina, Whatmough, and
  Saligrama]{acar2020federated}
Durmus Alp~Emre Acar, Yue Zhao, Ramon Matas, Matthew Mattina, Paul Whatmough,
  and Venkatesh Saligrama.
\newblock Federated learning based on dynamic regularization.
\newblock In \emph{International Conference on Learning Representations}, 2021.

\bibitem[Andrew et~al.(2021)Andrew, Thakkar, McMahan, and
  Ramaswamy]{dp_adaptive_clipping}
Galen Andrew, Om~Thakkar, H.~Brendan McMahan, and Swaroop Ramaswamy.
\newblock Differentially private learning with adaptive clipping.
\newblock \emph{Advances in Neural Information Processing Systems}, 34, 2021.

\bibitem[Balle et~al.(2019)Balle, Bell, Gasc{\'o}n, and
  Nissim]{dp_privacy_blanket_shuffle_model}
Borja Balle, James Bell, Adri{\`a} Gasc{\'o}n, and Kobbi Nissim.
\newblock The privacy blanket of the shuffle model.
\newblock In \emph{Annual International Cryptology Conference}, pages 638--667.
  Springer, 2019.

\bibitem[Bellet et~al.(2018)Bellet, Guerraoui, Taziki, and
  Tommasi]{personalized_ml}
Aur{\'e}lien Bellet, Rachid Guerraoui, Mahsa Taziki, and Marc Tommasi.
\newblock Personalized and private peer-to-peer machine learning.
\newblock In \emph{International Conference on Artificial Intelligence and
  Statistics}, pages 473--481. PMLR, 2018.

\bibitem[Cheu et~al.(2019)Cheu, Smith, Ullman, Zeber, and
  Zhilyaev]{dp_via_shuffling}
Albert Cheu, Adam Smith, Jonathan Ullman, David Zeber, and Maxim Zhilyaev.
\newblock Distributed differential privacy via shuffling.
\newblock In \emph{Annual International Conference on the Theory and
  Applications of Cryptographic Techniques}, pages 375--403. Springer, 2019.

\bibitem[Cohen et~al.(2017)Cohen, Afshar, Tapson, and Van~Schaik]{emnist}
Gregory Cohen, Saeed Afshar, Jonathan Tapson, and Andre Van~Schaik.
\newblock Emnist: Extending mnist to handwritten letters.
\newblock In \emph{2017 international joint conference on neural networks
  (IJCNN)}, pages 2921--2926. IEEE, 2017.

\bibitem[Duchi et~al.(2011)Duchi, Hazan, and Singer]{duchi2011adaptive}
John~C. Duchi, Elad Hazan, and Yoram Singer.
\newblock Adaptive subgradient methods for online learning and stochastic
  optimization.
\newblock \emph{Journal of machine learning research}, 12\penalty0 (7), 2011.

\bibitem[Duchi et~al.(2013)Duchi, Jordan, and Wainwright]{ldp_baseline}
John~C. Duchi, Michael~I. Jordan, and Martin~J. Wainwright.
\newblock Local privacy and statistical minimax rates.
\newblock In \emph{2013 IEEE 54th Annual Symposium on Foundations of Computer
  Science}, pages 429--438. IEEE, 2013.

\bibitem[Duchi et~al.(2018)Duchi, Jordan, and
  Wainwright]{ldp_mechanism_tradeoff}
John~C. Duchi, Michael~I. Jordan, and Martin~J. Wainwright.
\newblock Minimax optimal procedures for locally private estimation.
\newblock \emph{Journal of the American Statistical Association}, 113\penalty0
  (521):\penalty0 182--201, 2018.

\bibitem[Dwork and Roth(2014)]{dp_foundations}
Cynthia Dwork and Aaron Roth.
\newblock The algorithmic foundations of differential privacy.
\newblock \emph{Foundations and Trends® in Theoretical Computer Science},
  9\penalty0 (3-4):\penalty0 211--407, 2014.

\bibitem[Dwork et~al.(2010)Dwork, Rothblum, and Vadhan]{boosting_dp}
Cynthia Dwork, Guy~N. Rothblum, and Salil Vadhan.
\newblock Boosting and differential privacy.
\newblock In \emph{2010 IEEE 51st Annual Symposium on Foundations of Computer
  Science}, pages 51--60. IEEE, 2010.

\bibitem[Erlingsson et~al.(2019)Erlingsson, Feldman, Mironov, Raghunathan,
  Talwar, and Thakurta]{dp_amplification_by_shuffling}
{\'U}lfar Erlingsson, Vitaly Feldman, Ilya Mironov, Ananth Raghunathan, Kunal
  Talwar, and Abhradeep Thakurta.
\newblock Amplification by shuffling: From local to central differential
  privacy via anonymity.
\newblock In \emph{Proceedings of the Thirtieth Annual ACM-SIAM Symposium on
  Discrete Algorithms}, pages 2468--2479. SIAM, 2019.

\bibitem[Fallah et~al.(2020)Fallah, Mokhtari, and
  Ozdaglar]{fallah2020personalized}
Alireza Fallah, Aryan Mokhtari, and Asuman Ozdaglar.
\newblock Personalized federated learning with theoretical guarantees: A
  model-agnostic meta-learning approach.
\newblock \emph{Advances in Neural Information Processing Systems},
  33:\penalty0 3557--3568, 2020.

\bibitem[Fredrikson et~al.(2015)Fredrikson, Jha, and
  Ristenpart]{fredrikson2015model}
Matt Fredrikson, Somesh Jha, and Thomas Ristenpart.
\newblock Model inversion attacks that exploit confidence information and basic
  countermeasures.
\newblock In \emph{Proceedings of the 22nd ACM SIGSAC conference on computer
  and communications security}, pages 1322--1333, 2015.

\bibitem[Geiping et~al.(2020)Geiping, Bauermeister, Dr{\"o}ge, and
  Moeller]{inverting}
Jonas Geiping, Hartmut Bauermeister, Hannah Dr{\"o}ge, and Michael Moeller.
\newblock Inverting gradients-how easy is it to break privacy in federated
  learning?
\newblock \emph{Advances in Neural Information Processing Systems},
  33:\penalty0 16937--16947, 2020.

\bibitem[Geyer et~al.(2017)Geyer, Klein, and Nabi]{dp_fed_sgd_user_level}
Robin~C. Geyer, Tassilo Klein, and Moin Nabi.
\newblock Differentially private federated learning: A client level
  perspective.
\newblock \emph{arXiv preprint arXiv:1712.07557}, 2017.

\bibitem[Ghazi et~al.(2019)Ghazi, Pagh, and Velingker]{dp_shuffled_model}
Badih Ghazi, Rasmus Pagh, and Ameya Velingker.
\newblock Scalable and differentially private distributed aggregation in the
  shuffled model.
\newblock \emph{arXiv preprint arXiv:1906.08320}, 2019.

\bibitem[Girgis et~al.(2021{\natexlab{a}})Girgis, Data, and
  Diggavi]{girgis2021renyi}
Antonious~M. Girgis, Deepesh Data, and Suhas Diggavi.
\newblock Renyi differential privacy of the subsampled shuffle model in
  distributed learning.
\newblock \emph{Advances in Neural Information Processing Systems}, 34,
  2021{\natexlab{a}}.

\bibitem[Girgis et~al.(2021{\natexlab{b}})Girgis, Data, Diggavi, Kairouz, and
  Suresh]{fl_shuffle}
Antonious~M. Girgis, Deepesh Data, Suhas Diggavi, Peter Kairouz, and
  Ananda~Theertha Suresh.
\newblock Shuffled model of federated learning: Privacy, accuracy and
  communication trade-offs.
\newblock \emph{IEEE Journal on Selected Areas in Information Theory},
  2\penalty0 (1):\penalty0 464--478, 2021{\natexlab{b}}.

\bibitem[Gower et~al.(2019)Gower, Loizou, Qian, Sailanbayev, Shulgin, and
  Richt{\'a}rik]{gower_sgd_2019}
Robert~Mansel Gower, Nicolas Loizou, Xun Qian, Alibek Sailanbayev, Egor
  Shulgin, and Peter Richt{\'a}rik.
\newblock Sgd: General analysis and improved rates.
\newblock In \emph{International Conference on Machine Learning}, pages
  5200--5209. PMLR, 2019.

\bibitem[Hsu et~al.(2019)Hsu, Qi, and Brown]{measuring_heterogeneity}
Tzu-Ming~Harry Hsu, Hang Qi, and Matthew Brown.
\newblock Measuring the effects of non-identical data distribution for
  federated visual classification.
\newblock \emph{arXiv preprint arXiv:1909.06335}, 2019.

\bibitem[Hu et~al.(2020)Hu, Guo, Li, Pei, and Gong]{dp_personalized_fl}
Rui Hu, Yuanxiong Guo, Hongning Li, Qingqi Pei, and Yanmin Gong.
\newblock Personalized federated learning with differential privacy.
\newblock \emph{IEEE Internet of Things Journal}, 7\penalty0 (10):\penalty0
  9530--9539, 2020.

\bibitem[Kairouz et~al.(2015)Kairouz, Oh, and Viswanath]{composition_dp}
Peter Kairouz, Sewoong Oh, and Pramod Viswanath.
\newblock The composition theorem for differential privacy.
\newblock In \emph{International conference on machine learning}, pages
  1376--1385. PMLR, 2015.

\bibitem[Kairouz et~al.(2021)Kairouz, McMahan, Avent, Bellet, Bennis, Bhagoji,
  Bonawitz, Charles, Cormode, Cummings, D'Oliveira, Eichner, Rouayheb, Evans,
  Gardner, Garrett, Gascón, Ghazi, Gibbons, Gruteser, Harchaoui, He, He, Huo,
  Hutchinson, Hsu, Jaggi, Javidi, Joshi, Khodak, Konečný, Korolova,
  Koushanfar, Koyejo, Lepoint, Liu, Mittal, Mohri, Nock, Özgür, Pagh,
  Raykova, Qi, Ramage, Raskar, Song, Song, Stich, Sun, Suresh, Tramèr,
  Vepakomma, Wang, Xiong, Xu, Yang, Yu, Yu, and Zhao]{fl_problems}
Peter Kairouz, H.~Brendan McMahan, Brendan Avent, Aurélien Bellet, Mehdi
  Bennis, Arjun~Nitin Bhagoji, Kallista Bonawitz, Zachary Charles, Graham
  Cormode, Rachel Cummings, Rafael G.~L. D'Oliveira, Hubert Eichner, Salim~El
  Rouayheb, David Evans, Josh Gardner, Zachary Garrett, Adrià Gascón, Badih
  Ghazi, Phillip~B. Gibbons, Marco Gruteser, Zaid Harchaoui, Chaoyang He, Lie
  He, Zhouyuan Huo, Ben Hutchinson, Justin Hsu, Martin Jaggi, Tara Javidi,
  Gauri Joshi, Mikhail Khodak, Jakub Konečný, Aleksandra Korolova, Farinaz
  Koushanfar, Sanmi Koyejo, Tancrède Lepoint, Yang Liu, Prateek Mittal,
  Mehryar Mohri, Richard Nock, Ayfer Özgür, Rasmus Pagh, Mariana Raykova,
  Hang Qi, Daniel Ramage, Ramesh Raskar, Dawn Song, Weikang Song, Sebastian~U.
  Stich, Ziteng Sun, Ananda~Theertha Suresh, Florian Tramèr, Praneeth
  Vepakomma, Jianyu Wang, Li~Xiong, Zheng Xu, Qiang Yang, Felix~X. Yu, Han Yu,
  and Sen Zhao.
\newblock Advances and open problems in federated learning.
\newblock \emph{Foundations and Trends{\textregistered} in Machine Learning},
  14\penalty0 (1--2):\penalty0 1--210, 2021.

\bibitem[Karimireddy et~al.(2020{\natexlab{a}})Karimireddy, Jaggi, Kale, Mohri,
  Reddi, Stich, and Suresh]{karimireddy2020mime}
Sai~Praneeth Karimireddy, Martin Jaggi, Satyen Kale, Mehryar Mohri, Sashank~J.
  Reddi, Sebastian~U. Stich, and Ananda~Theertha Suresh.
\newblock Mime: Mimicking centralized stochastic algorithms in federated
  learning.
\newblock \emph{arXiv preprint arXiv:2008.03606}, 2020{\natexlab{a}}.

\bibitem[Karimireddy et~al.(2020{\natexlab{b}})Karimireddy, Kale, Mohri, Reddi,
  Stich, and Suresh]{fl_scaffold}
Sai~Praneeth Karimireddy, Satyen Kale, Mehryar Mohri, Sashank Reddi, Sebastian
  Stich, and Ananda~Theertha Suresh.
\newblock Scaffold: Stochastic controlled averaging for federated learning.
\newblock In \emph{International Conference on Machine Learning}, pages
  5132--5143. PMLR, 2020{\natexlab{b}}.

\bibitem[Kasiviswanathan et~al.(2011)Kasiviswanathan, Lee, Nissim,
  Raskhodnikova, and Smith]{dp_subsampling}
Shiva~Prasad Kasiviswanathan, Homin~K Lee, Kobbi Nissim, Sofya Raskhodnikova,
  and Adam Smith.
\newblock What can we learn privately?
\newblock \emph{SIAM Journal on Computing}, 40\penalty0 (3):\penalty0 793--826,
  2011.

\bibitem[Khaled et~al.(2020)Khaled, Mishchenko, and
  Richt{\'a}rik]{fl_theory_update}
Ahmed Khaled, Konstantin Mishchenko, and Peter Richt{\'a}rik.
\newblock Tighter theory for local sgd on identical and heterogeneous data.
\newblock In \emph{International Conference on Artificial Intelligence and
  Statistics}, pages 4519--4529. PMLR, 2020.

\bibitem[Kingma and Ba(2014)]{kingma2014adam}
Diederik~P. Kingma and Jimmy Ba.
\newblock Adam: A method for stochastic optimization.
\newblock \emph{arXiv preprint arXiv:1412.6980}, 2014.

\bibitem[LeCun et~al.(1998)LeCun, Bottou, Bengio, and
  Haffner]{lecun1998gradient}
Yann LeCun, L{\'e}on Bottou, Yoshua Bengio, and Patrick Haffner.
\newblock Gradient-based learning applied to document recognition.
\newblock \emph{Proceedings of the IEEE}, 86\penalty0 (11):\penalty0
  2278--2324, 1998.

\bibitem[Li et~al.(2020{\natexlab{a}})Li, Sahu, Zaheer, Sanjabi, Talwalkar, and
  Smith]{fl_fedprox}
Tian Li, Anit~Kumar Sahu, Manzil Zaheer, Maziar Sanjabi, Ameet Talwalkar, and
  Virginia Smith.
\newblock Federated optimization in heterogeneous networks.
\newblock \emph{Proceedings of Machine Learning and Systems}, 2:\penalty0
  429--450, 2020{\natexlab{a}}.

\bibitem[Li et~al.(2020{\natexlab{b}})Li, Chang, and
  Chi]{dp_secure_federated_averaging}
Yiwei Li, Tsung-Hui Chang, and Chong-Yung Chi.
\newblock Secure federated averaging algorithm with differential privacy.
\newblock In \emph{2020 IEEE 30th International Workshop on Machine Learning
  for Signal Processing (MLSP)}, pages 1--6. IEEE, 2020{\natexlab{b}}.

\bibitem[Marfoq et~al.(2021)Marfoq, Neglia, Bellet, Kameni, and
  Vidal]{Marfoq2021a}
Othmane Marfoq, Giovanni Neglia, Aur{\'e}lien Bellet, Laetitia Kameni, and
  Richard Vidal.
\newblock Federated multi-task learning under a mixture of distributions.
\newblock \emph{Advances in Neural Information Processing Systems}, 34, 2021.

\bibitem[McMahan et~al.(2017)McMahan, Moore, Ramage, Hampson, and
  y~Arcas]{communication_efficient_fl}
H.~Brendan McMahan, Eider Moore, Daniel Ramage, Seth Hampson, and Blaise~Aguera
  y~Arcas.
\newblock Communication-efficient learning of deep networks from decentralized
  data.
\newblock In \emph{Artificial intelligence and statistics}, pages 1273--1282.
  PMLR, 2017.

\bibitem[McMahan et~al.(2018)McMahan, Ramage, Talwar, and
  Zhang]{dp_fedavg_user_level}
H.~Brendan McMahan, Daniel Ramage, Kunal Talwar, and Li~Zhang.
\newblock Learning differentially private recurrent language models.
\newblock In \emph{International Conference on Learning Representations}, 2018.

\bibitem[Mironov(2017)]{renyi_dp}
Ilya Mironov.
\newblock R{\'e}nyi differential privacy.
\newblock In \emph{2017 IEEE 30th computer security foundations symposium
  (CSF)}, pages 263--275. IEEE, 2017.

\bibitem[Nesterov et~al.(2004)]{nesterov2018lectures}
Yurii Nesterov et~al.
\newblock \emph{Lectures on convex optimization}, volume 137.
\newblock Springer, 2004.

\bibitem[Reddi et~al.(2020)Reddi, Charles, Zaheer, Garrett, Rush,
  Kone{\v{c}}n{\`y}, Kumar, and McMahan]{reddi2020adaptive}
Sashank~J. Reddi, Zachary Charles, Manzil Zaheer, Zachary Garrett, Keith Rush,
  Jakub Kone{\v{c}}n{\`y}, Sanjiv Kumar, and H.~Brendan McMahan.
\newblock Adaptive federated optimization.
\newblock In \emph{International Conference on Learning Representations}, 2020.

\bibitem[Sattler et~al.(2020)Sattler, M{\"u}ller, and
  Samek]{sattler2020clustered}
Felix Sattler, Klaus-Robert M{\"u}ller, and Wojciech Samek.
\newblock Clustered federated learning: Model-agnostic distributed multitask
  optimization under privacy constraints.
\newblock \emph{IEEE Transactions on Neural Networks and Learning Systems},
  2020.

\bibitem[Shokri et~al.(2017)Shokri, Stronati, Song, and
  Shmatikov]{shokri2017membership}
Reza Shokri, Marco Stronati, Congzheng Song, and Vitaly Shmatikov.
\newblock Membership inference attacks against machine learning models.
\newblock In \emph{2017 IEEE Symposium on Security and Privacy (SP)}, pages
  3--18. IEEE, 2017.

\bibitem[Triastcyn and Faltings(2019)]{triastcyn2019federated}
Aleksei Triastcyn and Boi Faltings.
\newblock Federated learning with bayesian differential privacy.
\newblock In \emph{2019 IEEE International Conference on Big Data (Big Data)},
  pages 2587--2596. IEEE, 2019.

\bibitem[Van~Erven and Harremos(2014)]{property_rdp}
Tim Van~Erven and Peter Harremos.
\newblock R{\'e}nyi divergence and kullback-leibler divergence.
\newblock \emph{IEEE Transactions on Information Theory}, 60\penalty0
  (7):\penalty0 3797--3820, 2014.

\bibitem[Wang et~al.(2021)Wang, Charles, Xu, Joshi, McMahan, Al-Shedivat,
  Andrew, Avestimehr, Daly, Data, et~al.]{wang2021field}
Jianyu Wang, Zachary Charles, Zheng Xu, Gauri Joshi, H.~Brendan McMahan, Maruan
  Al-Shedivat, Galen Andrew, Salman Avestimehr, Katharine Daly, Deepesh Data,
  et~al.
\newblock A field guide to federated optimization.
\newblock \emph{arXiv preprint arXiv:2107.06917}, 2021.

\bibitem[Wang et~al.(2020)Wang, Balle, and
  Kasiviswanathan]{renyi_dp_sampled_gm_explicit}
Yu-Xiang Wang, Borja Balle, and Shiva Kasiviswanathan.
\newblock Subsampled {Rényi} {Differential} {Privacy} and {Analytical}
  {Moments} {Accountant}.
\newblock \emph{Journal of Privacy and Confidentiality}, 10\penalty0 (2), 2020.

\bibitem[Wei et~al.(2020)Wei, Li, Ding, Ma, Yang, Farokhi, Jin, Quek, and
  Poor]{dp_federated_averaging_nbafl}
Kang Wei, Jun Li, Ming Ding, Chuan Ma, Howard~H. Yang, Farhad Farokhi, Shi Jin,
  Tony Q.~S. Quek, and H.~Vincent Poor.
\newblock Federated learning with differential privacy: Algorithms and
  performance analysis.
\newblock \emph{IEEE Transactions on Information Forensics and Security},
  15:\penalty0 3454--3469, 2020.

\bibitem[Zhao et~al.(2021)Zhao, Zhao, Yang, Wang, Wang, Lyu, Niyato, and
  Lam]{ldp_mechanisms}
Yang Zhao, Jun Zhao, Mengmeng Yang, Teng Wang, Ning Wang, Lingjuan Lyu, Dusit
  Niyato, and Kwok-Yan Lam.
\newblock Local {Differential} {Privacy} based {Federated} {Learning} for
  {Internet} of {Things}.
\newblock \emph{IEEE Internet of Things Journal}, 8\penalty0 (11):\penalty0
  8836--8853, 2021.

\end{thebibliography}

\clearpage
\appendix

\onecolumn
\makesupplementtitle

\section*{ORGANIZATION OF THE APPENDIX}

This appendix is organized as follows. Appendix~\ref{app:complementary} summarizes the main notations and provides the detailed \dpfedavg algorithm for completeness.
Appendix~\ref{appendix:privacy} provides details on our privacy analysis. Appendix~\ref{app:othertheoreticalresults} gives the full proofs of our utility results for the convex, strongly convex and non-convex cases. Finally, Appendix~\ref{app:exp_all} provides more details on the experiments of Section~\ref{sec:experiments}, as well as additional results.

\section{ADDITIONAL INFORMATION}
\label{app:complementary}

\subsection{Table of Notations}

Table~\ref{table_notations} summarizes the main notations used throughout the paper.

\begin{table}[h]
\caption{Summary of the main notations.} \label{table_notations}
\begin{center}
\begin{tabular}{ll}
\textbf{Symbol}  &\textbf{Description} \\
\hline \\
    $[n]$ & set $\{1,2,...,n\}$ for any  $n \in \mathbb{N}$\\
    $M$, $i \in [M]$ & number and index of users\\
    $T$, $t \in [T]$ & number and index of communication rounds\\
    $K$, $k \in [K]$ & number and index of local updates (for each user)\\
    $D_i$ & local dataset held by the $i$-th user, composed of points $d_1^i, \hdots, d^i_\r$\\
    $\r$ & size of any local dataset $D_i$\\
    $D$ & joint dataset ($\bigsqcup_{i=1}^{M} D_i$)\\
    $f_i(x,d)$ & loss of the $i$-th user for model $x$ on data record $d$ \\
    $F_i$ & local empirical risk function of the $i$-th user $(\frac{1}{\r}\sum_{j=1}^\r f_i(\cdot, d_j^i))$\\
    $F$ & global objective function $(\frac{1}{M}\sum_{i=1}^M F_i)$\\
    $x^t\in \mathbb{R}^d$ & server model after round $t$\\
    $y_i^k\in \mathbb{R}^d$ & model of $i$-th user after local update $k$ \\
    $c^t \in \mathbb{R}^d$ & server control variate after round $t$ \\
    $c_i^t\in \mathbb{R}^d$ & control variate of the $i$-th user after round $t$ \\
    $l\in (0,1)$ &  user sampling ratio\\ 
    $s\in (0,1)$ & data sampling ratio\\ 
    $\epsilon,\delta$ & differential privacy parameters \\
    $\sigma_g$ & standard deviation of Gaussian noise added for privacy\\
    $\mathcal{C}$ & gradient clipping threshold\\
    $\nu$ & Lipschitz-smoothness constant\\
    $\mu$ & strong convexity parameter \\
    $\sigmaoracle^2$ & variance of stochastic gradients\\
    \hline
\end{tabular}
\end{center}
\end{table}

\subsection{\dpfedavg Algorithm}
The code of \dpfedavg is given in Algorithm~\ref{algo_dp_fedavg}.

 \begin{algorithm}
\DontPrintSemicolon
\caption{DP-FedAvg$(T,K,l,s,\sigma_g,\mathcal{C})$ \label{algo_dp_fedavg}}
  \KwInputServer{initial $x^0$}
  \KwOutput{$x^T$}
  \For{$t=1,...,T$}
    {
        \HiLiRed\textbf{User subsampling by the server:} $C^t \subset [M]$ of size $\lfloor lM\rfloor$\\
        \textbf{Server communicates} $x^{t-1}$ to users $i \in C^t$\\
        \For{user $i \in C^t$}
        {
            Initialize model: $y_i^0 \leftarrow x^t$\\
            \For{$k=1,...,K$}
            {   
                \HiLiRed\textbf{Data subsampling by user: }$S_i^k \subset D_i$ of size $\lfloor sR\rfloor$\\
                \For{sample $j \in S_i^k$}
                {
                    \textbf{Compute gradient:} $g_{ij} \leftarrow \nabla f_i(y_i^{k-1},d^i_{j}) $\\
                    \HiLiYellow \textbf{Clip gradient:} $\tilde{g}_{ij} \leftarrow g_{ij} / \max \big(1, ||g_{ij}||_2/\mathcal{C}\big)$
                    }
                 \HiLiYellow\textbf{Add DP noise to local gradients:} $\tilde{H}_i^k \leftarrow \frac{1}{s \r}\sum_{j \in S_i^k} \tilde{g}_{ij} + \frac{2\mathcal{C}}{s\r}\mathcal{N} (0,\sigma^2_g)$\\
                $y_i^{k} \leftarrow y_i^{k-1} - \eta_l \tilde{H}_i^k$}
            $\Delta y_i^t \leftarrow y_i^K - x^{t-1}$\\
            \textbf{User $i$ communicates to server:} $\Delta y_i^t$
            }
        \textbf{Server aggregates}: $\Delta x^t \leftarrow \frac{1}{lM}\sum_{i \in C^t}\Delta y_i^t$\ \\
        $x^{t} \leftarrow x^{t-1} + \eta_g\Delta x^t$
    }
 \end{algorithm}

\section{DETAILS ON PRIVACY ANALYSIS}
\label{appendix:privacy}

In this section, we provide the proof of our privacy results. We start by recalling standard differential privacy results
on composition and amplification by subsampling
in Section~\ref{formal_dp}.
Section~\ref{rdp_theory} reviews recent results in Rényi Differential Privacy (RDP) which allow to obtain tighter privacy bounds.
We then formally state and prove Claim~\ref{lemma_privacy} in Section~\ref{proof_claim}. Finally, we provide the proof of our main result (Theorem~\ref{theorem_privacy}) in Section~\ref{proof_theorem_privacy}.

\subsection{Reminders on Differential Privacy}
\label{formal_dp}

In the following, we denote by 
$D \in \mathcal{X}^n$ to a dataset of size $n$. Two datasets $D,D' \in \mathcal{X}^n$ are said to be neighboring (denoted by $||D-D'||\leq 1$) if they differ by at most one element.

\paragraph{Composition.}


Let $M_1(\cdot;A_1), ..., M_T(\cdot;A_T)$ be a sequence of $T$ \textit{adaptive} DP mechanisms where $A_t$ stands for the auxiliary input to the $t$-th mechanism, which may depend on the outputs of previous mechanisms $(M_{t'})_{t'<t}$. The ability to choose the sequences of mechanisms adaptively is crucial for the design of iterative machine learning algorithms.
DP allows to keep track of the privacy guarantees when such a sequence of private mechanisms is run on the same dataset $D$. 
Simple composition \citep[Theorem III.1.]{boosting_dp} states that the privacy parameters grow linearly with $T$. \cite{boosting_dp} provide a \textit{strong composition} result where the $\epsilon$ parameter grows sublinearly with $T$. This result is recalled in
Lemma~\ref{lemma_strong_composition}. 

\begin{lemma}[Strong adaptive composition,
\citealp{boosting_dp}]
\label{lemma_strong_composition}
Let $M_1,...,M_T$ be $T$ adaptive $(\epsilon,\delta)$-DP mechanisms. Then, for any $\delta' >0$, the mechanism $M=(M_1,...,M_T)$ is $(\overline{\epsilon},\overline{\delta})$-DP where:
\begin{align*}
     \overline{\epsilon}=\epsilon \sqrt{2T\log(1/\delta')} + T\epsilon (e^{\epsilon} -1) \text{ and }\overline{\delta}=T\delta+ \delta'.
\end{align*}
\end{lemma}

\textbf{Remark.} When stating theoretical results, $\overline{\epsilon}$ is typically approximated by $\mathcal{O}(\epsilon\sqrt{T\log(1/\delta')})$ when $\epsilon << 1$.

\paragraph{Privacy amplification by subsampling.} 

A key result in DP is that applying a private algorithm on a random subsample of the dataset amplifies privacy guarantees \citep{dp_subsampling}. 
In this work, we are interested in subsampling without replacement.

\begin{definition}[Subsampling without replacement]
\label{def_subsampling}
The subsampling procedure $\text{Samp}_{n,m} : \mathcal{X}^{n} \rightarrow \mathcal{X}^{m}$ (where $m \in \mathbb{N}$, with $m \leq n$) takes $D$ as input and chooses uniformly among its elements a subset $\underline{D}$ of $m$ elements. We may also denote $\text{Samp}_{n,m}$ as $\text{Samp}_{q}$ where $q=m/n$ in the rest of the paper.
\end{definition}

Lemma \ref{lemma_subsampling} quantifies the associated privacy amplification effect.

\begin{lemma}[Amplification by subsampling, \citealp{dp_subsampling}]
\label{lemma_subsampling}
Let $M' : \mathcal{X}^{m} \rightarrow \mathcal{Y}$ be a $(\epsilon,\delta)$-DP mechanism w.r.t. a given dataset $\underline{D} \in \mathcal{X}^{m}$. Then, mechanism $M : \mathcal{X}^{n} \rightarrow \mathcal{Y}$ defined as $M:=M'\circ \text{Samp}_{n,m}$ is $(\epsilon', \delta')$-DP w.r.t. to any dataset $D \in \mathcal{X}^{n}$ such that $\underline{D}=\text{Samp}_{n,m}(D)$, where:
\begin{center}
    $\epsilon'=\log(1 + q(e^{\epsilon} -1))$, $\delta'=q\delta$, $q=m /n$.
\end{center}
\end{lemma}
\textbf{Remark.} In theoretical results, $\epsilon'$ is often approximated by $\mathcal{O}(q\epsilon)$ when $\epsilon << 1$.

\subsection{Rényi Differential Privacy}
\label{rdp_theory}

\cite{dp_deep_l} demonstrated in practice that the privacy bounds provided by standard $(\epsilon,\delta)$-DP theory (see Section~\ref{formal_dp}) often overestimate the actual privacy loss. 
In order to better express inequalities on the tails of the output distributions of private algorithms, we introduce 
the \textit{privacy loss random variable} \citep{dp_foundations,dp_deep_l, renyi_dp_sampled_gm_explicit}. Given a random mechanism $M$, let $M(D)$ and $M(D')$ be the distributions of the output when $M$ is run on $D$ and $D'$ respectively. The privacy loss $L^M_{D,D'}$ is defined as:
\begin{align}
    L^M_{D,D'}(\theta):=\log\bigg(\displaystyle\frac{M(D)(\theta)}{M(D')(\theta)}\bigg) \quad \text{where } \theta \sim M(D). \label{eq:privacy_loss}
\end{align}
The interpretation of this quantity is easy to understand: $(\epsilon, \delta)$-DP ensures that the absolute value of the privacy loss is bounded by $\epsilon$ with probability at least $(1-\delta)$ for all pairs of neighboring datasets $D$ and $D'$ \citep[Lemma 3.17]{dp_foundations}.

We will reason on the \textit{Cumulant Generating Function} (CGF) of the privacy loss, denoted $K_M$, rather than on the privacy loss $L^M$ itself.
This CGF is expressed as follows for any $\lambda>0$:
\begin{align*}
 K_M(D,D',\lambda):=\mathbb{E}_{\theta\sim M(D)}\big[e^{\lambda L^M_{D, D'}(\theta)}\big]= \mathbb{E}_{\theta\sim M(D)}\bigg[\bigg(\displaystyle\frac{M(D)(\theta)}{M(D')(\theta)} \bigg)^{\lambda}\bigg],
\end{align*}
which is also equivalent to:
\begin{align}
     K_M(D,D',\lambda)=\mathbb{E}_{\theta\sim M(D')}\bigg[\bigg(\displaystyle\frac{M(D)(\theta)}{M(D')(\theta)} \bigg)^{\lambda+1}\bigg]. \label{eq:cgf}
\end{align}

By the property of the moment generating function, $K_M(D,D',\cdot)$ fully determines the distribution of the privacy loss random variable $L^M_{D, D'}$.
We also define $K_M(\lambda):=\sup_{||D-D'||\leq 1} K_M(D,D',\lambda)$, which is the upper bound on the CGF for any pair of neighboring datasets.

We can now introduce \textit{Rényi Differential Privacy} (RDP), which generalizes DP using the Rényi divergence $D_{\alpha}$.

\begin{definition}[Rényi Differential Privacy, 
\citealp{renyi_dp}]\label{def_rdp} For any $\alpha \in (1,\infty)$ and any $\epsilon>0$, a mechanism $M :\mathcal{X}^n \rightarrow \mathcal{Y}$ is said to be $(\alpha, \epsilon)$-RDP, if for all neighboring datasets $D$ and $D'$,
\begin{align}
    D_\alpha(M(D)||M(D')):=\displaystyle\frac{1}{\alpha -1}\log \mathbb{E}_{\theta\sim M(D')}\bigg[\bigg(\displaystyle\frac{M(D)(\theta)}{M(D')(\theta)} \bigg)^{\alpha}\bigg] \leq \epsilon.\label{eq:rdp_inequality}
\end{align}
\end{definition}

Given a mechanism $M$ and a RDP parameter $\alpha$, we can thus determine from Definition~\ref{def_rdp} the lowest value of the $\epsilon$-RDP bound, denoted $\epsilon_M(\alpha)$, such that $M$ is $(\alpha, \epsilon_M(\alpha))$-RDP. Indeed, $\epsilon_M(\alpha)$ is such that:
\begin{align*}
    \epsilon_M(\alpha)= \inf_{\epsilon \in \varepsilon(M)} \epsilon \quad \text{where}\quad \varepsilon(M):=\{\epsilon>0 : \sup_{||D-D'||\leq 1} D_\alpha(M(D)||M(D')) \leq \epsilon\}.
\end{align*}

The obvious similarity between Eq.~\eqref{eq:cgf} and Eq.~\eqref{eq:rdp_inequality} shows the link between the CGF and the notion of RDP. Indeed, for any $\alpha \in (1,\infty)$, it is easy to see that $(\alpha-1)\epsilon_M(\alpha)$ is equal to $K_M(\lambda)$ where $\lambda+1=\alpha$ (restated in Lemma~\ref{lemma_conversion_rdp_cgf}).
\begin{lemma}[\textbf{Equivalence RDP-CGF}]\label{lemma_conversion_rdp_cgf} Any mechanism $M$ is $(\lambda +1, K_M(\lambda)/\lambda)$-RDP for all $\lambda >0$. 
\end{lemma}


We now recall how we can convert RDP guarantees into standard DP guarantees.

\begin{lemma}[RDP to DP conversion,
\citealp{renyi_dp}]\label{lemma_conversion_rdp_dp} If $M$ is $(\epsilon, \alpha)$-RDP, then $M$ is $(\epsilon + \log(1/\delta)/(\alpha-1),\delta)$-DP for any $0<\delta<1$.
\end{lemma}


Given Lemma~\ref{lemma_conversion_rdp_dp} and Lemma~\ref{lemma_conversion_rdp_cgf}, it is possible to find the smallest $\epsilon$ from some fixed parameter $\delta$ or the smallest $\delta$ from some fixed parameter $\epsilon$ so as to achieve $(\epsilon, \delta)$-DP:
\begin{align}
    \epsilon(\delta)&=\min_{\lambda>0}\displaystyle\frac{\log(1/\delta)+K_M(\lambda)}{\lambda}, \label{epsilon_pb}\\
    \delta(\epsilon)&=\min_{\lambda>0}e^{K_M(\lambda)-\lambda\epsilon}.\label{delta_pb}
\end{align}

Moreover, $\lambda \rightarrow K_M(\lambda)/\lambda$ is monotonous \citep[Theorem 3]{property_rdp} and $\lambda \rightarrow K_M(\lambda)$ is convex \citep[Theorem 11]{property_rdp}. This last property enables to bound $K_M$ by a linear interpolation between the values of $K_M$ evaluated at integers, as stated below: 
\begin{align}
    \forall \lambda >0, K_{M}(\lambda) \leq (1-\lambda + \lfloor\lambda\rfloor)K_{M}(\lfloor\lambda\rfloor) + (\lambda - \lfloor\lambda\rfloor)K_{M}(\lceil\lambda\rceil).\label{eq:convexity_km}
\end{align}

Therefore, Problem~\eqref{epsilon_pb} is quasi-convex and Problem~\eqref{delta_pb} is log-convex, and both can be solved if we know the expression of $K_M(\lambda)$ for any $\lambda >0$.

We provide below other useful results from RDP theory, which we will use in our privacy analysis.

\begin{lemma}[RDP Composition,
\citealp{renyi_dp}]\label{lemma_composition_rdp} Let $\alpha \in (1, \infty)$. Let $M_1$ and $M_2$ be two mechanisms such that $M_1$ is $(\alpha, \epsilon_1)$-RDP and $M_2$, which takes the output of $M_1$ as auxiliary input, is $(\alpha, \epsilon_2)$-RDP. Then the composed mechanism $M_2\circ M_1$ is $(\alpha, \epsilon_1 + \epsilon_2)$-RDP.
\end{lemma}

\begin{lemma}[RDP Gaussian mechanism,
\citealp{renyi_dp}] \label{lemma_rdp_gaussian} If $f : \mathcal{X}^n \rightarrow \mathbb{R}^d$ has $\ell_2$-sensitivity 1, then the Gaussian mechanism $G_f(\cdot):= f(\cdot) + \mathcal{N}(0, \sigma_g^2 \text{I}_d)$ is $(\alpha, \alpha/2\sigma_g^2)$-RDP for any $\alpha>1$. 
\end{lemma}

\begin{lemma}[RDP for subsampled Gaussian mechanism, 
\citealp{renyi_dp_sampled_gm_explicit}]\label{lemma_subsampling_rdp} Let $\alpha \in \mathbb{N}$ with $\alpha \geq 2$ and $0<q<1$ be a subsampling ratio. Suppose $f : \mathcal{X}^n \rightarrow \mathbb{R}^d$ has $\ell_2$-sensitivity equal to 1. Let $G'_f(\cdot):=G_f\circ\text{Samp}_q(\cdot)$ be a subsampled Gaussian mechanism. 
Then $G'_f$ is $(\alpha, \epsilon'(\alpha, \sigma_g^2))$-RDP where 
\begin{align*}
     \epsilon'(\alpha, \sigma_g^2) \leq \displaystyle \frac{1}{\alpha-1}\log\bigg(1 +2q^2 \binom{\alpha}{2}\min \{2(e^{1/\sigma_g^2}-1), e^{1/\sigma_g^2}\} +\displaystyle\sum_{j=3}^\alpha 2q^j \binom{\alpha}{j}e^{j(j-1)/2\sigma_g^2}\bigg).
 \end{align*}
\end{lemma}

\textbf{Remark.} By considering $q=o(1)$, the dominant term in the upper bound of $\epsilon'(\alpha, \sigma_g^2)$ comes from the term of the sum of the order of $q^2$. In particular, when $\sigma_g^2$ is large (i.e. high privacy regime), the term $\min \{2(e^{1/\sigma_g^2}-1), e^{1/\sigma_g^2}\}$ simplifies to $2(e^{1/\sigma_g^2}-1) \leq 4/\sigma_g^2$. This thus simplifies the whole upper bound to $\mathcal{O}(\alpha q^2/\sigma_g^2)$.

\subsection{Proof of Claim~\ref{lemma_privacy}}
\label{proof_claim}

We restate below a more formal version of Claim~\ref{lemma_privacy} along with its proof. For any $t\in [T]$, we define \textit{subversions} of algorithms \texttt{DP-SCAFFOLD} (Alg.~\ref{algo_dp_scaffold}) and \texttt{DP-FedAvg} (Alg.~\ref{algo_dp_fedavg}), which stop at round $t$ and reveal an output, either to the server or to a third party:
\begin{itemize}
    \item \textbf{To the server.} We assume that the sampling of users $C^t$ is known by the server. Formally, we define $\mathcal{A}_{\dpscaf}^t$, which outputs (reveals) $\{y_i^t, c_i^t\}_{i \in C^t}$, and $\mathcal{A}_{\dpfedavg}^t$, which outputs $\{y_i^t\}_{i \in C^t}$ (those quantities being private w.r.t. $\{D_i\}_{i \in C^t}$).
    \item \textbf{To a third party.} We define $\mathcal{\tilde{A}}_{\dpscaf}^t$, which  outputs  $(x^t, c^t)$ and $\mathcal{\tilde{A}}_{\dpfedavg}^t$, which outputs $x^t$ (those quantities being private w.r.t. $D$).
\end{itemize}

In both privacy models, \texttt{DP-SCAFFOLD} and \texttt{DP-FedAvg} can be seen as $T$ adaptive compositions of these \textit{sub}-algorithms. 

\begin{claim}[Formal version of Claim~\ref{lemma_privacy}]\label{claim_privacy_complete} For any $t\in [T]$, the following holds:
\begin{itemize}
    \item $\mathcal{A}_{\dpscaf}^t$ and $\mathcal{A}_{\dpfedavg}^t$ have the same level of privacy (towards the server),
    \item $\mathcal{\tilde{A}}_{\dpscaf}^t$ and $\mathcal{\tilde{A}}_{\dpfedavg}^t$ have the same level of privacy (towards a third party).
\end{itemize}
\end{claim}

\begin{proof}
We prove the claim by reasoning by induction on the number of communication rounds $t$. We only give the proof for the first statement (including the \texttt{DP-SCAFFOLD-warm} version). The second one can be proved in a similar manner.

First, consider $t=1$. For any $i \in C^t$, control variates $c_i^0$ are either all set to 0 (\texttt{DP-SCAFFOLD}), or $c_i^0$ are at least as private as $y_i^1$ (\texttt{DP-SCAFFOLD-warm}). The level of privacy for $\mathcal{A}_{\dpscaf}^1$ is thus \textit{fully} determined by the level of privacy of $\{y_i^1\}_{i \in C^t}$, which is the same as $\mathcal{A}_{\dpfedavg}^1$. Therefore the claim is true for $t=1$.

Then, let $t \in [T]$ and suppose that the claim is verified for all $t'<t$.
Let $i\in C^t$ and first consider $\mathcal{A}_{\dpscaf}^t$. The update of the
$i$-th user model (see Eq. \ref{eq:update_y_i}) at round $t$ shows that an
additional information leakage may come from the correction $(c^{t-1} - c_i^
{t-1})$, or more precisely from $c_i^{t-1}$ since $c^{t-1}$ is known
by the server. By assumption of induction, $c_i^{t-1}$ is also known by the
server. Therefore, using the post-processing property of DP, the $y_i^t$ as
updated in \texttt{DP-SCAFFOLD} is as private w.r.t. $D_i$ as the $y_i^t$ as
updated in \texttt{DP-FedAvg}. Besides this, the update of the $i$-th control variate \textit{fully} depends on the local updates of $y_i^t$ through the average of the DP-noised stochastic gradients calculated over the local iterations. Therefore,  considering all the contributions from $C^t$, $\mathcal{A}_{\dpfedavg}^t$ and $\mathcal{A}_{\dpscaf}^t$ have the same level of privacy.
\end{proof}

\subsection{Proof of Theorem~\ref{theorem_privacy}}
\label{proof_theorem_privacy}

\paragraph{Preliminaries.}

Lemma~\ref{lemma_subsampling_rdp} only gives an upper bound of the RDP privacy for a subsampled Gaussian mechanism when $\alpha \in \mathbb{N}$ with $\alpha \geq 2$. However we will need to optimize our privacy bound w.r.t. $\alpha \in \mathbb{R}$ with $\alpha>1$. We thus use Lemma~\ref{lemma_conversion_rdp_cgf} and the convexity of the CGF (see Eq.~\ref{eq:convexity_km}) to generalize this upper bound to the following result.

Let $\alpha \in \mathbb{R}$ with $\alpha >1$. Under the same assumptions as in Lemma~\ref{lemma_subsampling_rdp}, $G'_f$ is $(\alpha, \epsilon''(\alpha, \sigma_g^2))$-RDP with
\begin{align}
\label{eq:comp_general_alpha}
    \epsilon''(\alpha, \sigma_g^2) \leq (1 - \alpha +\lfloor \alpha\rfloor)\frac{\lfloor\alpha\rfloor - 1}{\alpha - 1}\epsilon'(\lfloor\alpha\rfloor, \sigma_g^2) + (\alpha - \lfloor \alpha \rfloor)\frac{\lceil\alpha\rceil - 1}{\alpha - 1}\epsilon'(\lceil\alpha\rceil, \sigma_g^2),
\end{align}
where $\epsilon'(\cdot, \sigma_g^2)$ admits the upper bound given in Lemma~\ref{lemma_subsampling_rdp}.

\paragraph{Details of the proof.}

Our privacy analysis assumes that the query function has sensitivity 1, since the calibration of the Gaussian noise is locally adjusted in our algorithms with the constant $S=2\mathcal{C}/s\r$ (see Section~\ref{dp_framework}). We simply denote by $G$ the Gaussian mechanism with variance $\sigma_g^2$, which is $(\alpha, \alpha/2\sigma_g^2)$-RDP (Lemma~\ref{lemma_rdp_gaussian}). Below, we first prove privacy guarantees towards a third party observing only the final result, and then deduce the guarantees towards the honest-but-curious server.

\paragraph{Step 1: data subsampling.}

Let $t\in[T]$ be an arbitrary round.
We first provide an upper bound $\epsilon_a$ for the privacy loss \textit{after the aggregation} by the server of the $lM$ individual contributions (line~\ref{aggregation_step0} in Alg.~\ref{algo_dp_scaffold}), thanks to the local addition of noise.

Let $i \in C^t$, $\alpha>1$. We denote by $\epsilon_i(\alpha)$ the
$\alpha$-RDP budget (w.r.t. $D_i$) used to ``hide'' the individual contribution of the $i$-th user from the server. This contribution is the result of the composition of $K$ adaptative $s$-subsampled mechanisms $G$:
\begin{itemize}
    \item We first obtain an upper RDP bound for the $s$-subsampled mechanism with Lemma~\ref{lemma_subsampling_rdp}. Suppose first $\alpha\in \mathbb{N}$ and $\alpha\geq 2$, which is the case covered by Lemma~\ref{lemma_subsampling_rdp}. Under Assumption~\ref{ass:dplebvels}-(i) and Assumption~\ref{ass:dplebvels}-(iii), the resulting mechanism is $(\alpha, \mathcal{O}(s^2\alpha/\sigma_g^2))$-RDP. To extend this result to $\alpha > 1$, we use the result provided in \eqref{eq:comp_general_alpha}: by factoring by $s^2/\sigma_g^2$ in the upper bound of $\epsilon''(\alpha, \sigma_g^2)$, and bounding the rest of the inequality (a convex combination between $ \lfloor \alpha \rfloor(\lfloor \alpha \rfloor -1)/(\alpha -1)$ and $\lceil \alpha \rceil(\lceil \alpha \rceil -1)/(\alpha -1)$) by $\alpha+1$, we also obtain that this mechanism is $(\alpha, \mathcal{O}(s^2(\alpha+1)/\sigma_g^2))$-RDP.
    \item We then use the result of Lemma~\ref{lemma_composition_rdp} for the RDP composition rule over the $K$ local iterations, which gives that $\epsilon_i(\alpha) \leq \mathcal{O}(Ks^2(\alpha+1)/\sigma_g^2)$.
\end{itemize}

We now consider the aggregation step. Taking into account all the contributions of the users from $C^t$, we get a Gaussian noise of variance $S^2\sigma_a^2$ where $\sigma_a^2=\frac{1}{lM}\sigma_g^2$. 
Note that the sensitivity of the aggregation (w.r.t. the joint dataset $D$) is $lM$ times smaller than when considering an individual contribution. Therefore, with the previous approximation, the aggregated contributions satisfy $(\alpha, \mathcal{O}(Ks^2(\alpha+1)/lM\sigma_g^2))$-RDP w.r.t. $D$.

After converting this result into a DP bound (Lemma~\ref{lemma_conversion_rdp_dp}), we get that for any $0<\delta'<1$, the aggregation at line~\ref{aggregation_step0} in Alg.~\ref{algo_dp_scaffold} is $(\epsilon_a(\alpha, \delta'), \delta')$-DP w.r.t. $D$ where $\epsilon_a(\alpha, \delta')=\mathcal{O}\big(\frac{Ks^2(\alpha+1)}{lM\sigma_g^2} + \frac{\log(1/\delta')}{\alpha-1}\big)$. 

\textit{Without approximation:} we would obtain at this step an exact upper bound $\epsilon_a(\alpha, \delta')=K\epsilon''(\alpha,lM\sigma_g^2) + \frac{\log(1/\delta')}{\alpha-1}$.

\paragraph{Step 2: user subsampling.}

In order to get explicit bounds, we then use classical DP tools to estimate an upper DP bound after $T$ rounds taking into account the amplification by subsampling from the set of users. Remark that these tools are however sub-optimal for practical implementations \citep[Section 5.1.]{dp_deep_l}.

\begin{itemize}
    \item Using Lemma~\ref{lemma_subsampling}, the subsampling of users enables a gain of privacy of the order of $l$, which gives $\big(\mathcal{O}(l\epsilon_a(\alpha, \delta')), l\delta'\big)$-DP.
    \item Using Lemma~\ref{lemma_strong_composition}, we compose this mechanism over $T$ iterations, which under Assumption~\ref{ass:dplebvels}-(ii) gives for any $\delta''>0$, $\big(\mathcal{O}(\sqrt{T\log(1/\delta'')}l\epsilon_a(\alpha, \delta')), Tl\delta' + \delta''\big)$-DP.
\end{itemize}

\textit{Without approximation}: the mechanism is $\big(\epsilon^*(\alpha,\delta') \sqrt{2T\log(1/\delta'')} + T\epsilon^*(\alpha,\delta') (e^{\epsilon^*(\alpha,\delta')} -1), Tl\delta' + \delta''\big)$ where $\epsilon^*(\alpha,\delta')=\log(1 + l(e^{\epsilon_a(\alpha,\delta')} -1))$.

\paragraph{Step 3: setting parameters.}

We denote $\epsilon_T(\alpha,\delta',\delta'')=l\sqrt{T\log(1/\delta'')}\big(\frac{Ks^2(\alpha+1)}{lM\sigma_g^2} + \frac{\log(1/\delta')}{\alpha-1}\big)$. Given what is stated above, the final output of the algorithm is $(\mathcal{O}(\epsilon_T),Tl\delta' + \delta'')$-DP.

Considering our final privacy budget $\delta$, we \textit{arbitrarily} fix $\delta':=\delta/2Tl$ and $\delta'':=\delta/2$. We now aim to find an expression of $\sigma_g$ such that the privacy bound is minimized. \textit{By considering the approximated bound}, this gives the following minimization problem:
\begin{align*}
    \min_{\alpha >1} \epsilon_T(\alpha):=l\sqrt{T\log(2/\delta)}\bigg(\frac{Ks^2(\alpha+1)}{lM\sigma_g^2} + \frac{\log(2Tl/\delta)}{\alpha-1}\bigg).
\end{align*}
Using DP rather than RDP enables to solve this minimization problem pretty easily since only the second factor in $\epsilon_T(\alpha)$ depends on $\alpha$, that is:
\begin{align*}
    \min_{\alpha >1} \tilde{\epsilon_T}(\alpha):=\frac{Ks^2(\alpha+1)}{lM\sigma_g^2} + \frac{\log(2Tl/\delta)}{\alpha-1}.
\end{align*}

By omitting constants, we obtain the expression for the minimum value of $\epsilon_T(\alpha)$:
\begin{align*}
    \tilde{\epsilon}=l\sqrt{T\log(2/\delta)}\bigg(\frac{s\sqrt{K\log(2Tl/\delta)}}{\sigma_g\sqrt{lM}}+\frac{Ks^2}{lM\sigma_g^2}\bigg).
\end{align*}
Under Assumption~\ref{ass:dplebvels}-(iii), we can bound the second term by the first one, which gives:
\begin{align*}
    \tilde{\epsilon}=\mathcal{O}\bigg(\frac{s\sqrt{lTK\log(2/\delta)\log(2Tl/\delta)}}{\sigma_g\sqrt{M}}\bigg).
\end{align*}
We then invert the formula of this upper bound of $\tilde{\epsilon}$ to express $\sigma_g$ as a function of a given privacy budget $\epsilon$:
\begin{align*}
     \sigma_g=\Omega\big(s\sqrt{lTK\log(2Tl/\delta)\log(2/\delta)}/\epsilon\sqrt{M}\big),
\end{align*}
which proves that the algorithm is $(\mathcal{O}(\epsilon),\delta)$-DP towards a third party observing its final output.

\textit{Without approximation:} the minimization problem is much more complex and has to be solved numerically
\begin{align*}
    \min_{\alpha>1,\delta'>0,\delta''>0} \epsilon^*(\alpha,\delta') \sqrt{2T\log(1/\delta'')} + T\epsilon^*(\alpha,\delta') (e^{\epsilon^*(\alpha,\delta')} -1) \quad \text{s.t. } \delta=Tl\delta' + \delta'',
\end{align*}
or:
\begin{align*}
    \min_{\alpha>1,x \in (0,1)} \epsilon^*(\alpha,x\delta/Tl) \sqrt{2T\log(1/(1-x)\delta)} + T\epsilon^*(\alpha,x\delta/Tl) (e^{\epsilon^*(\alpha,x\delta/Tl)} -1).
\end{align*}

\paragraph{Extension to privacy towards the server.} The crucial difference with the third party case is that the server observes individual contributions and knows which users are subsampled at each step. Removing the privacy amplification effect of the $l$-subsampling of users and the aggregation step, the minimization problem becomes
\begin{align*}
    \min_{\alpha >1} \epsilon_T(\alpha):=\sqrt{T\log(2/\delta)}\bigg(\frac{Ks^2\alpha}{\sigma_g^2} + \frac{\log(2T/\delta)}{\alpha-1}\bigg),
\end{align*}
where the minimizing value can be approximated by:
\begin{align*}
    \tilde{\epsilon}=\sqrt{T\log(2/\delta)}\bigg(\frac{s\sqrt{K\log(2T/\delta)}}{\sigma_g}+\frac{Ks^2}{\sigma_g^2}\bigg).
\end{align*}

Under Assumption~\ref{ass:dplebvels}-(iii), we can bound the second term by the first one:
\begin{align*}
    \tilde{\epsilon}=\mathcal{O}\bigg(\frac{s\sqrt{TK\log(2/\delta)\log(2T/\delta)}}{\sigma_g}\bigg),
\end{align*}
which proves that we obtain $(\mathcal{O}(\epsilon_s), \delta_s)$-DP towards the server where  $\epsilon_s=\mathcal\epsilon\sqrt{\frac{M}{l}}$ and $\delta_s=\frac{\delta}{2}(\frac{1}{l} + 1)$.

\paragraph{Finer results for amplification by subsampling.} 

To establish privacy towards a third party, it is actually possible to combine the subsampling ratios (user and data) to determine a bound upon the subsampling of data \textit{directly} from $D$ and thus to quantify a more precise gain in privacy \citep{fl_shuffle}. The difficulty in this setup is that this combined subsampling is \textit{not uniform overall}, which requires extending the proof of Lemma~\ref{lemma_subsampling} as done by \citet{fl_shuffle} in the case of classical differential privacy.

\paragraph{Implementation.}
In practice, we determine a \textit{RDP} upper bound at Step 2, by using the theorem proved by \citet{renyi_dp_sampled_gm_explicit} (which is not restricted to Gaussian mechanisms) with the exact RDP bound obtained at Step 1 and sampling parameter $l$. This result being accurate only for $\alpha \in \mathbb{N}\setminus\{0,1\}$, we obtain an natural extension of the bound for any $\alpha>1$ with Eq~\eqref{eq:convexity_km}. Then, we invoke Lemma~\ref{lemma_composition_rdp} to obtain the final RDP bound after $T$ communication rounds, for any $\alpha>1$. Under fixed privacy parameter $\delta$ (chosen as $1/MR$ in our experiments), we finally obtain \textit{the} minimal value for $\epsilon$ w.r.t. $\alpha \in (1,\infty)$, which is determined by Eq~\eqref{epsilon_pb}. The last step is done by using a fine grid search over parameter $\alpha$.

\section{PROOF OF UTILITY}
\label{app:othertheoreticalresults}

In this section, we provide the proof of our utility results. We first establish in Section~\ref{subsec_utility_preliminaries} some preliminary results about the impact of DP noise over stochastic gradients. In Section~\ref{subsec_utility_dp_scaffold}, we provide the complete version of our utility result for \texttt{DP-SCAFFOLD-warm} (Theorem~\ref{theorem_utility_scaffold}), from which Theorem~\ref{utility_result} is an immediate corollary. We prove this theorem for convex local loss functions in Section~\ref{subsec_utility_dp_scaffold_convex} and non-convex loss functions in Section~\ref{subsec_utility_dp_scaffold_non_convex}. We finally state in Section~\ref{subsec_utility_dp_fedavg} our complete result for \texttt{DP-FedAvg} (Theorem~\ref{theorem_utility_fedavg}).

For any $\mathcal{C},\sigma_g>0$, we define $\Sigma_g(\mathcal{C}) := 2\mathcal{C}\sqrt{2d}\sigma_g/s\r$. We recall that we assume that $F$ is bounded from below by $F^*= F(x^*)$, for an $x^* \in \mathbb{R}^d$. 

\subsection{Preliminaries}
\label{subsec_utility_preliminaries}

\paragraph{Properties of DP-noised stochastic gradients.}

Let $i \in [M]$, $x \in \mathbb{R}^d$, $S_i \subset D_i$ and $\mathcal{C},\sigma_g>0$. Suppose Assumptions~\ref{ass:smooth_f_i} and \ref{ass:stoch_grad}.3 are verified (the last assumption ensures that the clipping on per-example local gradients with threshold $\mathcal{C}$ is not effective). 

We recall below the expression of $\tilde{H}_i(x)$ from Section~\ref{algo_description}, which is the noised version of the local gradient $H_i(x)$ of the $i$-th user over $S_i$ evaluated at $x$ (omitting index $k$):

\begin{equation*}
    \tilde{H}_i(x):= H_i(x) + \frac{2\mathcal{C}}{s\r} \mathcal N (0,\sigma^2_g), \quad \text{where} \quad H_i(x):= \frac{1}{s \r}\sum_{d_j^i \in S_i}\nabla f_i(x, d^i_j).
\end{equation*}

We recall that the $\ell_2$-sensitivity of $H_i(x)$ w.r.t. $S_i$ is upper bounded by $2\mathcal{C}/s\r$, which explains the scaling of the Gaussian noise in the expression of $\tilde{H}_i(x)$. Since the variance of $\mathcal{N}(0, \text{I}_d)$ is $2d$, the following statement holds directly:
\begin{align*}
    \mathbb{E}\big[\tilde{H}_i(x)\big]=H_i(x)
    \text{ and }\mathbb{E}\big[||\tilde{H}_i(x)- H_i(x)||^2\big] \leq
     \displaystyle\frac{8\mathcal{C}^2d\sigma_g^2}{s^2\r^2}=\Sigma_g(\mathcal{C})^2. 
\end{align*}

By combining our utility assumptions with the result stated above, we can deduce the following lemma. 

\begin{lemma}[Regularity of DP-noised stochastic gradients]\label{lemma_reg_h_tilde}
Under Assumptions~\ref{ass:smooth_f_i} and \ref{ass:stoch_grad}, for any iteration $t \in [T], k\in [K]$,
\begin{enumerate}
    \item$\mathbb{E}\big[\tilde{H}_i^k(y_i^{k-1})|y_i^{k-1}\big]=\nabla F_i(y_i^{k-1})$,
    \item $\mathbb{E}\big[||\tilde{H}_i^k(y_i^{k-1}) - \nabla F_i(y_i^{k-1})||^2|y_i^{k-1}\big] \leq 
     \displaystyle\frac{\varsigma^2}{s\r} + \Sigma_g^2(\mathcal{C})$.
\end{enumerate}
\end{lemma}

The proof of Lemma~\ref{lemma_reg_h_tilde} is easily obtained by conditioning on the two sources of randomness (i.e., mini-batch sampling and Gaussian noise) which are \emph{independent}, thus the variance is additive. This result can be seen as a \textit{degraded} version of Assumption~\ref{ass:stoch_grad} due to the local injection DP noise, a fact that we will strongly leverage to derive convergence rates.

We now enumerate several statements that will be used in the utility proof. First, Lemma~\ref{lemma_nesterov} enables to control $||\nabla F||^2$ using the assumption of smoothness over the local loss functions. Second, Lemma~\ref{lemma_mean_variance} provides separation inequalities of mean and variance \cite[Lemma 4]{fl_scaffold}, which enables to state a result on quantities of interest in Corollary~\ref{corollary_noise}.

\begin{lemma}[Nesterov inequality] \label{lemma_nesterov} Suppose Assumption~\ref{ass:smooth_f_i} is verified and assume that for all $i \in [M]$, $F_i$ is convex. Then,
    \begin{align*}
        \forall x \in \mathbb{R}^d, ||\nabla F(x)||^2 \leq 2 \nu (F(x) -F^*).
    \end{align*}
\end{lemma}
\begin{proof} Let $x \in \mathbb{R}^d$.
\begin{align*}
    ||\nabla F(x)||^2 & = ||\nabla F(x) -\nabla F(x^*)||^2=||\frac{1}{M}\sum_{i=1}^M \nabla F_i(x) -\nabla F_i(x^*)||^2\\
    & \leq \frac{1}{M}\sum_{i=1}^M||\nabla F_i(x) -\nabla F_i(x^*)||^2 && \text{(Jensen inequality)}\\
    &\leq 2\nu (F(x) -F^*) && \text{\cite[Theorem 2.1.5]{nesterov2018lectures}}
\end{align*}
\end{proof}

\begin{lemma}[Separating mean and variance] Let $(A_1,...,A_n)$ be $n$ random variables in $\mathbb{R}^d$ \emph{not necessarily independent}.\\
\label{lemma_mean_variance}
1. Suppose that their mean is  $\mathbb{E}[A_i]=a_i$ and their variance is uniformly bounded, i.e. for all $i\in [n]$,  $\mathbb{E}[||A_i - a_i||^2]\leq \sigma^2_A$. Then,
\begin{align*}\mathbb{E}\left\|\displaystyle\sum_{i=1}^n A_i\right\|^2 \leq \left\|\displaystyle\sum_{i=1}^n a_i\right\|^2 + n^2\sigma^2_A.
\end{align*}

2. Suppose that their conditional mean is $\mathbb{E}[A_i|A_{i-1},...A_1]=a_i$ and their variance is uniformly bounded, i.e. for all $i\in [n]$,  $\mathbb{E}[||A_i - a_i||^2]\leq \sigma^2_A$. Then,
\begin{align*}
    \mathbb{E}\left\|\displaystyle\sum_{i=1}^n A_i\right\|^2 \leq 2 \left\|\displaystyle\sum_{i=1}^n a_i\right\|^2 + 2n\sigma^2_A.
\end{align*}
\end{lemma}

\begin{corollary} 
\label{corollary_noise}
Let $t \in [T]$. In the following statements, the expectation is taken w.r.t.~the randomness from their local data sampling and from the Gaussian DP noise, conditionally to the users' sampling  $C^t$ and initial value of variables $y_i$, that is $y^0=x^{t-1}$ (same for all users). We have:

\begin{itemize}
\item $ \mathbb E\left[\Big\|\displaystyle \frac{1}{KlM}\sum_{i \in C^t}\sum_{k \in [K]} (\tilde{H}_i^k (y_i^{k-1}) - \nabla F_i(y_i^{k-1}))\Big\|^2 \bigg | C^t, y^0 \right]     {\le} \displaystyle\frac{\Sigma_g^2(\mathcal{C}) + \varsigma^2/s\r}{KlM} $,
    \item $\mathbb{E}\bigg[\Big\|\displaystyle\frac{1}{lM}\sum_{i \in C^t}c_i^t - \mathbb E[c^t_i]\Big\|^2|C^t,y^0\bigg] \le 
    \frac{\Sigma_g^2(\mathcal{C}) + \varsigma^2/s\r}{KlM}$,
    \item $\mathbb{E}\bigg[\Big\|c^t - \mathbb{E}[c^t]\Big\|^2|y^0\bigg] \leq \displaystyle\frac{\Sigma_g^2(\mathcal{C}) + \varsigma^2/s\r}{KlM}$.
\end{itemize}
\end{corollary}

\begin{proof}  

\textbf{First inequality.} We define a random variable $A$ such as $A:=\frac{1}{KlM}\sum_{i \in C^t}\sum_{k \in [K]} A_{i,k}$, with $A_{i,k}:= \tilde{H}_i^k(y_i^{k-1})$. 

From Lemma~\ref{lemma_reg_h_tilde}, we have  that  for all $i\in C^t, k\in [K]$: $\mathbb{E}[A_{i,k}|y_i^{k-1}]=\mathbb{E}[\tilde{H}_i^k(y_i^{k-1})| y_i^{k-1}]= \nabla F_i(y_i^{k-1})$. Furthermore, by Lemma~\ref{lemma_reg_h_tilde},  $\mathbb{E}\left[||\tilde{H}_i^k(y_i^{k-1}) - \nabla F_i(y_i^{k-1})||^2|y_i^{k-1} \right]\leq \Sigma_g^2(\mathcal{C}) + \varsigma^2/s\r $.

Furthermore, (a) for $i,j\in C^t$,  $\sum_{k\in [K]} A_{i,k}-\nabla F_i(y_i^{k-1})$ and $\sum_{k\in [K]} A_{j,k}-\nabla F_j(y_j^{k-1})$ are independent conditionally to $y^0$; (b) for any $i\in C^t$,  $(A_{i,k}-\nabla F_i(y_i^{k-1}))_{k\in [K]}$ is a martingale increment, i.e., $\mathbb E[A_{i,k}-\nabla F_i(y_i^{k-1})|\sigma(\{y_i^{k'}\}_{k'\in[k-1]})]=0$.
Consequently:
\begin{small}
\begin{align*}
     \mathbb E\left[\Big\| \frac{1}{KlM}\sum_{i \in C^t}\sum_{k \in [K]} A_{i,k} - \nabla F_i(y_i^{k-1})\Big\|^2 \bigg | C^t, y^0 \right]
     & \overset{(a)}{=} \frac{1}{(KlM)^2} \sum_{i \in C^t} \mathbb E\left[\Big\| \sum_{k \in [K]} A_{i,k} - \nabla F_i(y_i^{k-1})\Big\|^2 \bigg | y^0 \right] \\
    & \overset{(b)}{=} \frac{1}{(KlM)^2} \sum_{i \in C^t} \sum_{k \in [K]}  \mathbb E\left[\underbrace{ \mathbb{E}\left[\Big\| A_{i,k} - \nabla F_i(y_i^{k-1})\Big\|^2 \bigg |  \sigma\big(\{y_i^{k'}\}_{k'\in[k]}\big)\right]}_{\leq \Sigma_g^2(\mathcal{C}) + \varsigma^2/s\r}\bigg | y^0 \right] \\
    &{\le} \frac{\Sigma_g^2(\mathcal{C}) + \varsigma^2/s\r}{KlM} .
\end{align*}
\end{small}

To prove the second equality, we need to ``iteratively'' expand the squared norm and take the  conditional expectation w.r.t.  $\sigma\big(\{y_i^{k'}\}_{k'\in[k]}\big)$ for $k=K,K-1, \hdots, 1$ and use the martingale property to obtain that the scalar products are equal to 0.

\textbf{Second inequality.} We recall that for any $i \in C^t$,  $c_i^t=\frac{1}{K}\sum_{k=1}^K\tilde{H}_i^k(y_i^{k-1})$. Thus $\frac{1}{lM}\sum_{i \in C^t}c_i^t = A$ and we can directly use the results from the first inequality. 

\textbf{Third inequality.} We recall that $c^t=\frac{1}{M}\sum_{i=1}^M c_i^t$ (even if local control variates are not updated). Therefore, we can use the previous results and take the expectation over $C^t$, which gives:
\begin{align*}
    \mathbb{E}\bigg[\Big\|c^t - \mathbb{E}[c^t]\Big\|^2|y^0\bigg] \leq  \displaystyle\frac{\varsigma^2/s\r + \Sigma^2_g(\mathcal{C})}{KM} \leq \frac{\varsigma^2/s\r + \Sigma^2_g(\mathcal{C})}{KlM}.
\end{align*}
\end{proof}

\subsection{Theorem of Convergence for \texttt{DP-SCAFFOLD-warm}}
\label{subsec_utility_dp_scaffold}

\begin{theorem}[Utility rates for \texttt{DP-SCAFFOLD-warm}, $\sigma_g$ chosen arbitrarily]\label{theorem_utility_scaffold}Let $\sigma_g, \mathcal{C}>0$, $x^0 \in \mathbb{R}^d$. Suppose we run \texttt{DP-SCAFFOLD-warm}$(T,K,l,s,\sigma_g,\mathcal{C})$ with initial local controls such that $c_i^0=\frac{1}{K}\sum_{k=1}^K\tilde{H}_i^k(x^0)$ for any $i \in [M]$. Under Assumptions~\ref{ass:smooth_f_i} and \ref{ass:stoch_grad}, we  consider the sequence of iterates $(x^t)_{t\geq0}$ of the algorithm, starting from~$x^0$. 
\begin{enumerate}[topsep=0pt]
    \item If $F_i$ are \textbf{$\mu$-strongly convex} ($\mu>0$), $\eta_g=\sqrt{lM}$,
    $\bar{\eta}_l = \min\big(\frac{l^{\frac{2}{3}}}{24\nu K\eta_g},\frac{l}{54\mu K \eta_g}\big)$,    and $T \geq \max (\frac{54}{l},\frac{24 \nu}{\mu l^{\frac{2}{3}}} )$,
    then there exist weights $\{w_t\}_{t \in [T]}$ and local step-sizes $\eta_l\leq \bar{\eta}_l$ such that the averaged output of \texttt{DP-SCAFFOLD-warm}$(T,K,l,s,\sigma_g,\mathcal{C})$, defined by $\overline{x}^T=\sum_{t \in [T]} w_t x^t$, has expected excess of loss such that:
    \begin{center}
        $\mathbb{E}[F(\overline{x}^T)] -F(x^*) \leq \mathcal{O}\bigg(\displaystyle\frac{\varsigma^2/s\r + \Sigma^2_g(\mathcal{C})}{\mu T K lM} + \mu D_0^2\exp(-\min\big(\frac{l}{108}, \frac{\mu l^{\frac{2}{3}}}{48 \nu}\big)T)\bigg)$,
    \end{center}
    \item If $F_i$ are \textbf{convex}, $\eta_g=\sqrt{lM}$, $\bar{\eta}_l=\min\big(\frac{l^{\frac{2}{3}}}{24\nu K\eta_g}, \frac{D_0 \sqrt{KlM}}{K \eta_g\sqrt{T(\varsigma^2/sR + \Sigma_g(\mathcal{C})^2 )}}\big)$ and $T \geq 1$, then there exist weights $\{w_t\}_{t \in [T]}$ and local step-sizes $\eta_l\leq \bar{\eta}_l$ such that the averaged output of \texttt{DP-SCAFFOLD-warm}$(T,K,l,s,\sigma_g,\mathcal{C})$, defined by $\overline{x}^T=\sum_{t \in [T]} w_t x^t$, has expected excess of loss such that:
    \begin{center}
        $\mathbb{E}[F(\overline{x}^T)] -F(x^*) \leq \mathcal{O}\bigg(\displaystyle\frac{\varsigma/\sqrt{s\r} + \Sigma_g(\mathcal{C})}{\sqrt{TKlM}}D_0 + \frac{\nu}{l^{\frac{2}{3}}T}D_0^2 \bigg)$,
    \end{center}
    \item If $F_i$ are \textbf{non-convex}, $\eta_g=\sqrt{lM}$, $\bar{\eta}_l = \frac{l^{\frac{2}{3}}}{24\nu K\eta_g}$ and $T \geq 1$, then there exist weights $\{w_t\}_{t \in [T]}$ and local step-sizes $\eta_l\leq \bar{\eta}_l$ such that the randomized  output of \texttt{DP-SCAFFOLD-warm}$(T,K,l,s,\sigma_g,\mathcal{C})$, defined by $\{\overline{x}^T= x^t \text{ with probability } w_t \text{ for all } t\}$, has expected squared gradient of the loss such that:
\begin{center}
    $\mathbb{E}||\nabla F(\overline{x}^T)||^2 \leq \mathcal{O}\bigg(\displaystyle\frac{\varsigma/\sqrt{s\r} + \Sigma_g(\mathcal{C})}{\sqrt{TKlM}}\sqrt{F_0} +\frac{\nu}{l^{\frac{2}{3}} T}F_0\bigg)$,
\end{center}
\end{enumerate}
where $D_0=||x^0 - x^*||$ and $F_0:=F(x^0) - F^*$.
\end{theorem}

We recover the result of Theorem~\ref{utility_result} for \texttt{DP-SCAFFOLD-warm} where $F_i$ are convex by setting $\sigma_g=\sigma_g^*$ where $\sigma_g^*:=s\sqrt{lTK\log(2Tl/\delta)\log(2/\delta)}/\epsilon\sqrt{M}$, which gives $\Sigma_g(\mathcal{C})=2\mathcal{C}d\sqrt{2lTK\log(2Tl/\delta)\log(2/\delta)}/\epsilon\r\sqrt{M}$ (with numerical constants omitted for the asymptotic bound).

\subsection{Proof of Theorem~\ref{theorem_utility_scaffold} (Convex case)}
\label{subsec_utility_dp_scaffold_convex}

In this section, we give a detailed proof of convergence of \texttt{DP-SCAFFOLD-warm} with convex local loss functions. Our analysis is adapted from the proof given by \cite{fl_scaffold} without DP noise, but requires original modifications (see below). Throughout this part, we re-use the notations from Section~\ref{algo_description}.

\paragraph{Summary of the main steps.}

Let $t \in [T]$ be an arbitrary communication round of the algorithm. We detail below the updates that occur at this round.
\begin{itemize}
    \item Let $i \in C^t$. Starting from $y_{i}^0=x^{t-1}$, the random variable $y_i$ is updated at local step $k\in [K]$ such that $y_{i}^k:=y_{i}^{k-1} - \eta_l v_{i,k}^t$ where $v_{i,k}^t=\tilde{H}_i^k(y_{i}^{k-1}) - c_i^{t-1} + c^{t-1}$.
    \item Then we define the local control variate $\tilde{c}_i^t$ for this user by:
    \begin{align*}
        \tilde{c}_i^t:=c^{t-1} - c_i^{t-1} + \displaystyle\frac{1}{K \eta_l}(x^{t-1}- y_i^K)=\frac{1}{K}\sum_{k=1}^K\tilde{H}_i^k(y_i^{k-1}).
    \end{align*}
    \item For any $i \in [M]$, we update the control variate $c_i^t$ such that:
    \begin{itemize}
        \item $c_i^t:=\tilde{c}_i^t$ if $i \in C^t$,
        \item $c_i^t:=c_i^{t-1}$ otherwise.
    \end{itemize}
    \item Finally, the global update is computed as:
    \begin{center}
        $x^t=x^{t-1}+\displaystyle\frac{\eta_g}{lM}\sum_{i \in C^t}(y_i^K - x^{t-1})$
        and $c^t=\displaystyle\frac{1}{M}\big(\sum_{i \in C^t}c_i^t + \sum_{i \notin C^t}c_i^{t-1}\big)$.
    \end{center}
\end{itemize}

To keep track of the lag in the update of $c_i^t$, we introduce $\alpha_{i,k-1}^t$ defined for any $i \in[M]$, any $t \in [T]$ and any $k \in [K]$ by:
\begin{align*}
    \alpha_{i,k-1}^t =\left\{
    \begin{array}{ll}
        y_i^{k-1} & \text{ if } i \in C^t \\
        \alpha_{i,k-1}^{t-1}& \text{ otherwise}
    \end{array}
        \right.
\end{align*}
with $\alpha_{i,k-1}^0=x^0$.

We hence have the following property for any $i \in [M]$ and any $t\in [T]$: $c_i^t= \frac{1}{K}\sum_{k=1}^K\tilde{H}_i^k(\alpha_{i,k-1}^t)$.

\textbf{Additional definitions.}
\begin{itemize}[noitemsep,topsep=0pt]
    \item Model gap: $\Delta x^t:=x^t -x^{t-1}$,
    \item Global step-size: $\tilde{\eta}:=K\eta_l \eta_g$ which gives $\Delta x^t = -\displaystyle\frac{\tilde{\eta}}{KlM}\sum_{k\in[K], i \in C^t} \tilde{H}_i^k (y_i^{k-1}) + c^{t-1} - c_i^{t-1}$,
    \item User-drift: $\mathcal{E}_t:=\displaystyle\frac{1}{KM}\sum_{k=1}^K \sum_{i=1}^M \mathbb{E} || y_i^k - x^{t-1}||^2$, 
    \item Control lag: $\mathcal{F}_t:=\displaystyle\frac{1}{KM}\sum_{k=1}^K \sum_{i=1}^M \mathbb{E} || \alpha_{i,k-1}^t - x^{t}||^2$ with $\mathcal{F}_0=0$.
\end{itemize}

\paragraph{Originality of the proof.} 

The proof substantially differs form the proof by \citet{fl_scaffold} in the convex case. Indeed,  \citet{fl_scaffold} control a combination of the quadratic distance to the optimum and a control of the deviation between the controls   and the gradients at the optimal point $\|c_i^t -\nabla F_i(x^*)\|.$ Leveraging such a quantity in our proof would  result in a worse upper bound on the utility than the one we get, as either the noise added to ensure DP (if $c_i^0$ is defined w.r.t. a noised gradient) or the heterogeneity (if $c_i^0=0$)  would also appear in the initial condition $\|c_i^t -\nabla F_i(x^*)\|.$
On the other hand, in our approach, we combine the  quadratic distance to the optimum to a control of the lag and user-drift. In some sense this resembles some aspects of the proof in the non-convex regime in \citep{fl_scaffold}, in which the excess risk ($F(x^t) - F^*$) is combined with the lag. Nevertheless, our result (in the convex case), strongly leverages the convexity of the function in the proof.

\paragraph{Details of the proof.}

The idea of the proof is to find a contraction inequality involving $||x^t - x^*||^2$, $\mathbb{E}[F(x^{t-1})] - F(x^*)$, $\mathcal{F}_t$ and $\varsigma/\sqrt{s\r} + \Sigma_g(\mathcal{C})$. To do so, we will first bound the variance of the server's update. Then we will see how the control lag evolves through the communication rounds. We will also bound the user drift. To make the proof more readable, the index $t$ may be omitted on random variables when the only communication round that is considered is the $t$-th one.

\begin{lemma}[{Variance of the server's update}]\label{lemma_variance_server}
$\forall \tilde{\eta} \in [0, 1/\nu]$

\begin{align*}
\mathbb{E}||x^t - x^{t-1}||^2 \leq 4 \tilde{\eta}^2\nu^2 \mathcal{E}_t + 8 \nu^2 \tilde{\eta}^2 \mathcal{F}_{t-1} + 8 \nu \tilde{\eta}^2 \mathbb{E}\big(F(x^{t-1}) -F(x^*)\big) + \frac{9\tilde{\eta}^2}{KlM}(\varsigma^2/s\r + \Sigma_g^2(\mathcal{C})).
\end{align*}
\end{lemma}

\begin{proof} We consider the model gap $\Delta x^t = x^t-x^{t-1}$.
\begin{align*}
    \mathbb{E}||\Delta x^t||^2=\tilde{\eta}^2 \mathbb{E}\Big\|\underbrace{\bigg(\frac{1}{KlM}\sum_{k\in [K], i\in C^t} \tilde{H}_i(y_i^{k-1})\bigg)}_{A_1} + \underbrace{c^{t-1}}_{A_2} -\underbrace{\frac{1}{lM}\sum_{i \in C^t} c_i^{t-1}}_{A_3} \Big\|^2
\end{align*}

We combine Lemma~\ref{lemma_mean_variance}-1 on $A_1, A_2, A_3$ with Corollary~\ref{corollary_noise} which controls their individual variance (conditionally to the users' sampling and the local parameters) by $\frac{\varsigma^2/s\r + \Sigma_g^2(\mathcal{C})}{KlM}$. We first get rid of the terms related to the variance of the data sampling and the DP noise, before bounding the quantities of interest. It leads to:

\begin{align*}
    \mathbb{E}||\Delta x^t||^2 & = \tilde{\eta}^2\mathbb{E}\left[  \mathbb E \left[ \bigg\| \bigg(\frac{1}{KlM}\sum_{k\in [K], i\in C^t} \tilde{H}_i(y_i^{k-1})\bigg)+ c^{t-1}- \frac{1}{lM}\sum_{i \in C^t} c_i^{t-1} \bigg\|^2 \bigg| C^t, y^0 \right]  \right]\\ 
    & \le \tilde{\eta}^2\mathbb{E}\left[\Big\|\frac{1}{KlM}\sum_{k\in [K], i\in C^t} \mathbb{E}[\tilde{H}_i(y_i^{k-1})|y^0] + \mathbb{E}[c^{t-1}| y^0] - \mathbb{E}[c_i^{t-1}|y^0]\Big\|^2\right] + \frac{9\tilde{\eta}^2}{KlM}(\varsigma^2/s\r+ \Sigma_g^2(\mathcal{C})),
\end{align*}
where the inequality is given by Lemma~\ref{lemma_mean_variance}-1.

For any $i \in C^t, k \in [K]$, we have $\mathbb{E}[\tilde{H}_i(y_i^{k-1})|y^0]=\mathbb{E}\left[\mathbb{E}[\tilde{H}_i(y_i^{k-1})|y_i^{k-1}]\big|y^0\right]=\mathbb{E}\left[\nabla F_i(y_i^{k-1})|y^0\right]=\nabla F_i(y_i^{k-1})$. Then, 

\resizebox{\linewidth}{!}{
  \begin{minipage}{\linewidth}
\begin{align*}
   \mathbb{E}||\Delta x^t||^2 & \leq \tilde{\eta}^2 \mathbb{E}\left[\frac{1}{KlM}\sum_{k\in [K], i\in C^t} \Big\| \nabla F_i(y_i^{k-1}) + \mathbb{E}[c^{t-1}| y^0] - \mathbb{E}[c_i^{t-1}|y^0]\Big\|^2 \right]
    + \frac{9\tilde{\eta}^2}{KlM}(\varsigma^2/s\r + \Sigma_g^2(\mathcal{C})) \qquad \text{(convexity of $||.||^2$)}\\
    & = \tilde{\eta}^2 \frac{1}{KM}\sum_{k \in [K], i\in [M]}\mathbb{E}\Big[ \Big\| \underbrace{\nabla F_i(y_i^{k-1}) + \mathbb{E}[c^{t-1}|y^0] - \mathbb{E}[c_i^{t-1}|y^0]}_{\substack{\nabla F_i(y_i^{k-1}) - \nabla F_i(x^{t-1}) \\ + \mathbb{E}[c^{t-1}|y^0] - \nabla F(x^{t-1}) \\ - \mathbb{E}[c_i^{t-1}|y^0]+ \nabla F_i(x^{t-1})
    \\+ \nabla F(x^{t-1}) }}\Big\|^2 \Big] + \frac{9\tilde{\eta}^2}{KlM}(\varsigma^2/s\r + \Sigma_g^2(\mathcal{C})) \qquad \text{(definition of $C^t$)}\\
    & = \tilde{\eta}^2 \frac{1}{KM}\sum_{k \in [K], i\in [M]}\mathbb{E}\left[ \Big\| \mathbb{E}\left[ \nabla F_i(y_i^{k-1}) - \nabla F_i(x^{t-1}) + c^{t-1} - \nabla F(x^{t-1})- c_i^{t-1}+ \nabla F_i(x^{t-1})
    + \nabla F(x^{t-1}) \Big |y^0\right]\Big \|^2 \right]
    \\ &+ \frac{9\tilde{\eta}^2}{KlM}(\varsigma^2/s\r + \Sigma_g^2(\mathcal{C}))  \qquad\qquad\qquad\text{(all variables are measurable wrt $y^0$)}\\
    & \leq \footnotesize\tilde{\eta}^2 \frac{1}{KM}\sum_{k \in [K], i\in [M]}\mathbb{E}\left[  \mathbb{E}\left[ \Big\| \nabla F_i(y_i^{k-1}) - \nabla F_i(x^{t-1}) + c^{t-1} - \nabla F(x^{t-1})- c_i^{t-1}+ \nabla F_i(x^{t-1})
    + \nabla F(x^{t-1})\Big \|^2 \Big |y^0\right] \right] 
    \\ &+ \frac{9\tilde{\eta}^2}{KlM}(\varsigma^2/s\r + \Sigma_g^2(\mathcal{C})) \qquad\qquad\qquad \text{(Jensen inequality)}\\
    & \leq \frac{4 \tilde{\eta}^2}{KM}\sum_{k \in [K], i\in [M]}\mathbb{E}\left[\Big\|\nabla F_i(y_i^{k-1}) - \nabla F_i(x^{t-1})\Big\|^2\right] + \frac{8 \tilde{\eta}^2}{KM} \sum_{k \in [K], i\in [M]}\mathbb{E}\left[\Big\|\nabla F_i(\alpha_{i, k-1}^{t-1}) - \nabla F_i(x^{t-1})\Big\|^2\right] \\
    & + 4\tilde{\eta}^2 \mathbb{E}\Big\|\nabla F(x^{t-1})\Big\|^2 + \frac{9\tilde{\eta}^2}{KlM}(\varsigma^2/s\r + \Sigma_g^2(\mathcal{C})).
\end{align*}
  \end{minipage}
}

The last inequality is obtained by definition of $c$ and $c_i$ and by applying Jensen inequality. With Lemma~\ref{lemma_nesterov}, this leads to the result.
\end{proof}

\begin{lemma}[{Lag in the control variate}] 
\label{control_variable_lemma}
\leavevmode
$\forall \alpha \in [1/2, 1], \forall \tilde{\eta} \leq \frac{1}{24 \nu}l^{\alpha}$,
\begin{align*}
    \mathcal{F}_t \leq \left(1 -\frac{17}{36}l\right)\mathcal{F}_{t-1} + \frac{1}{24 \nu} l^{2 \alpha -1} \mathbb{E}\big(F(x^{t-1}) -F(x^*)\big) + \frac{97}{48}l^{2 \alpha -1}\mathcal{E}_t + \frac{l}{\nu^2}\frac{\varsigma^2/s\r + \Sigma_g^2(\mathcal{C})}{32 KlM}.
\end{align*}
\end{lemma}

\begin{proof}
We adapt the original proof made in the non-convex case \cite[Lemma 16]{fl_scaffold} and use Lemma~\ref{lemma_reg_h_tilde} and Lemma~\ref{lemma_nesterov}.  
\end{proof}

\begin{lemma}[{Bounding the user drift}]
\label{user_drift}
\leavevmode
$\forall \eta_g \geq 1, \forall \eta_l \leq 1/24\nu K \eta_g$,
\begin{align*}
    \frac{9}{2}\nu^2\tilde{\eta} \mathcal{E}_t \leq \frac{9}{2}\nu^3\tilde{\eta}^2 \mathcal{F}_{t-1} + \frac{9}{40}\frac{\tilde{\eta}\nu}{\eta_g^2} \mathbb{E}\big(F(x^{t-1}) -F(x^*)\big) + \frac{27}{40}\frac{\tilde{\eta}^2 \nu}{K \eta_g^2} \big(\varsigma^2/s\r + \Sigma_g^2(\mathcal{C})\big).
\end{align*}
\end{lemma}
\begin{proof}

We once again adapt the original proof made in the non-convex case \cite[Lemma 17]{fl_scaffold}, use Lemma~\ref{lemma_nesterov} and multiply on each side of the inequality by $\frac{9}{2}\nu^2\tilde{\eta}$.
\end{proof}

\begin{lemma}[{Progress made at each round}]
\leavevmode
\label{progress_round}
$\forall \eta_g \geq 1, \forall \eta_l \leq \min \big(\frac{1}{24 K \eta_g \nu}l^{2/3}, \frac{l}{54\mu K\eta_g}\big)$,
\begin{align*}
    \mathbb{E}||x^t - x^*||^2 + 27 \nu^2 \tilde{\eta}^2 \frac{1}{l}\mathcal{F}_t & \leq \left(1 - \frac{\mu \tilde{\eta}}{2}\right)\left[\mathbb{E}||x^{t-1} -x^*||^2 + 27 \nu^2 \tilde{\eta}^2 \frac{1}{l}\mathcal{F}_{t-1} \right]\\
    & - \frac{\tilde{\eta}}{2}\mathbb{E}\big(F(x^{t-1}) -F(x^*)\big) + \frac{10 \tilde{\eta}^2}{KlM}\left(1+\frac{lM}{\eta_g^2}\right)(\varsigma^2/s\r + \Sigma_g^2(\mathcal{C})).
\end{align*}
\end{lemma}

\begin{proof}

We recall that $\Delta x^t = -\displaystyle\frac{\tilde{\eta}}{KlM}\sum_{k\in[K], i \in C^t} \tilde{H}_i^k (y_i^{k-1}) + c^{t-1} - c_i^{t-1}$. Then,
\begin{align}
    \mathbb{E}[\Delta x^t|y^0]= \mathbb{E}[\Delta x^t|x^{t-1}]=-\tilde{\eta} \mathbb{E}[ c^{t-1}|y_0] =-\displaystyle\frac{\tilde{\eta}}{KM}\sum_{k\in[K], i \in [M]} \mathbb{E}[\nabla F_i (y_i^{k-1})|y_0]. \label{eq:auxi12}
\end{align}

We denote $\mathbb{E}_{t-1}[.]$ as the expectation conditioned on randomness generated (strictly) prior to round $t$, i.e. conditionally to $\sigma(x^\tau, \tau\le t-1)$. We first bound the quantity $\mathbb{E}_{t-1}||x^t - x^*||^2 =\mathbb{E}_{t-1}||x^{t-1} + \Delta x^t - x^*||^2$,

{\small
\begin{align*}
    \mathbb{E}_{t-1}||x^{t}- x^*||^2 & = \mathbb{E}_{t-1}||x^{t-1} - x^*||^2 + \mathbb{E}_{t-1}||\Delta x^t||^2 + 2\bigg[\bigg\langle \mathbb{E}_{t-1}[\Delta x^t|y_0], x^{t-1} - x^* \bigg\rangle \bigg]\\
    & = ||x^{t-1} - x^*||^2 + \mathbb{E}_{t-1}||\Delta x^t||^2 + 2\bigg[ \bigg\langle\underbrace{-\frac{ \tilde{\eta}}{KM}\sum_{k \in [K],i \in [M]}  \mathbb{E}[\nabla F_i (y_i^{k-1})|y_0]}_{\text{by \eqref{eq:auxi12}}}, x^{t-1} - x^* \bigg\rangle \bigg]\\
    & \leq \mathbb{E}_{t-1}||x^{t-1} - x^*||^2 + 4 \tilde{\eta}^2\nu^2 \mathcal{E}_t + 8 \nu^2 \tilde{\eta}^2 \mathcal{F}_{t-1} + 8 \nu \tilde{\eta}^2 \left(F(x^{t-1}) -F(x^*)\right)\\
    & + \frac{9\tilde{\eta}^2}{KlM}(\varsigma^2/s\r + \Sigma_g^2(\mathcal{C})) + \underbrace{\frac{2 \tilde{\eta}}{KM}\mathbb{E}_{t-1}\left[\sum_{k \in [K],i \in [M]}\langle \nabla F_i(y_i^{k-1}), x^* -x^{t-1} \rangle \right]}_{\mathcal{A}}, & \text{(Lemma~\ref{lemma_variance_server})}
\end{align*}
}%
where
\begin{align*}
    \mathbb{E}[\mathcal{A}] & \leq \frac{2 \tilde{\eta}}{KM}\mathbb{E}\left(\sum_{k\in [K], i\in [M]} F_i(x^*) -F_i(x^{t-1}) + \nu ||y_i^{k-1} -x^{t-1}||^2 - \frac{\mu}{4}||x^{t-1} - x^*||^2\right) \\
    & = - 2 \tilde{\eta}\left(\mathbb{E}(F(x^{t-1})) - F(x^*) + \frac{\mu}{4}\mathbb{E}||x^{t-1}-x^*||\right) + 2 \nu \tilde{\eta}\mathcal{E}_t,
\end{align*}
where the inequality comes from the convexity and $\nu$-smoothness property.

Hence, by taking the expectation:
\begin{align*}
    \mathbb{E}||x^t - x^*||^2 & \leq \mathbb{E}||x^{t-1} - x^*||^2 - 2 \tilde{\eta}\left(\mathbb{E}(F(x^{t-1})) - F(x^*) + \frac{\mu}{4}\mathbb{E}||x^{t-1}-x^*||^2\right) + 2 \nu \tilde{\eta}\mathcal{E}_t\\
    & + 4 \tilde{\eta}^2\nu^2 \mathcal{E}_t + 8 \nu^2 \tilde{\eta}^2 \mathcal{F}_{t-1} + 8 \nu \tilde{\eta}^2 \mathbb{E}\left(F(x^{t-1}) -F(x^*)\right)+ \frac{9\tilde{\eta}^2}{KlM}(\varsigma^2/s\r + \Sigma_g^2(\mathcal{C})).
\end{align*}

By combining all terms and multiplying by $\nu$ on each side of the inequality, it comes:
\begin{align}
    \label{first_inequ}
    \nu \mathbb{E}||x^t - x^*||^2 & \leq
    \left(1 - \frac{\mu \tilde{\eta}}{2}\right)\nu \mathbb{E}||x^{t-1}-x^*||^2 + (8 \nu^2\tilde{\eta}^2 - 2\tilde{\eta} \nu) \left(\mathbb{E}(F(x^{t-1})) - F(x^*)\right)\\
    & + \frac{9\tilde{\eta}^2\nu}{KlM}(\varsigma^2/s\r + \Sigma_g^2(\mathcal{C})) + (2 \nu^2\tilde{\eta} + 4 \nu^3\tilde{\eta}^2)\mathcal{E}_t + 8 \nu^3 \tilde{\eta}^2 \mathcal{F}_{t-1} \nonumber.
\end{align}

We now consider $\alpha \in [1/2, 1]$, $\eta_l \leq \frac{1}{24 K \nu \eta_g} l^\alpha$ and $\eta_g \geq 1$. We use the result of Lemma~\ref{control_variable_lemma} where each side is multiplied by $27\nu^3\tilde{\eta}^2\frac{1}{l}$ to obtain:

\begin{align}
\label{second_inequ}
    27\nu^3\tilde{\eta}^2\frac{1}{l}\mathcal{F}_t & \leq \big(1-\frac{\mu \tilde{\eta}}{2}\big)27\nu^3\tilde{\eta}^2\frac{1}{l}\mathcal{F}_{t-1} + 27 \big(\frac{\mu \tilde{\eta}}{2l}- \frac{17}{36}\big)\nu^3\tilde{\eta}^2\mathcal{F}_{t-1}\\
    & + \frac{9}{8}l^{2\alpha -2}\nu^2 \tilde{\eta}^2 \big(\mathbb{E}(F(x^{t-1})) - F(x^*)\big) + \frac{873}{16}l^{2\alpha -2}\nu^3\tilde{\eta}^2\mathcal{E}_t \nonumber\\
    & +\frac{27}{32}\nu \tilde{\eta}^2 \frac{\varsigma^2/s\r + \Sigma_g^2(\mathcal{C})}{KlM} \nonumber.
\end{align}

Since we have $\eta_l \leq \frac{1}{24 K \nu \eta_g}$, we recall the result from Lemma~\ref{user_drift}: 

\begin{align}
\label{third_inequ}
    \frac{9}{2}\nu^2\tilde{\eta} \mathcal{E}_t \leq \frac{9}{2}\nu^3\tilde{\eta}^2 \mathcal{F}_{t-1} + \frac{9}{40}\frac{\tilde{\eta}\nu}{\eta_g^2} \mathbb{E}\big(F(x^{t-1}) -F(x^*)\big) + \frac{27}{40}\frac{\tilde{\eta}^2 \nu}{K \eta_g^2} \big(\varsigma^2/s\r + \Sigma_g^2(\mathcal{C})\big).
\end{align}

By summing inequalities (\ref{first_inequ}), (\ref{second_inequ}), (\ref{third_inequ}), we obtain:
\begin{align}
    \nu \mathbb{E}||x^t - x^*||^2 + 27\nu^3\tilde{\eta}^2\frac{1}{l}\mathcal{F}_t & \leq \big(1-\frac{\mu \tilde{\eta}}{2}\big)\big(\nu \mathbb{E}||x - x^*||^2 + 27\nu^3\tilde{\eta}^2\frac{1}{l}\mathcal{F}_{t-1}\big) \nonumber\\
    & + \big( \frac{9}{8}l^{2\alpha -2}\nu^2 \tilde{\eta}^2 + \frac{9}{40}\frac{\tilde{\eta}\nu}{\eta_g^2} +  8 \nu^2\tilde{\eta}^2 - 2\tilde{\eta} \nu\big)\big(\mathbb{E}(F(x^{t-1})) - F(x^*)\big) \label{first_part}\\
    & + \big(\frac{315}{32} +\frac{27}{40}\frac{lM}{\eta_g^2}\big)\frac{\tilde{\eta}^2 \nu}{K lM} \big(\varsigma^2/s\r + \Sigma_g^2(\mathcal{C})\big) \label{second_part}\\
    & + \big(-\frac{5}{2}\nu \tilde{\eta} + 4 \nu^2\tilde{\eta}^2 + \frac{873}{16}l^{2\alpha -2}\nu^2\tilde{\eta}^2\big)\nu \mathcal{E}_t \label{third_part}\\
    & +\bigg(27\big(\frac{\mu\tilde{\eta}}{2l}-\frac{17}{36}\big)\nu^2\tilde{\eta}^2 + \frac{25}{2}\nu^2\tilde{\eta}^2\bigg)\nu\mathcal{F}_{t-1} \label{fourth_part}.
\end{align}

We now consider $\eta_l \leq l/54\mu K \eta_g$. Then $\tilde{\eta}\leq l/54\mu$ and we recall that $\nu \tilde{\eta} \leq 1/24$. We fix $\alpha=2/3$ (then $2 - 2\alpha=\alpha)$.\\

In this part, we aim at simplifying the terms on the right side of the last inequality.\\

\textbf{Simplifying (\ref{first_part}):}
\begin{align*}
    \frac{9}{8}l^{2\alpha -2}\nu^2 \tilde{\eta}^2 + \frac{9}{40}\frac{\tilde{\eta}\nu}{\eta_g^2} +  8 \nu^2\tilde{\eta}^2 - 2\tilde{\eta} \nu & \leq \big(\frac{9}{8\times 24} +\frac{9}{40}+ \frac{8}{24} -2\big)\nu\tilde{\eta}\\
    & = -\underbrace{\frac{1339}{960}}_{\sim 1.39}\nu\tilde{\eta} \leq -\frac{\nu\tilde{\eta}}{2}.
\end{align*}

\textbf{Simplifying (\ref{second_part}):}
\begin{align*}
    \frac{315}{32} +\frac{27}{40}\frac{lM}{\eta_g^2} \leq 10 \big(1 +\frac{lM}{\eta_g^2}\big).
\end{align*}

\textbf{Simplifying (\ref{third_part}):}\\

Since $l^{2\alpha -2}\nu\tilde{\eta}= \nu\tilde{\eta}\big(\frac{1}{l}\big)^{2/3}\leq 1/24$, 
\begin{align*}
    -\frac{5}{2}\nu \tilde{\eta} + 4 \nu^2\tilde{\eta}^2 + \frac{873}{16}l^{2\alpha -2}\nu^2\tilde{\eta}^2 & \leq \big(-\frac{5}{2} + \frac{4}{24} + \frac{873}{16}\frac{1}{24}\big)\nu\tilde{\eta}\\
    & = -\frac{23}{384}\nu\tilde{\eta} \leq 0.
\end{align*}

\textbf{Simplifying (\ref{fourth_part}):}\\

Since $\frac{\mu \tilde{\eta}}{2l}\leq 1/108$,
\begin{align*}
    27\big(\frac{\mu\tilde{\eta}}{2l}-\frac{17}{36}\big)\nu^2\tilde{\eta}^2 + \frac{25}{2}\nu^2\tilde{\eta}^2 \leq \big(27(\frac{1}{108}-\frac{17}{36}) + \frac{25}{2}\big)\nu^2\tilde{\eta}^2 =0.
\end{align*}

We then obtain the final result by dividing by $\nu$ on each side of the inequality.
\end{proof}

\begin{lemma}[{Convergence of \texttt{DP-SCAFFOLD-warm}  with convex loss functions}]\label{lemma_utility_with_contraction}
    \item If $f_i$ are \textbf{$\mu$-strongly convex} ($\mu>0$), $\eta_g\geq 1$,
    $\bar{\eta}_l = \min\big(\frac{l^{\frac{2}{3}}}{24\nu K\eta_g},\frac{l}{54\mu K \eta_g}\big)$,
    and $T \geq \max (\frac{54}{l},\frac{24 \nu}{\mu l^{\frac{2}{3}}} )$, then there exist weights $\{w_t\}$ and local step-sizes $\eta_l\leq \bar{\eta}_l$ such that $\overline{x}^T=\sum_{t \in [T]} w_t x^t$ and

    \begin{center}
        $\mathbb{E}[F(\overline{x}^T)] -F(x^*) \leq \mathcal{O}\bigg(\displaystyle\frac{\varsigma^2/s\r + \Sigma_g^2(\mathcal{C})}{\mu T KlM}\bigg(1 + \frac{lM}{\eta_g^2}\bigg) + \mu D_0^2\exp\bigg(-\min\bigg(\frac{l}{108}, \frac{\mu l^{\frac{2}{3}}}{48 \nu}\bigg)T\bigg)\bigg)$,
    \end{center}
    
    \item If $f_i$ are \textbf{convex}, $\eta_g\geq 1$, $\bar{\eta}_l=\min\big( \frac{l^{\frac{2}{3}}}{24\nu K\eta_g},\frac{D_0 \sqrt{KlM}}{K\eta_g\sqrt{T(1+lM/\eta_g^2)(\varsigma^2/sR + \Sigma_g(\mathcal{C})^2 )}} \big)$ and $T \geq 1$, then there exist weights $\{w_t\}$ and local step-sizes $\eta_l\leq \bar{\eta}_l$ such that $\overline{x}^T=\sum_{t \in [T]} w_t x^t$ and
    \begin{center}
        $\mathbb{E}[F(\overline{x}^T)] -F(x^*) \leq \mathcal{O}\bigg(\displaystyle D_0\frac{\varsigma/\sqrt{s\r} + \Sigma_g(\mathcal{C})}{\sqrt{TKlM}}\sqrt{1 + \frac{lM}{\eta_g^2}} + \frac{\nu}{T}D_0^2\bigg(\frac{1}{l}\bigg)^{\frac{2}{3}} \bigg)$,
    \end{center}

where $D_0:=||x^0 - x^*||$.
\end{lemma}
\begin{proof} We denote $D_0:=||x^0 - x^*||$. 

1. Let us first prove the result of Lemma~\ref{lemma_utility_with_contraction} for the strongly convex case. We start by unrolling the contraction inequality obtained in Lemma~\ref{progress_round}. Let $\eta_g\geq 1$ and $\bar{\eta}_l = \min\big(\frac{l^{\frac{2}{3}}}{24\nu K\eta_g},\frac{l}{54\mu K \eta_g}\big)$. We define $\tilde{\eta}_{\max}=K\eta_g\bar{\eta}_l$. In particular, we have $\tilde{\eta}_{\max}\in (0, 2/\mu]$. For any $\tilde{\eta}\leq \tilde{\eta}_{\max}$ and any $t\geq 1$, we have by Lemma~\ref{progress_round}
\begin{align*}
    \mathbb{E}\big(F(x^{t-1})\big) -F(x^*)  \leq \frac{1}{\tilde{\eta}}\left(1 - \frac{\mu \tilde{\eta}}{2}\right) A_{t-1} - \frac{1}{\tilde{\eta}} A_t + \tilde{\eta}C,
\end{align*}
where $A_t=2\mathbb{E}||x^t - x^*||^2 + 54 \nu^2 \tilde{\eta}^2 \frac{1}{l}\mathcal{F}_t$ and $C=\frac{20}{KlM}\left(1+\frac{lM}{\eta_g^2}\right)(\varsigma^2/s\r + \Sigma_g^2(\mathcal{C}))$. Remark that $C\geq 0$, $A_t\geq 0$ and 
$A_0=2D_0^2$, since $\mathcal{F}_0=0$. Let $T\geq \frac{1}{\tilde{\eta}_{\max} \mu}\geq\max (\frac{54}{l},\frac{21 \nu}{\mu l^{\frac{2}{3}}} )$. We invoke a technical contraction result used in the original proof \cite[Lemma 1]{fl_scaffold} to obtain that there exist weights $\{w_t\}$ and local step-sizes $\eta_l\leq \bar{\eta}_l$ such that
\begin{align*}
    \sum_{t=1}^{T+1} w_{t-1} \left\{\frac{1}{\tilde{\eta}}(1 - \frac{\mu \tilde{\eta}}{2}) A_{t-1} - \frac{1}{\tilde{\eta}} A_t + \tilde{\eta}C \right\}& =\mathcal{O}\left(\frac{2C}{\mu T} + \frac{\mu}{2} A_0 \exp( -\frac{\mu}{2} \tilde{\eta}_{\max} T) \right)\\
    & = \mathcal{O}\left(\frac{\varsigma^2/s\r + \Sigma_g^2(\mathcal{C})}{\mu T KlM}\bigg(1 + \frac{lM}{\eta_g^2}\bigg) + \mu D_0^2\exp\bigg(-\min\bigg(\frac{l}{108}, \frac{\mu l^{\frac{2}{3}}}{48 \nu}\bigg)T\bigg)\right), 
\end{align*}

recalling that $\tilde{\eta}=K\eta_g \eta_l$. We now define $\overline{x}^T=\sum_{t \in [T]} w_t x^t$, and directly obtain the result of Lemma~\ref{lemma_utility_with_contraction} using the convexity of $F$ and the last bound,
\begin{align*}
    \mathbb{E}[F(\overline{x}^T)] -F(x^*) & \leq \sum_{t=1}^{T+1} w_{t-1} \{\mathbb{E}\big(F(x^{t-1})
    -F(x^*)\big)\}\leq \sum_{t=1}^{T+1} w_{t-1} \left\{\frac{1}{\tilde{\eta}}(1 - \frac{\mu \tilde{\eta}}{2}) A_{t-1} - \frac{1}{\tilde{\eta}} A_t + \tilde{\eta}C \right\} .
\end{align*}

2. Let us now prove the result of Lemma~\ref{lemma_utility_with_contraction} for the convex case. Let $\eta_g\geq 1$, $\bar{\eta}_l=\min\big( \frac{l^{\frac{2}{3}}}{24\nu K\eta_g},\frac{D_0 \sqrt{KlM}}{K\eta_g\sqrt{T(1+lM/\eta_g^2)(\varsigma^2/sR + \Sigma_g(\mathcal{C})^2 )}} \big)$ and $T \geq 1$. By averaging over $t$ in Lemma~\ref{progress_round} with $\mu=0$, we have for any $\eta_l \leq \bar{\eta}_l$ (recalling that $\tilde{\eta}=K \eta_g \eta_l$)
\begin{align*}
    \frac{1}{T}\sum_{t=1}^T \mathbb{E}\big(F(x^{t-1})\big) -F(x^*) & \leq \frac{2}{\tilde{\eta}T}\{D_0^2 + 27 \nu^2 \tilde{\eta}^2 \frac{1}{l}\mathcal{F}_0\} + \frac{20 \tilde{\eta}}{KlM}\left(1+\frac{lM}{\eta_g^2}\right)(\varsigma^2/s\r + \Sigma_g^2(\mathcal{C}))\\
    & \leq \frac{48\nu}{l^{\frac{2}{3}}T}D_0^2 + 22 D_0\frac{\varsigma/\sqrt{s\r} + \Sigma_g(\mathcal{C})}{\sqrt{TKlM}}\sqrt{1 + \frac{lM}{\eta_g^2}} ,
\end{align*}

using that $\mathcal{F}_0=0$ and $\varsigma^2/s\r + \Sigma_g^2(\mathcal{C})\leq (\varsigma/\sqrt{s\r} + \Sigma_g(\mathcal{C}))^2$. The result of Lemma~\ref{lemma_utility_with_contraction} is thus straightforward by using the convexity of $F$ and considering $w_t=1/T$.
\end{proof}

\paragraph{Conclusion.} We obtain the result of Theorem~\ref{theorem_utility_scaffold} for the convex case with Lemma ~\ref{lemma_utility_with_contraction} and $\eta_g:=\sqrt{lM} \geq 1$.

\subsection{Proof of Theorem~\ref{theorem_utility_scaffold} (Non-Convex case)}
\label{subsec_utility_dp_scaffold_non_convex}

To state this result, we adapt the original proof in the case with a larger variance for DP-noised stochastic gradients (see Lemma~\ref{lemma_reg_h_tilde}), which gives the following result.

\begin{lemma}[{Convergence of \texttt{DP-SCAFFOLD-warm} with non-convex loss functions}] \label{lemma_utility_with_contraction_non_convex} If $f_i$ are \textbf{non-convex}, $\eta_g\geq 1$, $\bar{\eta}_l= \frac{l^{\frac{2}{3}}}{24\nu K\eta_g}$ and $T \geq 1$, then there exist weights $\{w_t\}$ and local step-sizes $\eta_l \leq \bar{\eta}_l$ such that $\overline{x}^T=\sum_{t \in [T]} w_t x^t$ and
    \begin{center}
        $\mathbb{E}||\nabla F(\overline{x}^T)||^2 \leq \mathcal{O}\bigg(\displaystyle\sqrt{F_0}\frac{\varsigma/\sqrt{s\r} + \Sigma_g(\mathcal{C}) }{\sqrt{TKlM}}\sqrt{1 + \frac{lM}{\eta_g^2}} +\frac{\nu}{T}F_0 \bigg(\frac{1}{l}\bigg)^{\frac{2}{3}}\bigg)$,
    \end{center}
where $F_0:=F(x^0) - F(x^*)$.
\end{lemma}

\paragraph{Conclusion.} We obtain the result of Theorem~\ref{theorem_utility_scaffold} with Lemma~\ref{lemma_utility_with_contraction_non_convex} and $\eta_g:=\sqrt{lM} \geq 1$.

\subsection{Theorem of Convergence for \texttt{DP-FedAvg}}
\label{subsec_utility_dp_fedavg}

\begin{theorem}[Utility rates of \texttt{DP-FedAvg}$(T,K,l,s,\sigma_g,\mathcal{C})$, $\sigma_g$ chosen arbitrarily]\label{theorem_utility_fedavg}Let $\sigma_g, \mathcal{C}>0$, $x^0 \in \mathbb{R}^d$. Suppose we run \texttt{DP-FedAvg}$(T,K,l,s,\sigma_g,\mathcal{C})$ (see Algorithm~\ref{algo_dp_fedavg}). Under Assumptions~\ref{ass:smooth_f_i} and \ref{ass:stoch_grad}, we  consider the sequence of iterates $(x^t)_{t\geq0}$ of the algorithm, starting from $x^0$. 

\begin{enumerate}
    \item If $F_i$ are \textbf{$\mu$-strongly convex} ($\mu>0$), $\eta_g=\sqrt{lM}$, $\bar{\eta}_l= \frac{1}{8(1+B^2)\nu K\eta_g}$ and $T\geq \frac{8(1+b^2)\nu}{\mu}$, then  there exist weights $\{w_t\}_{t \in [T]}$ and local step-sizes $\eta_l \leq \bar{\eta}_l$ such that the averaged output of \texttt{DP-FedAvg}$(T,K,l,s,\sigma_g,\mathcal{C})$, defined by $\overline{x}^T=\sum_{t \in [T]} w_t x^t$, has expected excess of loss such that:
    \begin{center}
        $\mathbb{E}[F(\overline{x}^T)] -F(x^*) \leq \mathcal{O}\bigg(\displaystyle \frac{\varsigma^2/s\r + \Sigma_g^2(\mathcal{C})}{\mu TKlM} + (1-l) \frac{G^2}{\mu T lM} + \frac{\nu G^2}{\mu^2 T^2} + \mu D_0^2 \exp\bigg(-\frac{\mu}{16(1+B^2)\nu}T\bigg)\bigg)$,
    \end{center}
    
    \item If $F_i$ are \textbf{convex}, $\eta_g=\sqrt{lM}$, $\bar{\eta}_l= \frac{1}{8(1+B^2)\nu K\eta_g}$ and $T \geq 1$,  then there exist weights $\{w_t\}_{t \in [T]}$ and local step-sizes $\eta_l \leq \bar{\eta}_l$ such that the averaged output of \texttt{DP-FedAvg}$(T,K,l,s,\sigma_g,\mathcal{C})$, defined by $\overline{x}^T=\sum_{t \in [T]} w_t x^t$, has expected excess of loss such that:
    \begin{center}
        $\mathbb{E}[F(\overline{x}^T)] -F(x^*) \leq \mathcal{O}\bigg(\displaystyle D_0\frac{\varsigma/\sqrt{s\r} + \Sigma_g(\mathcal{C})}{\sqrt{TKlM}} + \frac{GD_0\sqrt{1-l}}{\sqrt{T lM}} + \frac{D_0^{4/3}\nu^{1/3}G^{2/3}}{T^{2/3}} + \frac{B^2\nu D_0^2}{T} \bigg)$,
    \end{center}
    
    \item If $F_i$ are \textbf{non-convex}, $\eta_g=\sqrt{lM}$, $\bar{\eta}_l= \frac{1}{8(1+B^2)\nu K\eta_g}$ and $T \geq 1$,  then there exist weights $\{w_t\}_{t \in [T]}$ and local step-sizes $\eta_l \leq \bar{\eta}_l$ such that the randomized  output of \texttt{DP-SCAFFOLD-warm}$(T,K,l,s,\sigma_g,\mathcal{C})$, defined by $\{\overline{x}^T= x^t$ with probability $w_t$ for all $t\}$, has expected squared gradient of the loss such that:
    \begin{center}
        $\mathbb{E}||\nabla F(\overline{x}^T)||^2 \leq \mathcal{O}\bigg(\displaystyle \nu\sqrt{F}\frac{ \varsigma/\sqrt{s\r} + \Sigma_g(\mathcal{C})}{\sqrt{TKlM}} + \frac{\nu G \sqrt{F(1-l)}}{\sqrt{TlM}} + \frac{F^{2/3}\nu^{1/3}G^{2/3}}{T^{2/3}} + \frac{B^2 \nu F}{T}\bigg)$,
    \end{center}
\end{enumerate}
where $D_0:=||x^0 - x^*||$ and $F:=F(x^0) - F(x^*)$.
\end{theorem}

\begin{proof}
To state the result of Theorem~\ref{theorem_utility_fedavg}, we combine the original result \cite[Theorem V]{fl_scaffold} provided for any type of loss functions with the result of Lemma~\ref{lemma_reg_h_tilde}.
\end{proof}

\section{ADDITIONAL EXPERIMENTS DETAILS AND RESULTS}
\label{app:exp_all}

In this section, we give additional details on our experimental setup (Section~\ref{app:setting_algo}) and synthetic data generation process (Section~\ref{app:data_gen}), and provide additional results (Section~\ref{appendix_experiment}). All results are summarized in Table~\ref{summary_experiments}.

\begin{table*}[h!]
\caption{Summary of the experiments of the paper.}
\label{summary_experiments}
\begin{center}
\resizebox{\linewidth}{!}{
\begin{tabular}{llll}
\toprule
\textbf{Dataset} & \textbf{Model} & \textbf{Reference} & \textbf{Take-away message}\\
\hline
Synthetic & LogReg & Figs~\ref{logistic_heterogene_K50_K100},\ref{logistic_10_20_epochs_accuracy}-\ref{logistic_10_20_epochs_data}& Superiority of \texttt{DP-SCAFFOLD} over \texttt{DP-FedAvg} \\
FEMNIST & LogReg & Figs~\ref{femnist_and_mnist_heterogene_K50} (first row),\ref{femnist_10_20_epochs_accuracy}-\ref{femnist_10_20_epochs_diss} & Superiority of \texttt{DP-SCAFFOLD} over \texttt{DP-FedAvg} \\
MNIST & DNN & Fig~\ref{femnist_and_mnist_heterogene_K50} (second row) & Superiority of \texttt{DP-SCAFFOLD} over \texttt{DP-FedAvg}\\
\hline
Synthetic & LogReg & Tables~\ref{table_experiments_varying_k_sigma_gaussian_5_5},\ref{table_experiments_varying_k_sigma_gaussian_0_0} & Tradeoffs between $K$, $T$  and $\sigma_g$ under fixed $\epsilon$ \\
\hline
Synthetic, FEMNIST & LogReg & Fig~\ref{trade_off_user_ratio} (first, second rows) & Tradeoffs between $T$ and $l$ under fixed $\epsilon$\\
MNIST & DNN & Fig~\ref{trade_off_user_ratio} (third row) & Tradeoffs between $T$ and $l$ under fixed $\epsilon$\\
\hline
Synthetic, FEMNIST & LogReg & Fig~\ref{trade_off_sample_ratio} (first, second rows) & Tradeoffs between $T$ and $s$ under fixed $\epsilon$\\
MNIST & DNN & Fig~\ref{trade_off_sample_ratio} (third row) & Tradeoffs between $T$ and $s$ under fixed $\epsilon$\\
\bottomrule
\end{tabular}
}
\end{center}
\end{table*}

\subsection{Algorithms Setup}
\label{app:setting_algo}

\paragraph{Hyperparameter tuning.} 

We tuned the step-size hyperparameter $\eta_0$ for each dataset, each algorithm and each version (with or without DP) over a grid of 10 values with the lowest level of heterogeneity (5-fold cross validation conducted on the training set). We then kept the same $\eta_0$ for experiments with higher heterogeneity.

\paragraph{Clipping heuristic.} 

Setting a good clipping threshold $\mathcal{C}$ while preserving accuracy can be difficult  \citep{dp_fedavg_user_level}. Indeed, if $\mathcal{C}$ is too small, the clipped gradients may become biased, thereby affecting the convergence rate. On the other hand, if $\mathcal{C}$ is too large, we have to add more noise to stochastic gradients to ensure differential privacy (since the variance of the Gaussian noise is proportional to $\mathcal{C}^2$). In practice, we follow the strategy proposed by \citet{dp_deep_l}, which consists in setting $\mathcal{C}$ as the median of the norms of the unclipped gradients over each stage of local training. Throughout the iterations, $\mathcal{C}$ will then decrease. However, we are aware that locally setting $\mathcal{C}$ may leak information to the server about the magnitude of stochastic gradients. We here consider this leak as minor and neglect its impact on privacy guarantees. Adaptive clipping \citep{dp_adaptive_clipping} could be used to mitigate these concerns.

\paragraph{Deep neural network.} To prove the advantage of \texttt{DP-SCAFFOLD} with non-convex objectives, we perform experiments on MNIST data with a deep neural network. Its architecture is inspired by the network used by \cite{dp_deep_l} for \texttt{DP-SGD}. We use a feedforward neural network with ReLU units and softmax of 10
classes (corresponding to the 10 digits of MNIST) with cross-entropy loss. Our network combines a 60-dimensional Principal Component Analysis (PCA) projection layer and a hidden layer with 200 hidden units. Since the error bound for DP-FL algorithms grows linearly with the dimension of the parameters for non-convex objectives (see Theorems~\ref{theorem_utility_scaffold},\ref{theorem_utility_fedavg}), the PCA layer is actually necessary to prevent the curse of dimensionality due to the addition of noise for privacy. Note that neural networks with more layers would also suffer from the curse of dimensionality in the DP-FL context. Using a batch size of 500, we can reach a test accuracy higher than 98\% with this architecture in 100 epochs under the centralized setting. This result is consistent with what can
be achieved with a vanilla neural network \citep{lecun1998gradient}. In our framework, the PCA procedure is applied as preprocessing to all the samples \textit{without differential privacy}. To avoid privacy leakage at this step, it would need to include a private mechanism, whose privacy loss should be added to that of the training phase \citep[see the discussion in][Section 4]{dp_deep_l}.

\subsection{Synthetic Data Generation}
\label{app:data_gen}


Each ground-truth model for user $i$ consists in weights $W_i \in \mathbb{R}^{d'\times 10}$ and bias $b_i \in \mathbb{R}^{10}$, which are sampled from the following distributions: $ W_i | u_i \sim \mathcal{N}_{d' \times 10}(u_i,\text{Id})$ and $b_i |u_i' \sim \mathcal{N}_{10}(u_i', \text{Id})$ where $u_i \sim \mathcal{N}_{d' \times 10}(0, \alpha \text{Id})$ and $u_i' \sim \mathcal{N}_{10}(0, \alpha \text{Id})$. The data matrix $X_i$ of user $i$ is sampled according to $X_i |v_i \sim \mathcal{N}_{d'}(v_i, \Sigma)$ where $\Sigma$ is the covariance matrix defined by its diagonal $\Sigma_{j,j}=j^{-1.2}$ and $v_i |B_i \sim \mathcal{N}_{d'}(B_i,\text{Id})$ where $B_i \sim \mathcal{N}_{d'}(0,\nu\text{Id})$.
The labels are obtained by independently changing the labels given by the ground truth model with probability $0.05$.


\subsection{Additional Experimental Results}
\label{appendix_experiment}

We provide below more results on the experiments described in Section~\ref{sec:experiments}, including additional metrics and more extensive choices of heterogeneity levels. We also present additional experiments with higher privacy, including a study on the effect of sampling parameters $l$ and $s$ (and the trade-off with $T$) on privacy and convergence.
\vspace{-.5em}
\paragraph{Metrics.} 

To measure the convergence and performance of the algorithms at any communication round $t\in [T]$, we consider the following metrics:
\begin{itemize}[topsep=0pt,itemsep=1pt,leftmargin=*,noitemsep,wide]
    \item \textit{Accuracy(t)}: the average test accuracy of the model over all users,
    \item \textit{Train Loss(t)}$=\log_{10} (F(x^t)-F(x^*))$: the log-gap between the objective function evaluated at parameter $x^t$ and its minimum,
    \item \textit{Train Gradient Dissimilarity(t)}$=\frac{1}{M}\sum_{i=1}^M ||\nabla F_i(x^t)||^2 - ||\nabla F(x^t)||^2$, and similarly the \textit{Train Gradient Log-Dissimilarity(t)}$=\log\big(\frac{1}{M}\sum_{i=1}^M ||\nabla F_i(x^t)||^2\big) - \log\big(||\nabla F(x^t)||^2\big)$, which measure how the local gradients differ from the global gradient (i.e., the average across users) when evaluated at $x^t$, and hence quantify the \textit{user-drift} over the rounds of communication. 
\end{itemize}
\vspace{-.5em}

\subsubsection{Results with other metrics and different heterogeneity levels}
We provide below some additional results which complement those provided in Section~\ref{sec:experiments}. 

\begin{itemize}
    \item \textbf{Synthetic data.} We plot in Fig.~\ref{logistic_10_20_epochs_accuracy} the evolution of the accuracy over the rounds, which is consistent with the evolution of the train loss in Fig.~\ref{logistic_heterogene_K50_K100}. While the variance of the accuracy for \texttt{DP-FedAvg} grows with the heterogeneity, the results of \texttt{DP-SCAFFOLD-warm} are not affected. We can observe an average difference of $10\%$ in the accuracy for these two algorithms over the various heterogeneity settings. We provide in Fig.~\ref{logistic_10_20_epochs_diss} the evolution of the gradient dissimilarity for the same settings as in Fig.~\ref{logistic_heterogene_K50_K100} and Fig.~\ref{logistic_10_20_epochs_accuracy}, which once again shows a better convergence of \texttt{DP-SCAFFOLD-warm} compared to \texttt{DP-FedAvg} for the same privacy level. We also provide the evolution of the train loss when varying a single heterogeneity parameter: either $\alpha$ (which controls \textit{model} heterogeneity across users) in Fig.~\ref{logistic_10_20_epochs_model} or $\beta$ (which controls \textit{data} heterogeneity across users) in Fig.~\ref{logistic_10_20_epochs_data}. In both of these settings, \texttt{DP-SCAFFOLD-warm} performs consistently better.
    \item \textbf{FEMNIST data.} In Fig.~\ref{femnist_10_20_epochs_accuracy}, we put in perspective the accuracy observed with $K=50$ (see Fig.~\ref{femnist_and_mnist_heterogene_K50}, first row) with the one observed with $K=100$. We also show the evolution of the gradient dissimilarity in Fig.~\ref{femnist_10_20_epochs_diss}. These results on real data again show the superior performance of \texttt{DP-SCAFFOLD-warm}, consistently with our observations on synthetic data. 
\end{itemize}

\begin{figure*}[h!]
    \centering
    \includegraphics[width=\linewidth]{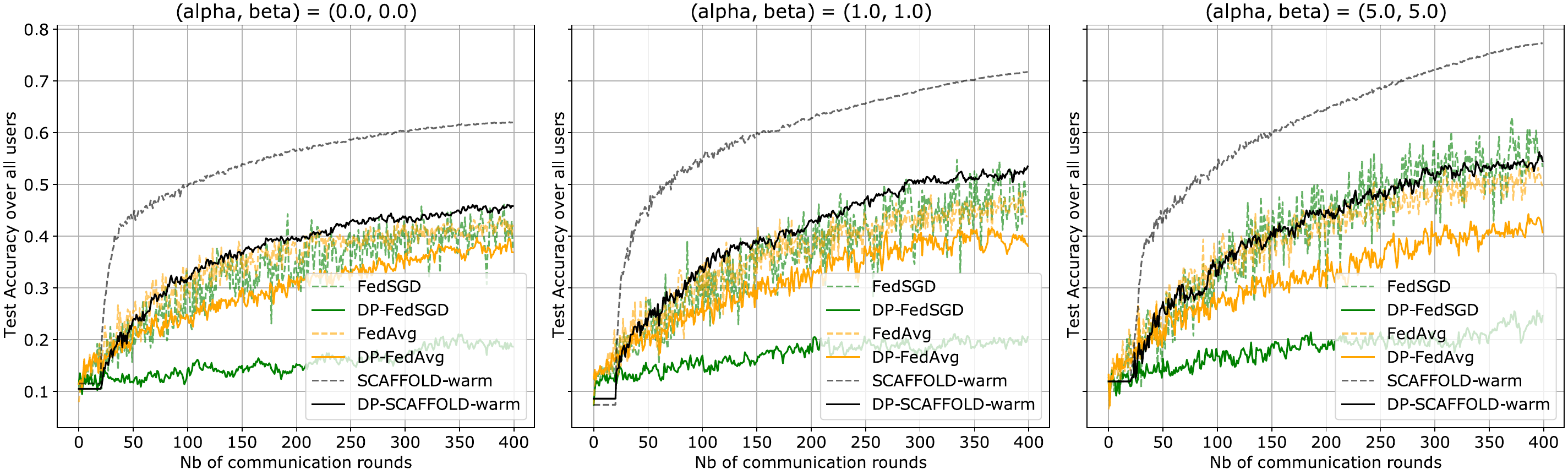} \\
\vspace{0.6em}
    \includegraphics[width=\linewidth]{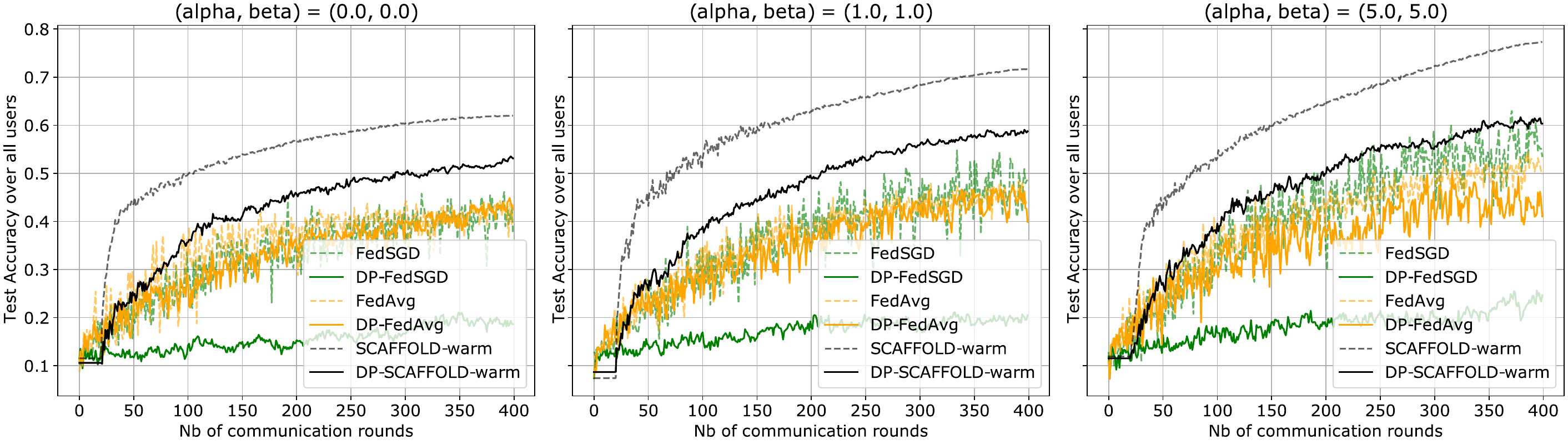}
    \caption{Test Accuracy On Synthetic Data ($\epsilon=13$). First Row: $K=50$; Second Row: $K=100$.}
    \label{logistic_10_20_epochs_accuracy}
\end{figure*}

\begin{figure*}[h!]
    \centering
    \includegraphics[width=\linewidth]{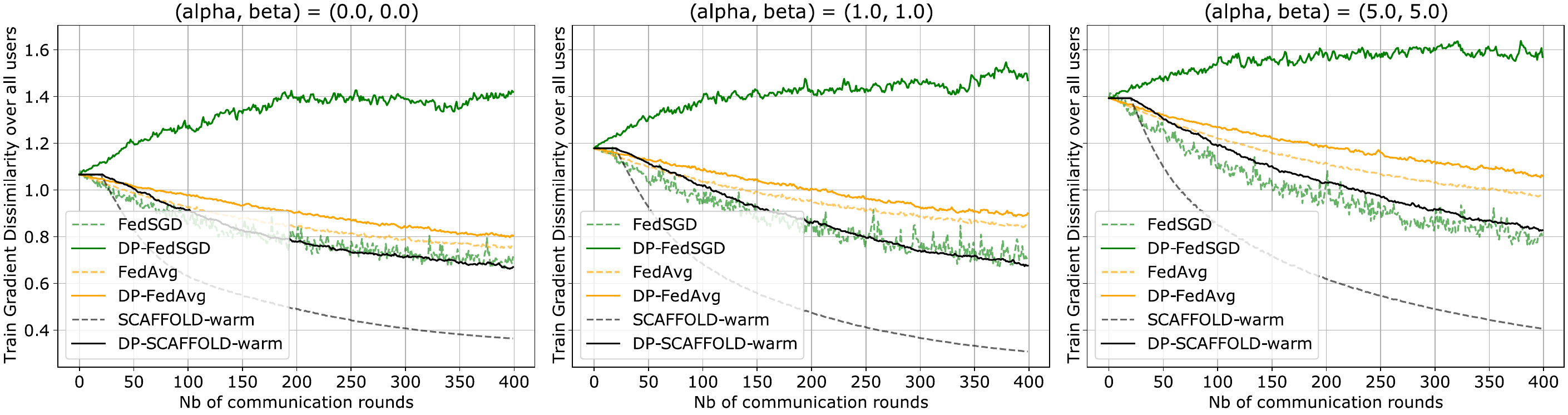} \\
\vspace{0.6em}
    \includegraphics[width=\linewidth]{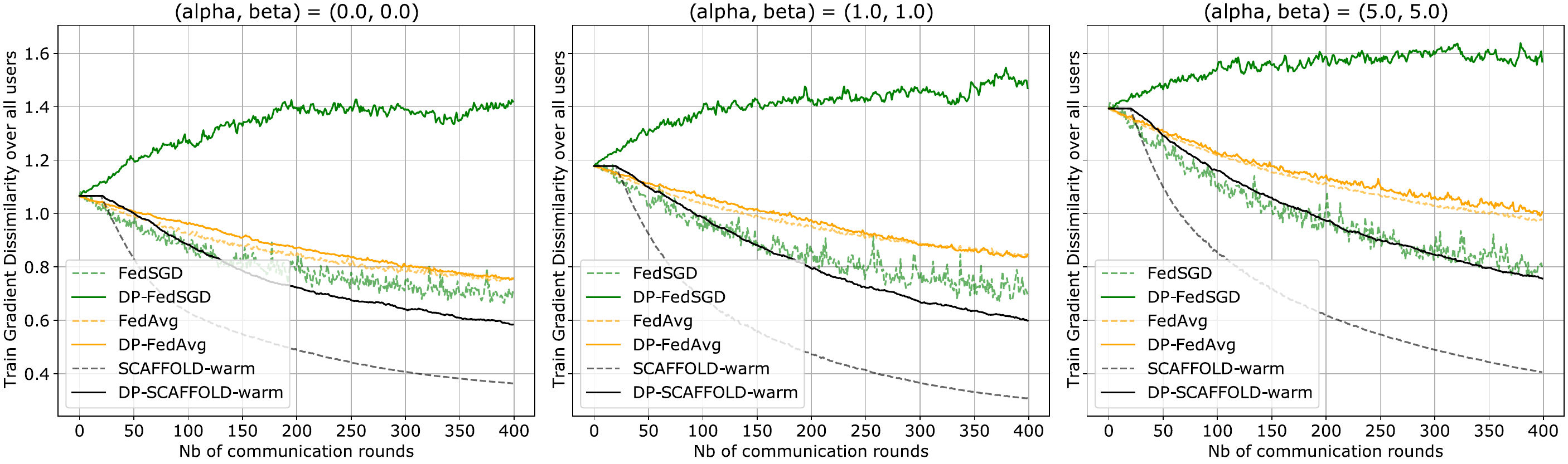}
    \caption{Train Gradient Dissimilarity On Synthetic Data ($\epsilon=13$). First Row: $K=50$; Second Row: $K=100$.}
    \vspace{-2em}
    \label{logistic_10_20_epochs_diss}
\end{figure*}

\begin{figure*}[h!]
    \centering
    \includegraphics[width=\linewidth]{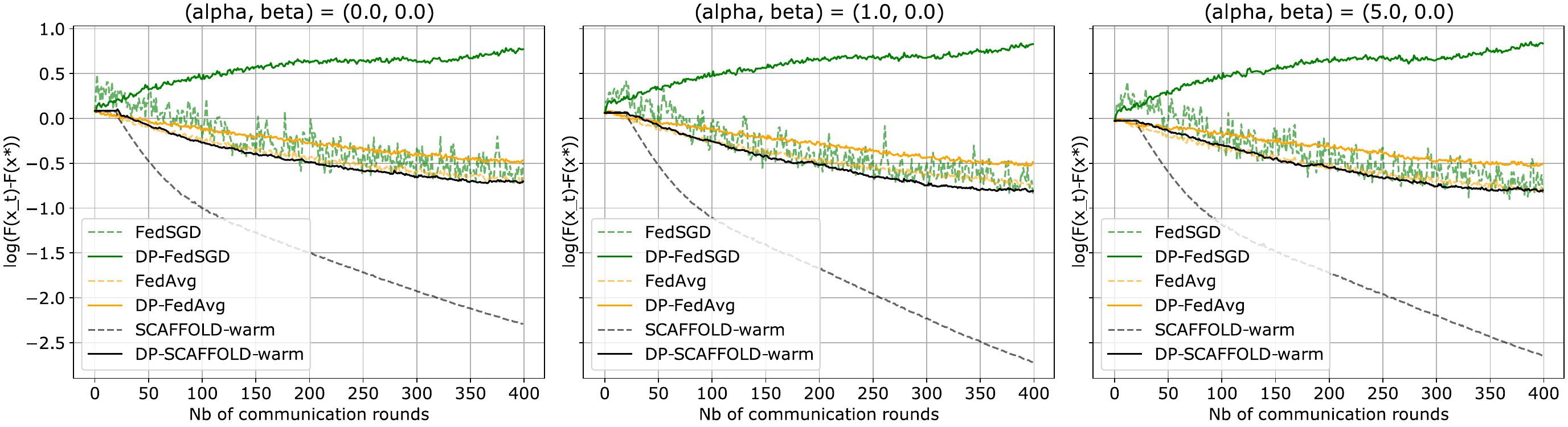} \\
\vspace{0.6em}
    \includegraphics[width=\linewidth]{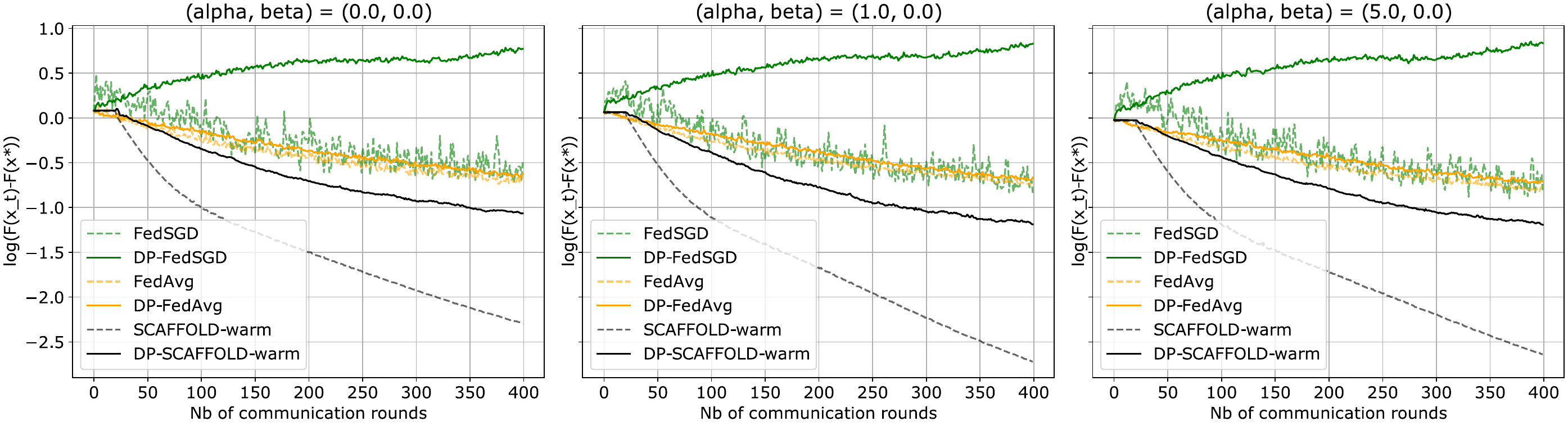}
    \caption{Model Heterogeneity (Varying $\alpha$): Train Loss On Synthetic Data ($\epsilon=13$). First Row: $K=50$; Second Row: $K=100$.}
    \label{logistic_10_20_epochs_model}
\end{figure*}

\begin{figure*}[h!]
    \centering
    \includegraphics[width=\linewidth]{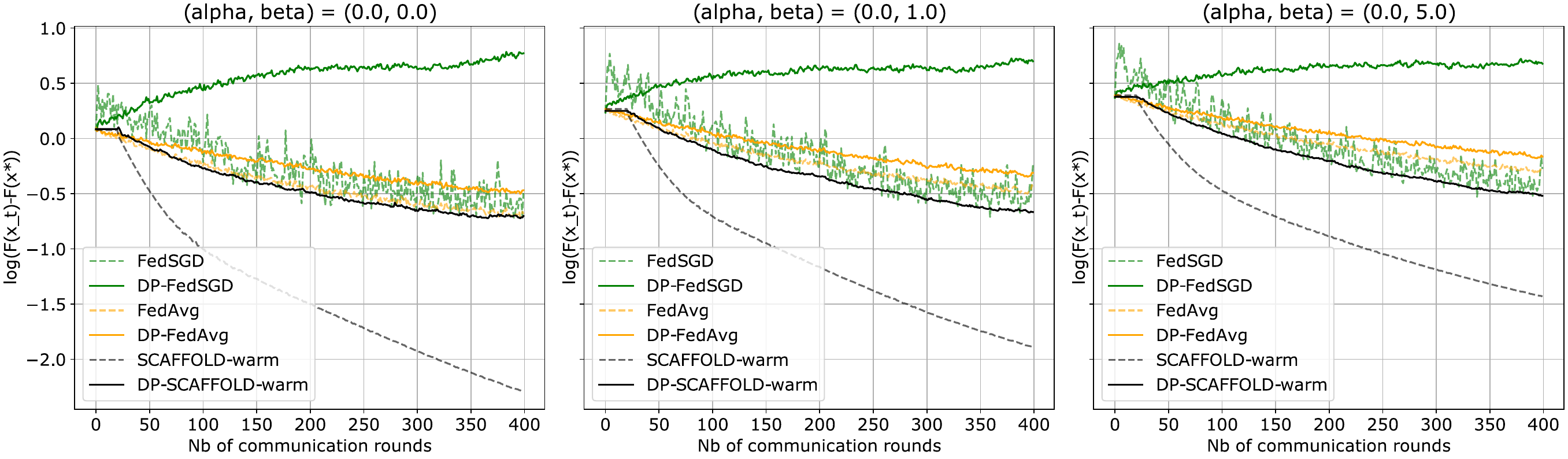} \\
\vspace{0.6em}
    \includegraphics[width=\linewidth]{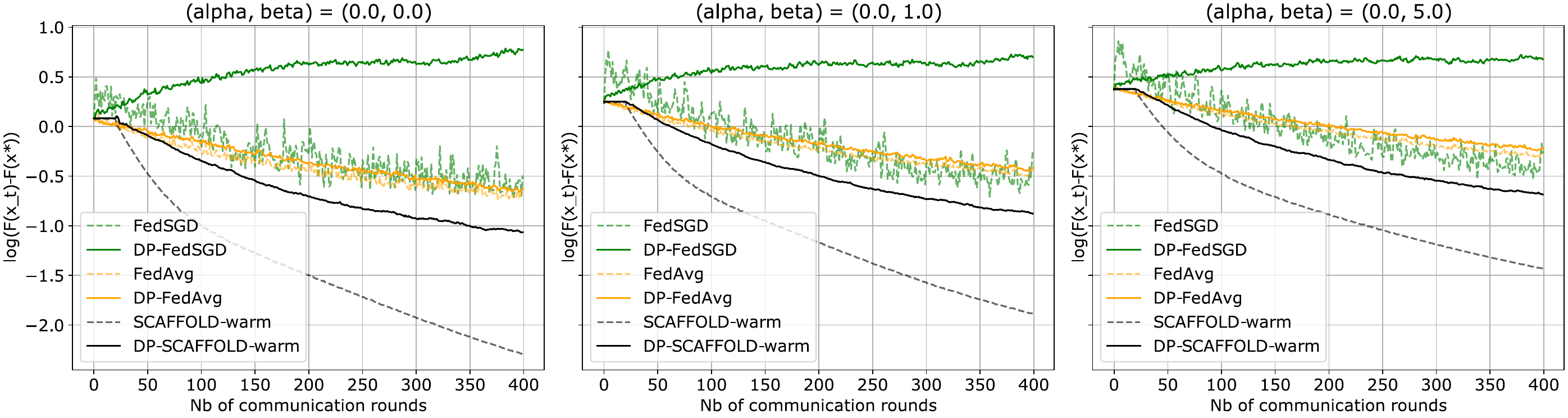}
    \caption{Data Heterogeneity Varying $\beta$): Train Loss On Synthetic Data ($\epsilon=13$). First Row: $K=50$; Second Row: $K=100$.}
    \vspace{-1em}
    \label{logistic_10_20_epochs_data}
\end{figure*}
\clearpage

\begin{figure*}[h!]
    \centering
    \includegraphics[width=\linewidth]{figures_main_aistats/femnist_accuracy_k_10.pdf} \\
\vspace{0.6em}
    \includegraphics[width=\linewidth]{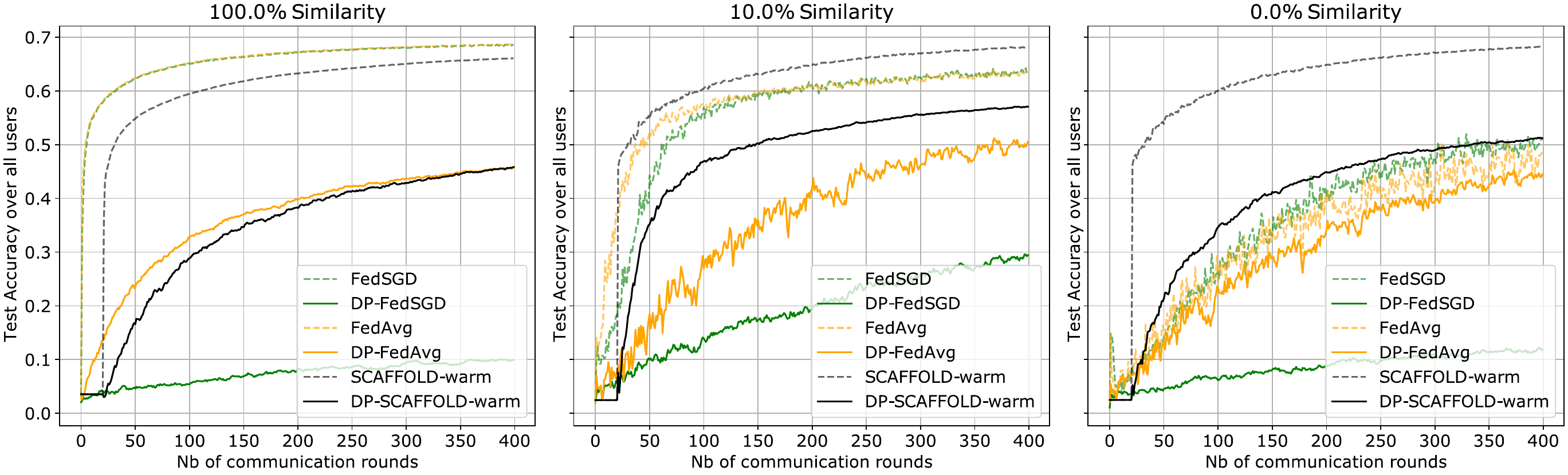}
    \vspace{-1.5em}
    \caption{Test Accuracy On FEMNIST Data (LogReg) with $\epsilon=11.5$. First Row: $K=50$; Second Row: $K=100$.}
    \label{femnist_10_20_epochs_accuracy}
\end{figure*}

\begin{figure*}[h!]
    \centering
    \includegraphics[width=\linewidth]{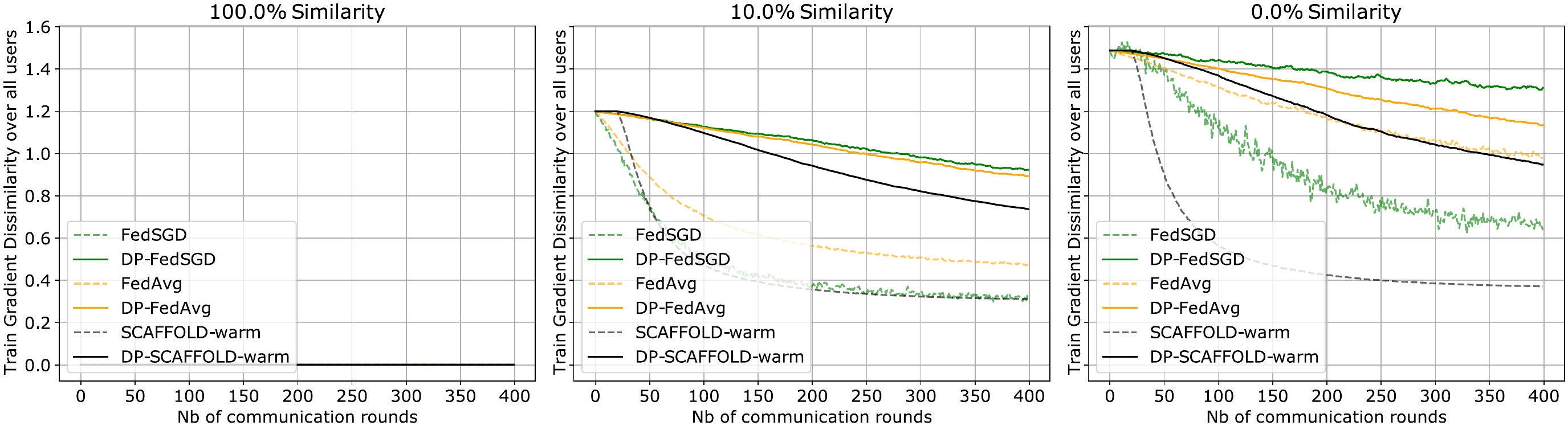} \\
\vspace{0.6em}
    \includegraphics[width=\linewidth]{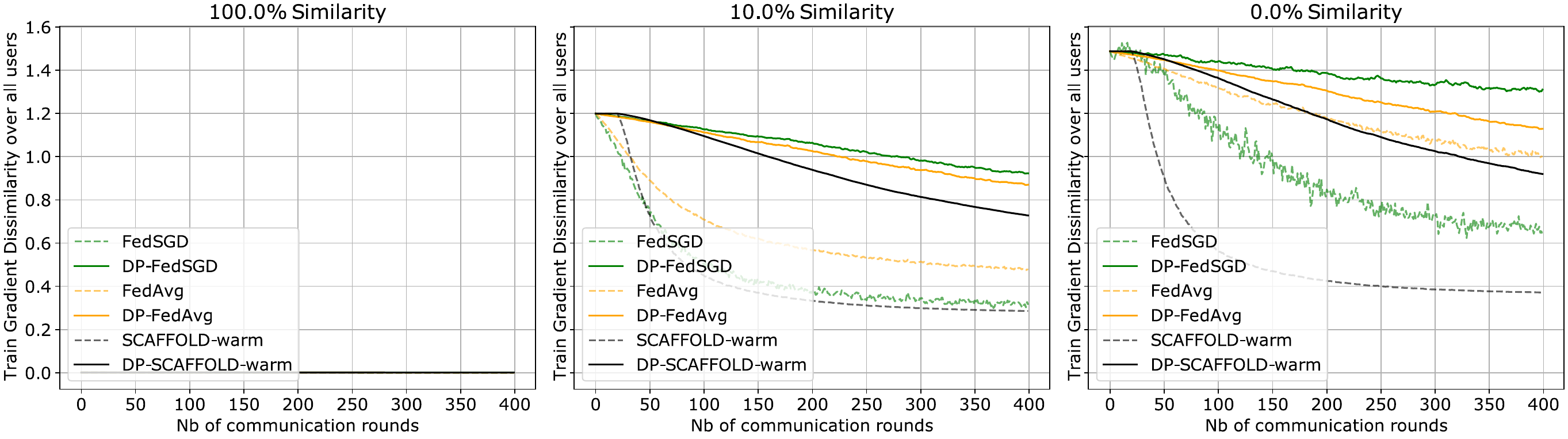}
    \caption{Train Gradient Dissimilarity On FEMNIST Data (LogReg) with $\epsilon=11.5$. First Row: $K=50$; Second Row: $K=100$.}
    \label{femnist_10_20_epochs_diss}
\end{figure*}

\newpage

\subsubsection{Additional results on the trade-offs between $K$, $T$ and $\sigma_g$}

In Section~\ref{sec:experiments} of the main text, we presented these trade-offs for \texttt{DP-SCAFFOLD} under a high level of heterogeneity (see Table~\ref{table_experiments_varying_k_sigma_gaussian_5_5}). We provide below the results of these trade-offs for \texttt{DP-SCAFFOLD} with a lower level of heterogeneity $(\alpha, \beta)=(0,0)$ (see Table~\ref{table_experiments_varying_k_sigma_gaussian_0_0}).
We consider again synthetic data with $\epsilon=3$ towards any third party.
To report the test accuracy which is obtained \textit{at the end of the iterations}, we proceed in two steps: (i) we compute the average test accuracy over the last 10\% of the iterations for each random run, (ii) we calculate the mean and the standard deviation over the 3 runs and report them in the tables. We highlight in bold for each row (i.e., for each value of $\sigma_g$)  in Tables~\ref{table_experiments_varying_k_sigma_gaussian_5_5},\ref{table_experiments_varying_k_sigma_gaussian_0_0} the best accuracy score obtained over all values of $K$.

\begin{table*}[t]
\caption{Test Accuracy ($\%$) For \texttt{DP-SCAFFOLD} On Synthetic Data $(\epsilon=3,l=0.05, s=0.2,(\alpha, \beta)=(0,0))$.}
\label{table_experiments_varying_k_sigma_gaussian_0_0}
\begin{center}
\resizebox{\linewidth}{!}{
\begin{tabular}{cclclclclcl}
\toprule
\textbf{$\sigma_g$} &\multicolumn{2}{c}{$K=1$} &\multicolumn{2}{c}{$K=5$}&\multicolumn{2}{c}{$K=10$} & \multicolumn{2}{c}{$K=20$} & \multicolumn{2}{c}{$K=40$} \\
\cmidrule(l){1-1}\cmidrule(l){2-3}\cmidrule(l){4-5}\cmidrule(l){6-7}\cmidrule(l){8-9}\cmidrule(l){10-11}
$10$ &  $31.29_{\pm 0.50}$ & $T=542$
     &$\mathbf{44.37}_{\pm 0.15}$ & $T=488$
     & $44.17_{\pm 0.56}$ & $T=428$
     & $41.71_{\pm 0.36}$ & $T=324$
     & $25.92_{\pm 0.90}$ & $T=72$\\
$20$ &  $28.31_{\pm 1.15}$ & $T=545$
     & $41.97_{\pm 0.77}$ & $T=502$
     & $\mathbf{43.27}_{\pm 0.90}$ & $T=451$
     & $41.13_{\pm 0.55}$ & $T=352$
     & $28.20_{\pm 2.23}$ & $T=83$\\
$40$ & $21.07_{\pm 0.41}$ & $T=546$
     & $33.01_{\pm 1.20}$ & $T=505$
     & $\mathbf{35.96}_{\pm 0.84}$ & $T=457$
     & $32.82_{\pm 1.11}$ & $T=360$
     & $23.49_{\pm 1.70}$ & $T=86$\\
$80$ & $17.09_{\pm 1.31}$ & $T=546$
     & $21.84_{\pm 0.96}$ & $T=506$
     & $\mathbf{25.92}_{\pm 0.59}$ & $T=458$
     & $24.02_{\pm 1.27}$ & $T=362$
     & $17.95_{\pm 1.14}$ & $T=87$\\
$160$ & $15.24_{\pm 1.63}$ & $T=546$
     & $15.37_{\pm 0.28}$ & $T=506$
     & $\mathbf{20.09}_{\pm 1.51}$ & $T=458$
     & $18.09_{\pm 1.86}$ & $T=362$
     & $15.38_{\pm 0.87}$ & $T=87$\\
\bottomrule
\end{tabular}
}
\end{center}
\end{table*}

We observe the same trends as the ones described for Table~\ref{table_experiments_varying_k_sigma_gaussian_5_5} in the main text. Indeed, our results clearly show a trade-off between $T$ and $K$ for \texttt{DP-SCAFFOLD} under a fixed privacy budget. If $K$ is set too low or too large, the performance of the algorithm is sub-optimal either because $T$ has to be chosen too low or because control variates are inefficient under few local updates. 

Moreover, we observe that setting $\sigma_g$ to a high value does not necessarily improve the gain in the number of communication rounds. In particular, for high values of $\sigma_g$, the calculation of the privacy bound does not allow to obtain a large increase in $T$ ($T$ does not change between $\sigma_g=40$ and $\sigma_g=160$ for any value of $K$), which thus leads to poor performance. This can be explained by the fact that the upper bound for the subsampled Gaussian mechanism given in Lemma~\ref{lemma_subsampling_rdp} does not converge to $0$ when $\sigma_g$ is very large. Indeed, given a subsampling ratio $q<1$, we can observe that this bound in the asymptotic regime becomes $\frac{1}{\alpha-1}\log(1+2\sum_{j=3}^\alpha q^j \binom{\alpha}{j})$, which is positive. Hence, by increasing the value of $\sigma_g$, we cannot hope to \textit{inconditionally} increase the number of compositions of the mechanism under a fixed privacy budget, if $q$ is taken too large. This artefact thus proves how small subsampling ratios have to be chosen to compute differential privacy in practice. In our experiments, we can clearly notice that this asymptotic regime is reached as soon as $\sigma_g$ is greater than $80$. One would have to consider lower subsampling ratios (for instance $l=10^{-4}, s=10^{-4}$), to obtain different values for $T$ when $\sigma_g=80$ and $\sigma_g=160$. We investigate such trade-offs in the next section.

\clearpage

\subsubsection{Experiments under higher privacy regime and role of sampling parameters}

In Section~\ref{sec:experiments} of the main text, given private parameter $\sigma_g$, we conducted experiments for (i) convex objective on FEMNIST ($\sigma_g=30$) and synthetic data ($\sigma_g=60$) and (ii) non-convex objective on MNIST data ($\sigma_g=30$). Considering the sampling parameters we used ($l=0.2, s=0.2$), this setting allows to reach, towards any third party, $(11.4, \delta)$-DP for FEMNIST data and $(13, \delta)$-DP for synthetic data. In the case of MNIST data, we obtain $(7.2, \delta)$-DP towards any third party for $(K,T)=(50, 100)$. Although these experiments only allow ``low privacy'', stronger DP guarantees can be obtained by simply \textit{decreasing} the subsampling ratios (thus amplifying the privacy). For instance, setting $l'=l/4=0.05$ in case of synthetic data would provide $(4.2, \delta)$-DP for $(K,T)=(50, 400)$. We present below some results on numerical trade-offs under a higher privacy regime.

As observed in the experiments of the main text and consistently with previous work \citep{fl_scaffold}, we observe the superiority of \texttt{SCAFFOLD} over \texttt{FedAvg} and \texttt{FedSGD} under heterogeneous data, but most importantly our results show that this hierarchy is preserved in our DP-FL framework with privacy constraints: this is especially clear with growing heterogeneity and with growing number $K$ of local updates. Besides this, the results provided for logistic regression in the privacy regime numerically demonstrate that \textbf{\texttt{DP-SCAFFOLD} sometimes even outperforms (non-private) \texttt{FedAvg}} despite the local injection of Gaussian noise, see for instance Fig.~\ref{logistic_heterogene_K50_K100}-\ref{trade_off_sample_ratio} and Fig.~\ref{femnist_and_mnist_heterogene_K50} (bottom row), and to a lesser extent Fig.~\ref{trade_off_user_ratio}. Therefore, our results are quite promising with respect to obtaining efficient DP-FL algorithms under heterogeneous data for higher privacy regimes.

\newpage
\paragraph{Trade-offs between $l$ and $T$.} In this section, we compare the robustness of \texttt{DP-FedAvg} and \texttt{DP-SCAFFOLD} w.r.t. user sampling ratio $l$, under a fixed privacy bound $\epsilon$ towards a third party. We fix parameters $s=0.2, K=10$ and report in Figure~\ref{trade_off_user_ratio} the evolution of the test accuracy of these algorithms for $l \in \{0.08, 0.1, 0.12\}$ over the communication rounds. For each value of $l$, the number of communication rounds $T_l$ is determined to be \textit{maximal} w.r.t. to the privacy bound, so that the desired privacy level is achieved for the output after $T_l$ rounds. These values of $T_l$ are represented on Figure~\ref{trade_off_user_ratio} with red vertical lines (note that the higher $l$, the lower $T_l$). We conduct this experiment on the following datasets: synthetic data with $\epsilon=5$ (Fig~\ref{trade_off_user_ratio}, first row), FEMNIST data with LogReg model, $\epsilon=5$ (Fig~\ref{trade_off_user_ratio}, second row) and MNIST data with DNN model, $\epsilon=3$ (Fig~\ref{trade_off_user_ratio}, third row).

Our results first show that \texttt{DP-SCAFFOLD} achieves much better performance than \texttt{DP-FedAvg} in these high privacy regimes. The superiority of \texttt{DP-SCAFFOLD} towards \texttt{DP-FedAvg} is especially strong under high heterogeneity: we notice a gap of 20\% in the accuracy score with synthetic data for $(\alpha, \beta)=(5,5)$ and FEMNIST data for $\gamma=0\%$ with logistic regression model, a gap of 10\% in the accuracy score for MNIST data with $\gamma=0\%$ and DNN. Furthermore, \texttt{DP-SCAFFOLD} is robust to a low value of user sampling parameter $l$. We observe that we obtain the best performance by choosing $l=0.08$ (the lowest value considered), which allows to set $T_l$ to a high value. Note that the evolution of the accuracy is similar for all values of $l$. Therefore, setting a low $l$ provides an effective way to achieve good accuracy with higher privacy. 
\vspace{-0.2em}
\begin{figure*}[h!]
    \centering
    \includegraphics[width=\linewidth,height=4.8cm]{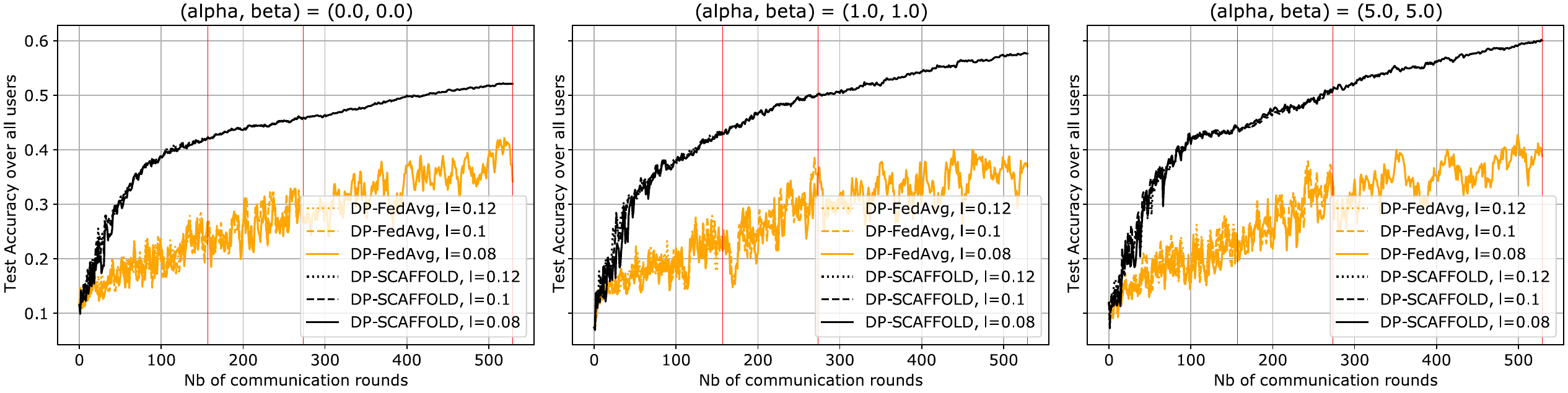} \\
\vspace{0.6em}
    \includegraphics[width=\linewidth, height=4.8cm]{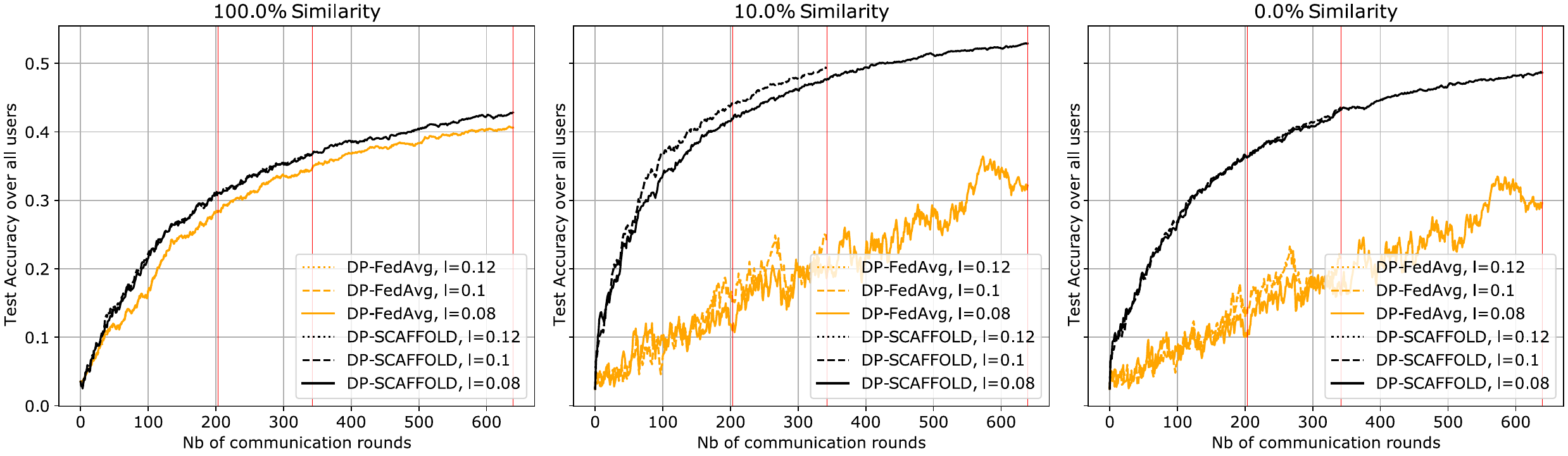}\\
\vspace{0.6em}
    \includegraphics[width=\linewidth, height=4.8cm]{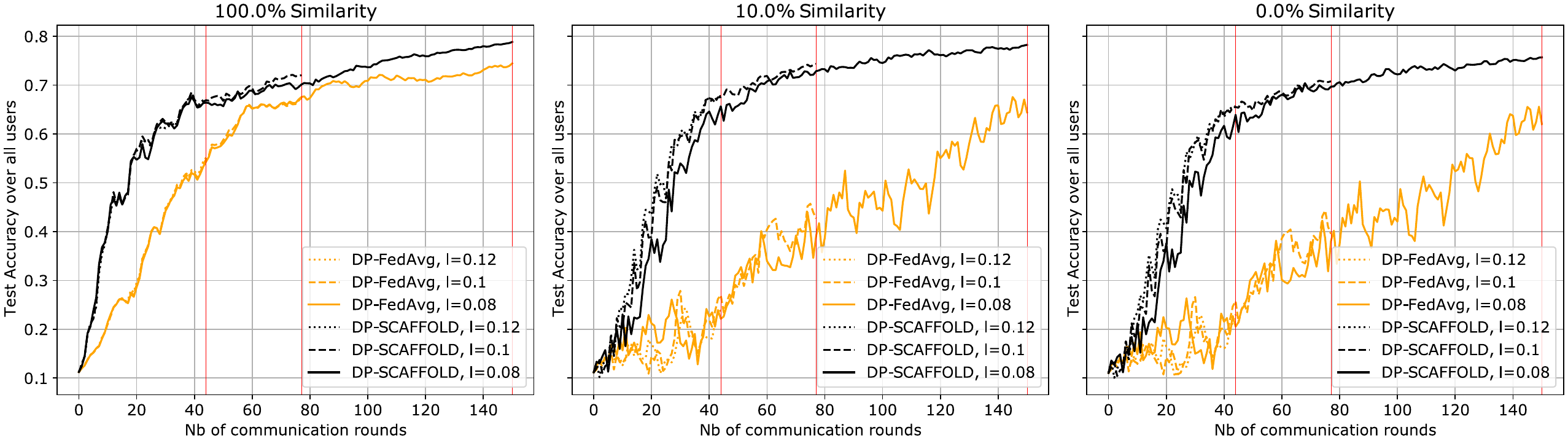}
    \caption{Test Accuracy With Various Values For $l$ Under Fixed Privacy Budget. First Row: Synthetic data ($\epsilon=5$); Second Row: FEMNIST data (LogReg, $\epsilon=5$); Third Row: MNIST (DNN, $\epsilon=3$).}
    \label{trade_off_user_ratio}
\end{figure*}

\clearpage
\paragraph{Trade-offs between $s$ and $T$.} In this section, we compare the behavior of \texttt{DP-FedAvg} and \texttt{DP-SCAFFOLD} w.r.t. data sampling ratio $s$ under a fixed privacy bound $\epsilon$ towards a third party. We fix parameters $l=0.1, K=10$ and report in Figure~\ref{trade_off_sample_ratio} the evolution of the test accuracy of these algorithms for $s \in \{0.05, 0.1, 0.2\}$ over the communication rounds. For each value of $s$, the number of communication rounds $T_s$ is determined to be \textit{maximal} w.r.t. to the privacy bound, so that the desired privacy level is achieved for the output after $T_s$ rounds. These values of $T_s$ are represented on Figure~\ref{trade_off_sample_ratio} with red vertical lines (note that the higher $s$, the lower $T$). We conduct this experiment for the following datasets: synthetic data with $\epsilon=5$ (Fig~\ref{trade_off_sample_ratio}, first row), FEMNIST data with LogReg model, $\epsilon=5$ (Fig~\ref{trade_off_sample_ratio}, second row) and MNIST data with DNN model, $\epsilon=3$ (Fig~\ref{trade_off_sample_ratio}, third row).

Our results confirm that \texttt{DP-SCAFFOLD} leads to better performance than \texttt{DP-FedAvg} with any value of $s\in\{0.2, 0.1, 0.05\}$ under heterogeneity (as expected, we obtain very similar accuracy scores for these two algorithms with $\gamma=100\%$ for FEMNIST and MNIST data). However, this superiority decreases as $s$ decreases. Consider for instance FEMNIST data with 10\% similarity: the gap in accuracy drops from 30\% with $s=0.2$, to less than 20\% with $s=0.1$. Our results seem to show that we obtain better performance with a high value of $s$, although it implies to set $T_s$ to a lower value. This is contrast to the effect of $l$ shown in Fig~\ref{trade_off_user_ratio}.

\begin{figure*}[h!]
    \centering
    \includegraphics[width=\linewidth,height=5cm]{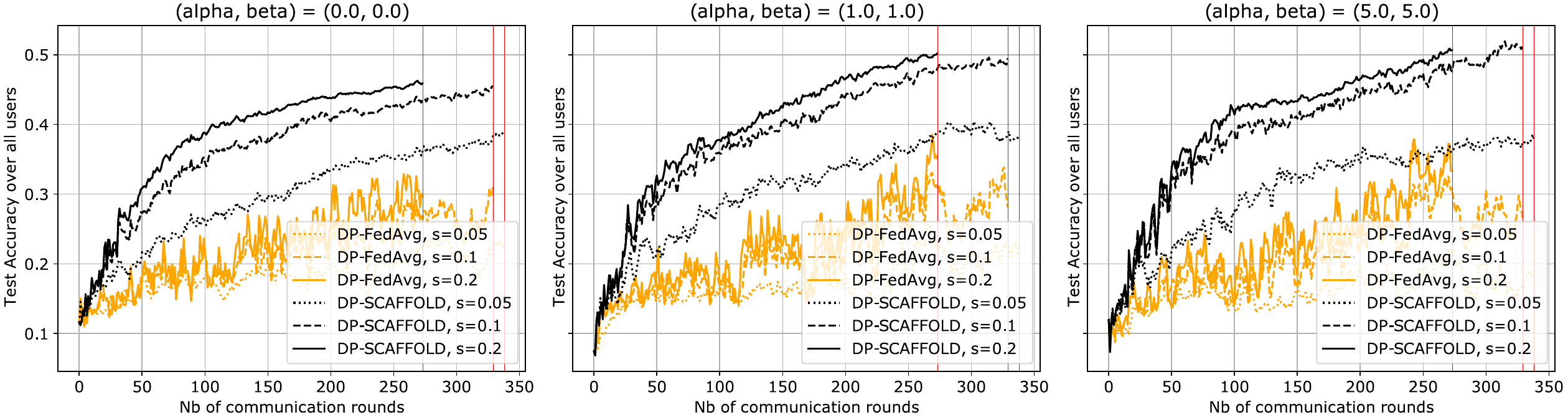} \\
\vspace{0.6em}
    \includegraphics[width=\linewidth,height=5cm]{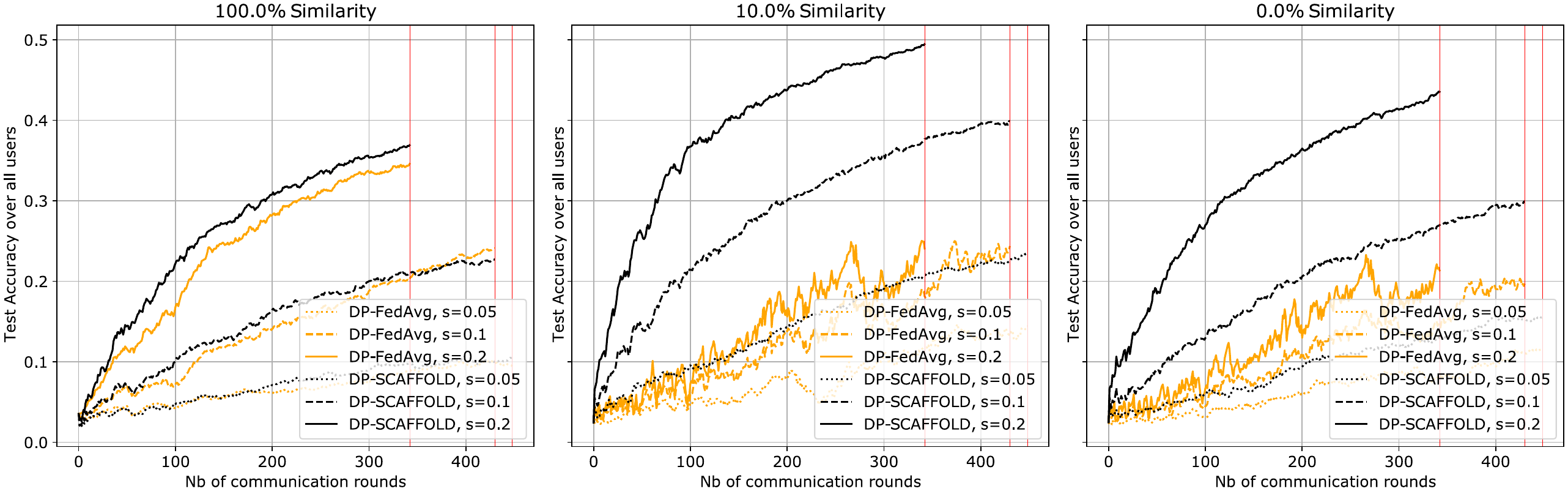}\\
\vspace{0.6em}
    \includegraphics[width=\linewidth,height=5cm]{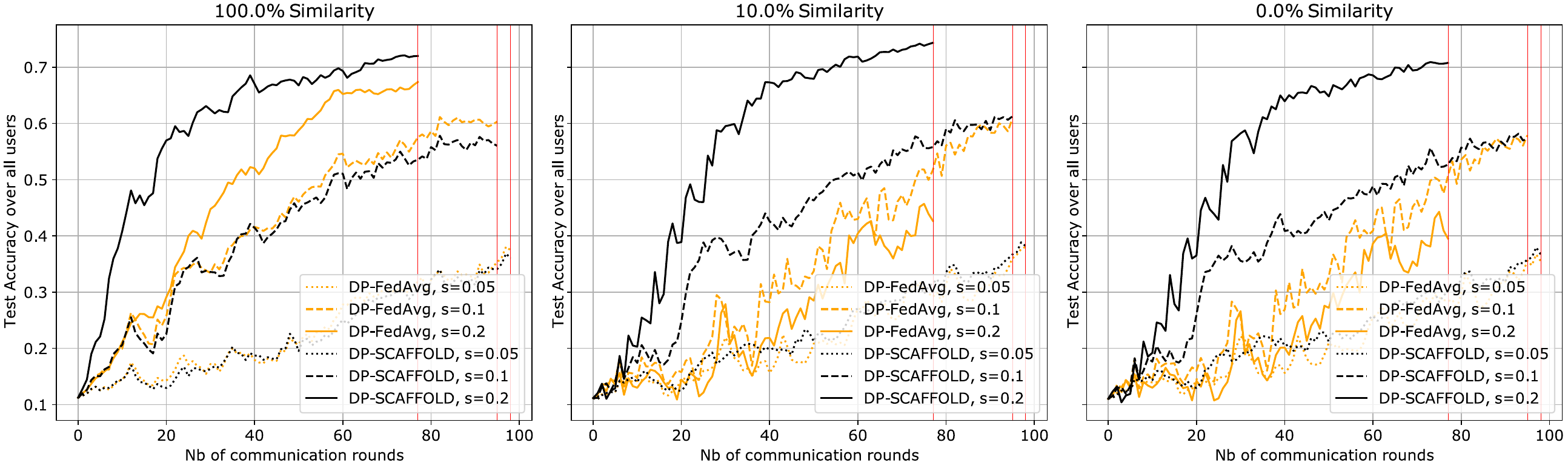}
    \caption{Test Accuracy With Various Values For $s$ Under Fixed Privacy Budget. First Row: Synthetic Data ($\epsilon=5$); Second Row: FEMNIST Data (LogReg, $\epsilon=5$); Third Row: MNIST Data (DNN, $\epsilon=3$).}
    \label{trade_off_sample_ratio}
\end{figure*}

\end{document}